\documentclass[graybox,envcountchap,sectrefs]{svmono}
\usepackage{amsmath,amssymb,amsbsy,eucal,epsfig,algorithm}
\usepackage{algpseudocode}
\usepackage[utf8x]{inputenc}

\usepackage{type1cm}         

\usepackage{graphicx}        %
\usepackage{multicol}        %
\usepackage[bottom]{footmisc}%

\usepackage{newtxtext}       %
\usepackage{newtxmath}       %

\def\qed{$\square$}

\let\epsilon\varepsilon
\let\phi\varphi
\let\eps\varepsilon

\def\X{{\bf X}}

\def\y{{\bf y}}

\def\x{{\bf x}}
\def\C{\mathcal C}
\def\S{\mathcal S}

\def\O{{\cal O}}             
\def\v{{v}}

\def\bd{{\bar L}}

\def\Sp{{\mathcal S}^+}

\def\I{{{}^{{s}}}\hskip-2ptI}

\def\geklz{\ge_{\scriptscriptstyle KL}}

\def\getv{\ge_{\scriptscriptstyle tv}}

\def\odt{{\textstyle{1\over 2}}}
\def\odn{{\textstyle{1\over n}}}

\def\SetN{I\!\!N}

\def\E{\mathbb E}

\def\N{\mathbb N}

\def\as{\text{a.s.}}
\def\-as{\text{-a.s.}}

\def\P{{\mathcal P}}

\newif\ifcompress\compressfalse 
\def\,{\mskip 3mu} \def\>{\mskip 4mu plus 2mu minus 4mu} \def\;{\mskip 5mu plus 5mu} \def\!{\mskip-3mu}
\def\dispmuskip{\thinmuskip= 3mu plus 0mu minus 2mu \medmuskip=  4mu plus 2mu minus 2mu \thickmuskip=5mu plus 5mu minus 2mu}
\def\textmuskip{\thinmuskip= 0mu                    \medmuskip=  1mu plus 1mu minus 1mu \thickmuskip=2mu plus 3mu minus 1mu}
\textmuskip
\ifcompress
\def\beq{\vspace{-1ex}\dispmuskip\begin{equation}}    \def\eeq{\vspace{-1ex}\end{equation}\textmuskip}
\def\beqn{\vspace{-1ex}\dispmuskip\begin{displaymath}}\def\eeqn{\vspace{-1ex}\end{displaymath}\textmuskip}
\def\bqa{\vspace{-1ex}\dispmuskip\begin{eqnarray}}    \def\eqa{\vspace{-1ex}\end{eqnarray}\textmuskip}
\def\bqan{\vspace{-1ex}\dispmuskip\begin{eqnarray*}}  \def\eqan{\vspace{-1ex}\end{eqnarray*}\textmuskip}
\else
\def\beq{\dispmuskip\begin{equation}}    \def\eeq{\end{equation}\textmuskip}
\def\beqn{\dispmuskip\begin{displaymath}}\def\eeqn{\end{displaymath}\textmuskip}
\def\bqa{\dispmuskip\begin{eqnarray}}    \def\eqa{\end{eqnarray}\textmuskip}
\def\bqan{\dispmuskip\begin{eqnarray*}}  \def\eqan{\end{eqnarray*}\textmuskip}
\fi

\title{Universal time-series forecasting with mixture predictors}
\author{Daniil Ryabko}
\date{}

\begin{document}

 \maketitle
\preface
This book is devoted to the problem of sequential probability forecasting, that is, predicting the probabilities of the next outcome of a growing sequence of  observations given the past. This problem is considered in a very  general  setting that unifies commonly used probabilistic and non-probabilistic settings, trying to make as few as possible assumptions on the  mechanism generating the observations. A common form that arises in various formulations of this problem is that of mixture predictors, which are  formed as a combination of a finite or infinite set of other predictors attempting to combine their predictive powers. The main subject of this book are such mixture predictors, and the main results demonstrate the universality of this method in a  very  general probabilistic setting,  but also show some of its limitations. 
While the problems considered are motivated by practical applications, involving, for example, financial, biological or behavioural data, this motivation is left implicit and all the results exposed are theoretical. 

The book targets graduate students and researchers interested in the problem of sequential prediction, and, more generally, in theoretical analysis of problems in  machine learning and non-parametric statistics, as well as mathematical and philosophical foundations of these fields. 

The material in this volume is presented in a way that presumes familiarity with basic concepts of probability and statistics, up to and including probability distributions over spaces of infinite sequences. Familiarity with the literature on learning or stochastic processes is not required.

\vskip3cm
\noindent Antsiranana, 2019. \hfill Daniil Ryabko

\tableofcontents

\chapter{Introduction}\label{ch:intro}
The problem of sequential probability forecasting, also known as  time-series or sequence  prediction, can be introduced as follows. 
A sequence $x_1,\dots,x_n,\dots$ of  observations is generated 
according to some unknown probabilistic law $\mu$. This means that $\mu$ is a  measure on the space of one-way infinite sequences, where the usual sigma-algebra generated by cylinders is assumed. In this book we only consider finite-valued sequences, and the measure $\mu$ is always unknown. 
 After observing each outcome $x_{n}$,  it is required to give the conditional probabilities of the next observation, $x_{n+1}$ given the past $x_1,\dots,x_n$, after which the process continues sequentially.
 We are interested in constructing predictors $\rho$ whose 
 conditional probabilities $\rho(\cdot|x_1,\dots,x_n)$ converge (in some sense) to the
``true'' $\mu$-conditional probabilities $\mu(\cdot|x_1,\dots,x_n)$, as the sequence of observations increases ($n\to\infty$).
We would also like this convergence to be as fast as possible. 

Note that  the probabilities forecast by a predictor $\rho$ having seen a sequence $x_1,\dots,x_n$ may be interpreted as conditional probabilities of $x_{n+1}$ given $x_1,\dots,x_n$. Since the latter sequence is arbitrary, extending the measure from finite sequences to the space of infinite sequences (as is always possible),  the predictor $\rho$ uniquely defines a probability measure over~$\X^\infty$. Assuming (without loss of generality) that the mechanism that generates the data, that is, the observed sequence $x_1,\dots,x_n,\dots$, is probabilistic, predictors and the data-generating mechanism are objects of the same kind: probability measures on the space of one-way infinite sequences.

\begin{svgraybox}
 Predictors and the data-generating mechanism are objects of the same kind: probability measures on the space of one-way infinite sequences.
\end{svgraybox}

The prediction problem is considered in full generality; in particular, the outcomes $x_i, i\in\N$  may exhibit an arbitrary form of probabilistic dependence. This means, in particular, that they may also form an arbitrary deterministic sequence.

\section{General motivation}\label{s:gm}
Problems of this kind arise in a variety of applications, where the data may be financial, such as a sequence of stock prices; human-generated, such as a written text or a behavioural sequence; biological (DNA sequences); physical measurements and so~on. The common feature of these applications that constitutes the motivation of the settings of this volume is the absence of probabilistic models that would describe with at least an acceptable level of accuracy the mechanism that generates the data. 
This does not mean that rather specific probabilistic models  are not used in  applications. 
Consider the following classical example, which is due to Laplace. 
Supposing  
 that the sun has risen every day for five thousand years, what is the
 probability that it will rise tomorrow? Laplace suggested to assume
 that the probability that the sun rises is the same every day and
 the trials are independent of each other. Thus, he considered
the task of sequence prediction when the true  measure generating the data is known to
 belong to the family of Bernoulli i.i.d.\ measures  with the binary
 alphabet $\X=\{0,1\}$. The predicting measure suggested by Laplace
 is 
\begin{equation}\label{eq:lapl}
 \rho_L(x_{n+1}=1|x_1,\dots,x_n)=\frac{k+1}{n+2},
\end{equation}
where $k$ is
 the number of 1s in $x_1,\dots,x_n$. The conditional probabilities
 of the measure $\rho_L$ converge to the true conditional
 probabilities $\mu$-almost surely for every Bernoulli i.i.d\
 process measure $\mu$.

This example, which we would now call a toy example, and which  perhaps was regarded similarly by its author, nevertheless illustrates the situation that is rather common today. Namely, that probabilistic assumptions and models that are obviously wrong~--- it is clear that the events that the sun rises or not on different days are not independent~--- are used for prediction. In particular, the i.i.d.\ model is one of the most used ones to this day. More generally, the choice of a model is often determined not so much by its applicability to a given practical task as by its ease of use.
Specifically, under the assumptions of a given model, one is typically able to prove a rapid convergence of the prediction error to zero. 
In contrast, nothing can be said about the quality of prediction if the model assumptions do not hold. 
A complementary approach is not to assume anything about the mechanism that generates the data, and, instead, consider a given model as a set of comparison predictors or experts. However, in this approach the comparison model class is also typically very small, which allows one to obtain algorithms amenable to theoretical analysis, and, more specifically, a rapid convergence of the prediction error to that of the best comparison method for the data observed. Below in this chapter we consider these approaches in some more detail, in particular, introducing the notions of loss and regret that will be used henceforth to measure the prediction error. 
For now, it suffices to give the following two informal formulations of the   the  sequence prediction problem.
\begin{svgraybox}
 {\bf Sequence prediction: realizable case.} A set of probability measures $\C$ is given and is considered as a model set: the unknown data-generating mechanism $\mu$  is assumed  to belong to the set $\C$. It is required to construct a predictor $\rho$ whose forecasts converge to the $\mu$-conditional probabilities, if any $\mu\in\C$ is chosen to generate the data. 
\end{svgraybox}

\begin{svgraybox}
 {\bf Sequence prediction: non-realizable case.} A set of probability measures $\C$ is given and is considered as a set of predictors to compare to. The unknown data-generating mechanism $\mu$ can be arbitrary.   It is required to construct a predictor $\rho$ whose error is not much greater than that of a comparison predictor $\nu$,  where  $\nu\in\C$ is the best predictor in $\C$ for the measure $\mu$ that generates the data sequence.
\end{svgraybox}

What is important to note here, however, is that some assumptions either directly on the data-generating mechanism or on the comparison class of predictors are necessary in order to obtain at least some non-trivial guarantees on the performance. Indeed,  otherwise, the best one can do is to predict equal probabilities of outcomes  on each step independently of the data.

 To see this, consider any predictor $\rho$, and let us task it with predicting any sequence out of the set of all possible deterministic binary-valued sequences. At time step 1, if the predictor outputs any prediction that differs from $\rho(x_1=0)=\rho(x_1=1)=1/2$, say, $\rho(x_1=0)<1/2$, then let the first outcome of the sequence be the opposite, in our case, $x_0=0$. On the next step, the situation repeats: if,  say, $\rho(x_1=1|x_0=0)<1/2$ then we define $x_1=1$, otherwise $x_1=0$. Consequently, the error of the predictor on the observed sequence is always at least~$1/2$ (or 1 if we measure the error in KL divergence); this simple argument is formalise in Lemma~\ref{th:disc}.

Given that some modelling assumptions are necessary, one can ask the general question of how to select a model for a problem of sequence prediction at hand. Furthermore, given an arbitrary model set $\C$, how can one find a predictor?  If we want to break away from the tradition of studying the models that are  easy to study, we have  to consider these modelling questions in their full generality. 

It turns out that if the data-generating mechanism does belong to the model class, no matter how big the model class is, then asymptotically optimal  performance can be achieved by a mixture predictor, and the finite-time relative performance guarantees can be established. This gives a rather specific form in which to look for a solution to the problem of sequence  prediction, even though it does not provide an algorithm for finding such a predictor.  Interestingly, no assumptions whatsoever are required for this result to hold (not even measurability of the model set). 
\section{Mixture predictors}
Mixture predictors have emerged in various settings in the sequence-prediction literature under various names. The general probabilistic form is $\sum_{k\in\N} w_k\mu_k,$
 where $w_k$ are positive real weights that sum to 1, and $\mu_k$ are probability measures over the space of one-way infinite sequences, that is, predictors.  
\begin{svgraybox}
 A {\bf mixture predictor} (over a set $\C$) is a probability measure of the form 
\begin{equation}\label{eq:mixdef}
\sum_{k\in\N} w_k\mu_k,
 \end{equation}
where $w_k$ are positive real weights that sum to 1, and $\mu_k$ are probability measures (that belong to a given set $\C$).
\end{svgraybox}
  In some settings the sum of the weights may be finite but not 1, and in non-probabilistic settings, such as prediction with expert advise, one have to squint somewhat to see that the predictors have this form; yet, the general principle is the same. The weights of the mixture predictors considered in  this volume   always sum to 1.
Mixture predictors also can be  seen as Bayesian predictors with a  discrete prior, as the weights $w_k$ constitute a probability distribution. In other words, the infinite sum in~\eqref{eq:mixdef} is a special case of an integral. However, as we shall see in this volume, a discrete sum is sufficient to achieve near-optimal performance,  and this also allows us to bypass the requirement of measurability of the set the integral would be taken over (the model set $\C$). A more important difference with Bayesian settings is that  here we focus  on results that hold for every distribution of interest, and not with a prior probability~1.
 
Using a mixture as a predictor means  evaluating  the posterior distribution on each time step $n$, given the observed sequence $x_1,\dots,x_n$. 
Thus, for a mixture predictor $\rho$ of the form~\eqref{eq:mixdef}, its prediction  of the probability of next outcome $x_{n+1}$  to be $a$ after having seen $x_1,\dots,x_n$ is given~by 
\begin{equation}\label{eq:bayes}
 \rho(x_{n+1}=a|x_1,\dots,x_n)=%
\sum_{k\in\N} \frac{w_k\mu_k(x_1,\dots,x_n)}{\sum_{j\in\N} w_j\mu_j(x_1,\dots,x_n)}\mu_k(x_{n+1}=a|x_1,\dots,x_n).
\end{equation}
From this form it is easy to see that those predictors $\mu_k$ that were ``good'' on the sequence $x_1,\dots,x_n$, that is, with respect to which this sequence had a higher likelihood, will have their posterior weight increase as compared to the weights of the rest of the predictors. This forms the basis of Bayesian prediction. However, while typically a Bayesian is concerned with what happens when the measure that generates the data belongs to the set the prior is concentrated on, which, in our case, would be simply the set $\{\mu_k:k\in\N\}$, here we are interested in what happens in a potentially much larger set $\C$ from which the measures $\mu_k$ are carefully selected.
\section{Loss: KL divergence and total variation}

To evaluate the quality of prediction we will mostly use the expected (with respect to data) average (over time)
Kullback-Leibler divergence, as well as the total variation distance (see Chapter~\ref{s:pqpri} for definitions). 
Prediction in total variation is a very strong notion of performance; in particular, it is not even possible
to predict an arbitrary i.i.d.\ Bernoulli distribution in this sense. Prediction in expected average KL divergence
is much weaker and therefore more practical, but it is also more difficult to study, as is explained below.

Total variation distance measures the difference between the true and predicted probabilities not just on the next step but on the infinite horizon; that is, at every time step, we are trying to predict the probability of all future events. Requiring that this error converges to 0 is a rather tough goal. 
Studying the prediction problem with this notion of loss goes back to the seminal work \cite{Blackwell:62}, which shows that if a measure $\mu$ is absolutely continuous with respect to a measure  $\rho$  then $\rho$ predicts $\mu$ in total variation. The converse implication has been established in \cite{Kalai:94}. 
This equivalence between prediction in total variation and absolute continuity proves to be crucial for the  analysis of the sequence prediction problem in this setting. 
In particular, it  can be used to derive a complete characterization of those sets $\C$ for which a predictor exists that predicts all the measures in $\C$. It can also be shown that when such a predictor exists, it can be obtained in the form of a mixture predictor. 
While very insightful from the theoretical point of view, the notion of prediction in total-variation is too strong from a practical point of view, as it is not possible to have  a predictor that predicts all Bernoulli i.i.d.\ measure.  
A rather more practical and interesting alternative is the Kullback-Leibler divergence, which is the main focus in this volume.  This notion of loss has proven to be the natural choice in a variety of applications, including economics (starting with the classical  results of Kelly \cite{Kelly:56}), information theory and data compression (see, e.g. \cite{Cover:91,BRyabko:09,BRyabko:16}).
It would be interesting to generalize the results presented here to different losses, but so far this is left for future work. Some steps towards this are taken in the last chapter, where we consider some basic question about predictors and their combinations for other notions of loss, namely, averaged and not averaged absolute loss, as well as KL loss without time-averaging.

\section{Some of the results}
To get the flavour of the results presented in this book, consider the following theorem, in whose formulation  $L_n(\mu,\rho)$ refers to the time-cumulative (up to time $n$) expected log-loss of $\rho$ on a sequence generated by $\mu$:
\begin{theorem}\label{th:b1i}
For every set $\C$ of probability measures and for every predictor $\rho$ (measure) there is a mixture predictor $\nu$ over $\C$ such that for every $\mu\in\C$ we have $$
 L_n(\mu,\nu)\le L_n(\mu,\rho)+ 8\log n + O(\log\log n),
$$ 
with only small constants hidden inside $O()$.
\end{theorem}
As a corollary of this result, one can show that, in the realizable case, near-optimal performance can always be achieved by a mixture predictor over the model class~$\C$. 
In some more detail, one can consider  the  minimax value $V_\C$, defined as the infinum  over all possible predictors $\rho$ of the supremum over all measures $\mu$ in $\C$ that generate the data (as well as their probabilistic combinations) of the asymptotic average loss $\lim_{n\to
\infty} L_n(\mu,\rho)$. It follows from Theorem~\ref{th:b1i} that this value is achievable by a mixture of a countably many measures from $\C$. 
This latter result can be seen as a complete-class theorem and also as a partial minimax theorem for prediction. These corollaries are considered in Chapter~\ref{s:dt}. %
It is worth noting that the the vast majority of the literature on the problem of sequence prediction in comparable settings is concerned with the case $V_\C=0$, that is, the  loss of the best predictor vanishes in asymptotic. This is the case not only for i.i.d.\ data, but also for finite-memory and stationary processes (see Section~\ref{s:lit} below for more details and references). More general cases still considered in the literature include piece-wise i.i.d.\ processes but rarely go beyond these. Here we argue, based on the general result above and on its couter-part below, that more general models $\C$, with $V_\C>0$,  deserve a closer attention.

Turning to the non-realizable case, that is, to the case where  the measure generating the data is arbitrary and we are comparing the performance of our predictor to that of the best predictor in a given set $\C$, an opposite result is true. 
More specifically, it may happen that there exists a predictor with a vanishing regret (with respect to the best predictor in $\C$), but any Bayesian combination (not necessarily a discrete mixture) of predictors in $\C$ has a linear (non-vanishing) regret. This effectively renders the set $\C$ useless. 
From these two opposite results taken together, a a rather general conclusion emerges:
\begin{svgraybox}
 It is better to take a model large enough to make sure it includes the process that generates the data, even if it entails positive asymptotic average loss, for otherwise any combination of predictors in the model class may be useless. 
\end{svgraybox}
Shall this conclusion be not quite satisfying from an application point of view, yet another setting may be considered, which is called here the ``middle case'' prediction problem.
\begin{svgraybox}
 {\bf The middle case problem.} Given a set $\C$ of predictors, we want to find a predictor whose
prediction error converges to zero if there is at least one predictor in $\C$ whose prediction error converges to zero.
\end{svgraybox}
From the point of view of finding a model for the sequence prediction problem, this setting allows one to assume that the data-generating mechanism is such that {\em the model is good enough} for the prediction problem at hand. This is much weaker than assuming that the measure generating the data belongs to the given family; yet much stronger than  allowing the data to be arbitrary (while weakening the goal to be a comparative one)  as in the non-realizable case. 
As we shall see in this volume (Chapter~\ref{ch:p2}), in this case, similarly to the  realizable case, optimal performance is attained by a mixture predictor. However, unlike in the realizable case, for this setting   only an asymptotic result is presented.

\section{Mixture predictors in the literature on sequence prediction and related settings}\label{s:lit}
Mixture predictors are a special case of Bayesian predictors, and thus their roots can be traced back to the founders of probability and statistics. In particular, the Laplace predictor mentioned above is actually a Bayesian mixture of Bernoulli i.i.d.\ processes with the uniform prior over the parameter.  The use of discrete mixtures for sequence prediction goes back to at least Solomonoff \cite{Solomonoff:78} that had used the Zvonkin-Levin measure \cite{Zvonkin:70} for this purpose (basing also on  ideas of  earlier works \cite{Solomonoff:64}). The latter measure is a mixture of all computable measures, and is thus a good predictor for them. All computable measures in  a certain sense capture all the possible regularities that one can encounter, so constructing a predictor for this set is a rather attractive goal. Yet it is elusive, since, although a predictor for this set exists, it is not computable itself; see \cite{Hutter:04uaibook} for an overview and a discussion of this and related problems. Furthermore, unless one accepts a philosophical premise that all the probabilities in the world are computable, predicting all such measures is not the ultimate goal for sequential prediction. It is, however, easy to see that a mixture over an infinite set of measures may aslo be a good predictor for some measures outisde this set. Thus, a mixture over Bernoulli measures with rational (or with all computable) coefficients predicts all Bernoulli measures, as does the Solomonoff's predictor mentioned above (e.g., \cite{Hutter:07upb}).  As a perhaps more practical example, in \cite{BRyabko:84} a mixture $\sum w_k \mu_k$ of predictors  is proposed, where each $\mu_k$  is  a predictor for all processes with memory $k$ (these are themselves Bayesian predictors, see \cite{Krichevsky:93}). Not only does this mixture predict all processes with finite memory, but, as \cite{BRyabko:88} shows, it is also a good predictor for all stationary processes. These examples are considered in some detail in Section~\ref{s:exvc0}.

Perhaps the first result on mixture predictors that goes beyond specific families of processes has been obtained in \cite{Ryabko:10pq3+}, which shows that if there exists a predictor $\rho$ whose loss vanishes on any measure from a given set $\C$, then such a predictor can be obtained as a mixture over a countable subset of $\C$. This result does not require any assumptions on $\C$ to be made, besides the existence of the predictor $\rho$ whose loss vanishes.  This last requirement has been lifted in \cite{Ryabko:17pq5}. That is, the latter work shows that the best asymptotic performance can always be attained by a mixture predictor, no matter whether the minimax value is 0 or positive. Finally, a finite-time version of this result has been obtained in \cite{Ryabko:19pq6}, which is Theorem~\ref{th:b1i} mentioned above and one of the main results presented in this book.

The non-realizable case is usually studied in a slightly different, non-probabilistic, setting. We 
refer to \cite{Cesa:06} for a comprehensive overview. In this setting, it is usually assumed 
that the observed sequence of outcomes is an arbitrary (deterministic) sequence; instead of giving 
conditional probabilities, a predictor has to output deterministic guesses (although these guesses can be selected in a  randomised manner). Predictions result in a certain loss, which 
is required to be small as compared to the loss of a given set of reference predictors (experts) $\C$. The losses of the experts
and the predictor are observed after each round.
 In this approach, it is mostly assumed
that the set $\C$ is finite or countable.
The main difference with the formulation considered in this section is that we require a predictor to give probabilities, and 
thus the loss is with respect to something never observed (probabilities, not outcomes). The loss itself is not completely 
observable in our setting.  In this sense our non-realizable
version of the problem is more  difficult. 
Assuming that the data-generating mechanism is probabilistic, 
even if it is completely unknown, makes sense in such problems as, for example, game playing, or market analysis. 
In these cases one may wish to assign smaller loss to those models or experts who give probabilities closer 
to the correct ones (which are never observed), even though different probability forecasts can often result
in the same action. 
Aiming at predicting probabilities of outcomes  also allows one to abstract
from the actual use of the predictions (for example, making bets), and thus it becomes less important to consider general losses.
On the other hand, using observable losses instead of the unobserved ``true''   probabilities has the merit of being able to evaluate a given predictor directly based on data; see the works \cite{Dawid:84,Dawid:92} that argue in favour of this approach.
Despite the difference in settings, it can be seen that the predictors used in the expert-advice literature are also mixture predictors, where likelihood-like quantities are calculated for the observed sequence  based on the loss suffered by each expert, and then are used as weights to calculate the forecasts akin to \eqref{eq:bayes} (see \cite{Cesa:06} for a formal exposition).

The first attempt to unify the realizable and the non-realizable case of the sequence prediction problem in its general form has been taken in \cite{Ryabko:11pq4+}.
In this work both of these problems are considered  as questions about a set $\C$ of measures, which is  regarded,  alternatively, as a set of processes to predict or a set of experts to compare with. It is demonstrated in \cite{Ryabko:11pq4+} that the problems become one in the case of prediction in total variation distance,  providing  a complete characterisation of predictability for this loss. It is also shown that  for KL divergence the problems are different.  In the same work  the middle-case problem is introduced, which lies in-between the realizable and the non-realizable case; this latter problem is the subject of Chapter~\ref{ch:p2} in this volume. An important example of stationary processes for this problem has been addressed in \cite{Ryabko:15stno}, which shows that one cannot predict all processes for which a stationary predictor exists. Furthermore, the question of optimality of mixture predictors in the non-realizable case has been answered in the negative in \cite{Ryabko:16pqnot,Ryabko:19pq6}. Together with the mentioned result on universality of mixture predictors in the realizable case, this invites the general conclusion that one should try to take a model set as large as to include the data-generating mechanism in it, even if the best achievable asymptotic average loss becomes positive \cite{Ryabko:19pq6}.

Another closely related direction of research, mostly pursued in the econometrics literature, attempts to generalize the classical result of Blackwell and Dubins \cite{Blackwell:62} on prediction in total variation. Realizing that this notion of prediction is too strong, the question of under which conditions one measure predicts another is considered for weaker notions of convergence of error, mostly in Bayesian settings, by  \cite{Kalai:94,Kalai:92,Jackson:99,Noguchi:15}  and others. 
For KL divergence some of these questions are addressed in  the works \cite{Ryabko:07pqisit,Ryabko:08pqaml}, which  make the basis of Chapter~\ref{ch:other}.

\section{Organization}
This book is organized as follows. The next chapter introduces the  notation and the main definitions, including the definitions of the two notions of loss used and the notion of regret. Chapter~\ref{ch:tv} is concerned with prediction in total variation. Here we shall see that the realizable and non-realizable case become equivalent, and provide a complete characterisation of those sets $\C$ for which a solution to the problem of prediction in either setting exists. Chapter~\ref{ch:kl} contains the main results of this volume: it is devoted to the problems of prediction with KL loss. The main results show that, in the realizable case, the best performance can be matched by a mixture predictor, providing both upper and lower bounds on the difference between the finite-time loss of the best mixture predictor and an arbitrary comparison predictor. It is also established  in this chapter (Section~\ref{s:not}) that,  in the non-realizable case, the opposite is true: it is possible that a vanishing regret is achievable but it cannot be achieved by any Bayesian predictor (including mixture predictors). Various examples of different sets of processes are considered.  In addition, in section \ref{s:pq3+} some sufficient conditions on a class $\C$ for the asymptotic KL loss to vanish are considered; these are inspired by the characterisations obtained for predictability in total variation, although for KL loss only sufficient conditions (and not necessary and sufficient) are presented. In the same vein, some sufficient conditions for the regret to vanish are presented in Section~\ref{s:sufreg}.  In particular, it is shown that vanishing regret can be obtained for the set of all finite-memory processes, but not for stationary processes.
Chapter~\ref{s:dt} presents decision-theoretic interpretations of the asymptotic results obtained. Specifically, connections to minimax and complete-class theorems of decision theory are made. In Chapter~\ref{ch:p2} we consider a prediction problem that lies in-between the realizable and the non-realizable case, which is called middle-case here for want of a better term. This problem is considered in finding a predictor that predicts all measures that are predicted well (that is, with a vanishing loss) by at least one measure in a given set $\C$. It is shown that, when this can be done, it can be done using a mixture predictor. The latter result is established for KL loss, whereas, unsurprisingly, for total variation the equivalence of the realizable and non-realizable case can be extended to the middle case as well (i.e, all the three problems coincide). Finally, Chapter~\ref{ch:other} considers some related topics in sequential prediction, namely, the question of under which conditions one measures is a good predictor for another, and how this property is affected by taking mixtures of measures. In this last chapter we also venture beyond the two notions of loss considered in the rest of the book.
\chapter{Notation and definitions}\label{s:pqpri}
 For a  finite set $A$ denote $|A|$ its cardinality; $\N$ stands the set of naturals without 0. 
Let $\X$ be a finite set and let \begin{equation}\label{eq:M}
 M:=\log |\X|.
\end{equation}
 The notation $x_{1..n}$ is used for $x_1,\dots,x_n$. 
 We consider  stochastic processes (probability measures) on $\Omega:=(\X^\infty,\mathcal F)$ where $\mathcal F$
is the sigma-field generated by the the (countable) set $(B_i)_{i\in\N}$ of cylinders, $B_i:=\{{\bf x}\in\X^\infty: x_{1..|b_i|}=b_i\}$ where the words $b_i$ take all possible values in $\X^*:=\cup_{k\in\N}\X^k$.
We use  $\E_\mu$ for
expectation with respect to a measure $\mu$.

\begin{definition}[Classes of processes: all, discrete, Markov, stationary] Consider the following classes of process measures: $\mathcal P$ is the set of all probability measures, 
 $\mathcal D$  is the set of all  discrete process measures (Dirac $\delta$ processes), $\mathcal S$ is the set of all stationary processes and
$\mathcal M_k$ is the set of all stationary measures with memory not greater than $k$ ($k$-order Markov processes, with $\mathcal M_0$ being
the set of all i.i.d.\ processes):
\begin{equation}
 \mathcal P:=\text{ the set of all probability  measures on }(\X^\infty,\mathcal F),
\end{equation}
\begin{equation}
\mathcal D:=\left\{\mu\in\mathcal P: \exists  x\in\X^\infty \ \ \mu(x)=1\right\},
\end{equation}
\begin{equation}
\mathcal S:=\left\{\mu\in\mathcal P: \forall n,k\ge1\,\forall a_{1..n}\in\X^n\, \mu(x_{1..n}=a_{1..n})=\mu(x_{1+k..n+k}=a_{1..n})\right\}.
\end{equation}
\begin{multline}\label{eq:mk}
\mathcal M_k:=\left\{\mu\in\mathcal S:\forall n\ge k\,\forall a\in\X\, \forall a_{1..n}\in\X^n \right. \\ \left. \mu(x_{n+1}=a|x_{1..n}=a_{1..n})=\mu(x_{k+1}=a|x_{1..k}=a_{n-k+1..n})\right\}.
\end{multline}
\end{definition}
Abusing the notation, we will sometimes use  elements of $\mathcal D$ and $\X^\infty$ interchangeably.
The following (simple and well-known)  statement   will be used repeatedly in the examples.
\begin{lemma}[It is impossible to predict every process]\label{th:disc}
 For every $\rho\in\mathcal P$ there exists $\mu\in\mathcal D$ such that $L_n(\mu,\rho)\ge n\log|\X|$ for all $n\in\N$.
\end{lemma}
\begin{proof}
 Indeed, for each $n$ we can select $\mu(x_n=a)=1$ for $a\in\X$ that minimizes $\rho(x_n=a|x_{1..n-1})$, so that $\rho(x_{1..n})\le|\X|^{-n}$.
\qed\end{proof}

\begin{definition}[Mixture predictors]\label{def:mixp}
 A probability measure $\rho$ is a (discrete) {\em mixture predictor}, also known as a {\em discrete Bayesian predictor} with a prior over $\C$ if $\rho=\sum_{k\in\N}w_k\mu_k$, for some measures $\mu_k\in\C$, $k\in\N$, where  $w_k$ are positive real weights that sum to 1. The latter weights can be seen as a prior distribution over (a subset of) $\C$.  
\end{definition}

In most of this book we are only dealing with discrete Bayesian predictors, but the more general case of possibly non-discrete prior distributions is necessary for some impossibility results (specifically, in  Section~\ref{s:not}). For this we need a  structure of the probability space on $\P$.
Here we shall define it in a standard way, following~\cite{Gray:88}. Consider the countable set of  sets $\{\nu\in\mathcal P: \nu(B_i)\in u\}$, where $u\subset[0,1]$, obtained by taking  all the cylinders $B_i,i\in\N$ and all intervals $u$ with rational endpoints. This set generates a sigma-algebra $\mathcal F'$. 
Denoting $\P'$ the set of all  probability measures over $\P$ we obtain the measurable space $(\P',\mathcal F')$.

Associated with any probability measure $W\in\P'$ there is a probability measure $\rho_W\in\P$  defined by $\rho_W=\int_P\alpha d W(\alpha)$ ({\em barycentre}, in the terminology of \cite{Gray:88}; see the latter work  for a detailed exposition). A measure $\rho\in\P$ is {\em Bayesian} for a set $\C\subset\P$ if 
$\rho=\rho_W$ for some $W\in\P'$ such that there exists a measurable set $\bar\C\subset\C$ with $W(\bar \C)=1$. 
(The subset $\bar\C$ is necessary since $\C$ may not be measurable.) Note that the distribution $W$ can be seen as a prior over $\C$, 
and $\rho_W$  as a predictor whose predictions are obtained by evaluating the posterior distribution $\rho_W(\cdot|x_1,\dots,x_n)$ where $x_1,\dots,x_n$ are observations 
up to time $n$.
\section{Loss}
For two measures  $\mu$ and $\rho$  we are interested in how different 
 the $\mu$- and $\rho$-conditional probabilities are, given a data sample $x_{1..n}$. This can be measured in various ways, of which we consider two. Some other losses that are only used in the last chapter are defined there.
Most of the results of this book are obtained with respect to the so-called KL divergence or expected log-loss. The choice is motivated by its ease of use as well as by its importance in applications. The other loss that we use is the total variation distance. The latter one is far too strong for most applications, but it is instructive to consider it because of the connection of the prediction problem to that of dominance of measures, as explained below. 
\subsection{KL divergence}\label{s:kl}
For two probability measures $\mu$ and $\rho$ introduce the {\em expected cumulative Kullback-Leibler divergence (KL divergence)} as
\begin{multline}\label{eq:kl} 
  L_n(\mu,\rho) :=  \E_\mu
  \sum_{t=1}^n  \sum_{a\in\X} \mu(x_{t}=a|x_{1..t-1}) \log \frac{\mu(x_{t}=a|x_{1..t-1})}{\rho(x_{t}=a|x_{1..t-1})}
 \\=
  \sum_{x_{1..n}\in\X^n}\mu(x_{1..n}) \log \frac{\mu(x_{1..n})}{\rho(x_{1..n})}.
\end{multline}
In words, we take the $\mu$-expected (over data) cumulative (over time) KL divergence between $\mu$- and $\rho$-conditional (on the past data) 
probability distributions of the next outcome; and this gives, by the chain rule for entropy used in the last equality, simply the $\mu$-expected log-ratio of the likelihoods. Here $\mu$ will be interpreted as the distribution generating the data.

Jumping ahead, the last equality in~\eqref{eq:kl}, will be used repeatedly in the proofs of various results presented, using the following simple but very useful trick (in the context of prediction introduced perhaps by \cite{BRyabko:88}). Let $\rho$ be a mixture predictor given by~\eqref{eq:mixdef}, and let $\nu$ be an arbitrary measure from $\mathcal P$. 
Then
\begin{multline}\label{eq:tric}
  L_n(\nu,\rho)=\E_{\nu}\log\frac{\log\nu(x_{1..n})}{\log\rho(x_{1..n})}
= \E_{\nu}\log\frac{\log\nu(x_{1..n})}{\log\sum_{k\in\N}w_k\mu_k(x_{1..n})} 
\\\le \E_{\nu}\log\frac{\log\mu_k(x_{1..n})}{w_k\log\mu_q(x_{1..n})} 
= L_n(\nu,\mu_k) -\log w_k. 
\end{multline}
In words, the loss of the mixture, on the data generated by an {\em arbitrary measure} is upper-bounded by the loss of any of the components of the mixture plus a constant, where the constant is $-\log w_k$, $w_k$ being the a priori weight of the corresponding  component $\mu_k$ in the mixture. 
This trick, while very simple, may take some time to assimilate: it is somewhat confusing that the constant on the right-hand side that bounds the difference between the cumulative prediction error of the mixture and the component does not depend on the measure $\nu$ that generates the data, and only depends on the a priori weight. Put differently, the a posteriori weights, cf.~\eqref{eq:bayes}, play no role in the bound. Moreover, the contribution of the a priori weight is negligible, $O(1)$, with respect to the cumulative time $n$. 

The asymptotic  average expected loss is then defined as
$$
\bar L(\nu,\rho):=\limsup_{n\to\infty} {1\over n} L_n(\nu,\rho),
$$
where the upper limit is chosen so as to reflect the worst performance over time.
 One can define the worst-case performance  of a strategy $\rho$ by
$$
\bd(\C,\rho):=\sup_{\mu\in\C}\bd(\mu,\rho)
$$ 
 and the minimax value by %
\begin{equation}\label{eq:vc}
  V_\C:=\inf_{\rho\in\P} \bd(\C,\rho). %
\end{equation}
Some examples of calculating the latter value for different sets $\C$ are considered in Section~\ref{s:vcpos}; the most common case in the literature, however, is the case of $V_\C=0$; this is the case, for example, if $\C$ is the set of all i.i.d.\ or all stationary distributions, see Section~\ref{s:exvc0}.

\subsection{Total variation}
Introduce the (unconditional) total variation distance between process measures.
\begin{definition}[unconditional total variation distance]
 The  (unconditional) total variation distance is defined as 
$$
v(\mu,\rho):= \sup_{A\in\mathcal B} |\mu(A)-\rho(A)|.
$$
Furthermore, the  {\em (conditional) total variation} distance is defined similarly as
$$
v(\mu,\rho,x_{1..n}):= \sup_{A\in\mathcal B} |\mu(A|x_{1..n})-\rho(A|x_{1..n})|,
$$
if $\mu(x_{1..n})\ne0$ and $\rho(x_{1..n})\ne0$, and $v(\mu,\rho,x_{1..n})=1$ otherwise.
\end{definition}

\begin{definition}\label{def:tvloss} Introduce the  almost sure total variation loss of $\rho$ with respect to~$\mu$
$$
l_{tv}(\mu,\rho):=\inf\{\alpha\in[0,1]: \limsup_{n\to\infty} v(\mu,\rho,x_{1..n})\le\alpha\ \mu\text{--a.s.}\},
$$
\end{definition}

\begin{definition}\label{d:ctv}
We say that $\rho$ predicts $\mu$ in total variation if 
$$l_{tv}(\mu,\rho)=0,$$
or, equivalently, if 
$$
v(\mu,\rho,x_{1..n})\to0\ \mu\-as
$$
\end{definition}
This convergence is rather strong. In particular, it means that $\rho$-conditional probabilities
of arbitrary far-off events converge to $\mu$-conditional probabilities. 
Moreover,  $\rho$ predicts $\mu$ in
total variation if \cite{Blackwell:62} and only if \cite{Kalai:94} $\mu$ is absolutely continuous with respect to $\rho$.
\begin{definition}[absolute continuity; dominance]\label{d:ac}
 A measure $\mu$ is {\em absolutely continuous} (or is {\em dominated by}) with respect to a measure $\rho$ if the implication $\mu(A)>0\Rightarrow\rho(A)$ holds for every $A\in\mathcal F$.
\end{definition}

\begin{theorem}[\cite{Blackwell:62, Kalai:94}]\label{th:bd}
If  $\rho$, $\mu$ are arbitrary probability measures on $(\X^\infty,\mathcal F)$, then $\rho$ predicts $\mu$ in total variation if and only if  $\mu$ is absolutely continuous with respect to $\rho$.
\end{theorem}

Denote $\getv$ the relation of absolute continuity:  $\rho\getv\mu$ if $\mu$ is absolutely continuous with respect to $\rho$.

Thus, for a class $\C$ of measures there is a predictor $\rho$ that predicts every $\mu\in\C$ in total
variation if and only if every $\mu\in\C$ has a density with respect to $\rho$.
Although such  sets of processes are rather large, they do not include even such basic 
examples as the set of all Bernoulli i.i.d.\ processes.
That is, there is no $\rho$ that would predict in total variation every Bernoulli i.i.d.\ process measure $\delta_p$, $p\in[0,1]$,
where $p$ is the probability of $0$. 
Indeed, all these processes $\delta_p$, $p\in[0,1]$, are singular with respect to one another; in particular, 
each of the non-overlapping sets $T_p$ of all sequences which have limiting fraction $p$ of 0s has probability 1 with respect to one of the measures and 
0 with respect to all others; since there are uncountably many of these measures, there is no measure $\rho$ with respect to which they all would have a density 
(since such a measure should have $\rho(T_p)>0$ for all $p$).

\section{Regret}
While the notions introduced above suit well the realizable version of the sequence prediction problem, that is, assuming that the measure generating the observed sequence belongs to the set $\C$, for the non-realizable case we switch the roles and consider a set $\C$ as the set of comparison predictors or experts.  
The measure generating the data can now be arbitrary. %
The goal now is to find a strategy that is as good as the best one in $\C$, for the given data. 

Thus, in terms of KL divergence, we are interested in the (asymptotic) {\em regret} %
 $$
 \bar R^\nu(\mu,\rho):=\limsup_{n\to\infty}{1\over n}\left[L_n(\nu,\rho) - L_n(\nu,\mu)\right],
 $$
 of using $\rho$ as opposed to $\mu$  on the data generated by $\nu$. %
The goal is to find $\rho$ that minimizes the worst-case (over measures generating the data) regret with respect to the best comparison predictor (expert) from the given set $\C$:
 $$R(\C,\rho):=\sup_{\nu\in\P}\sup_{\mu\in\C}\bar R^\nu(\mu,\rho).$$

Note than in the expert-advice literature the regret is typically defined on finite sequences of length $n$, thus allowing both the experts and the algorithms to depend on $n$ explicitly. 

Similarly to $V_\C$, we can now define the value 
$$
U_\C:=\inf_{\rho\in\mathcal{P}}  R(\C,\rho),
$$
which is the best (over predictors) worst-case asymptotic average regret with respect to the set of experts~$\C$.

In view of the (negative) result that is obtained for regret minimization, here we are mostly concerned with  the case $U_\C=0$.

For {\bf total variation}, we can similarly define the regret of using $\rho$ as opposed to $\mu$  on data generated by $\nu$
$$
r_{tv}^\nu(\mu,\rho):=l_{tv}(\nu,\rho)-l_{tv}(\nu,\mu),
$$
and 
 the worst-case (over measures generating the  data) regret with respect to the best comparison predictor (expert) from the given set $\C$:
$$
r_{tv}(\C,\rho):=\sup_{\nu\in\P}\sup_{\mu\in\C}r_{tv}^\nu(\mu,\rho).
$$
We shall only make a very limited use of these notions, however, since, as is shown in the next chapter, the total variation loss converges to either 0 or 1.

\chapter{Prediction in total variation: characterizations}\label{ch:tv}
Studying prediction in total variation hinges on the fact (formalized above as Theorem~\ref{th:bd}) 
that  a measure $\mu$ is absolutely continuous with respect to a measure $\rho$
if and only if $\rho$ predicts $\mu$ in total variation distance. 
This is indeed of great help; as a result, for prediction in total variation it is possible to provide complete characterizations of those 
sets of measures $\C$ for which predictors exist (realizable case). The non-realizable case turns out to be somewhat degenerate as the asymptotic error in total variation is either 0 or 1. These and related results are exhibited in this chapter.

\section{Optimality of mixture predictors}
We start with the realizable case and show that if a solution to the sequence prediction problem  can be found for a set $\C$ then it can be found in the form of a mixture predictor.
\begin{theorem}[Optimality of mixture predictors for the realizable case, total variation]\label{th:pq3-1} Let $\C$ be a set of probability measures on $(\X^\infty, \mathcal F)$. If there is a measure $\rho$ such that $\rho$ predicts every $\mu\in\C$ in 
total variation, then there is a sequence $\mu_k\in\C$, $k\in\N$ such that the measure $\nu:=\sum_{k\in\N} w_k\mu_k$ predicts every $\mu\in\C$ in 
total variation, where $w_k$ are any positive weights  that sum to~1.
\end{theorem}
This relatively simple fact can be proven in different ways relying on the mentioned equivalence with absolute continuity. %
The proof presented below is not the shortest possible, but it uses ideas and techniques that are then generalized to 
the case of  prediction in expected average KL-divergence, which is more involved, since in all interesting cases 
all measures $\mu\in\C$ are singular with respect to any predictor that predicts all of them. 
Another proof of Theorem~\ref{th:pq3-1} can be obtained from Theorem~\ref{th:sep1} below. Yet another way would
be to derive it from algebraic properties of the relation of absolute continuity, given in \cite{Plesner:46}.
\begin{proof}
We break the (relatively easy) proof of this theorem into three steps, which will help make the 
proof of the corresponding  theorems about KL divergence more understandable.

\noindent{\em Step 1: densities.}
 For any $\mu\in\C$, since $\rho$ predicts $\mu$ in total variation, by Theorem~\ref{th:bd}, $\mu$ has a density (Radon-Nikodym derivative) $f_\mu$ with respect 
to $\rho$. Thus, for the (measurable) set $T_\mu$ of all sequences $x_1,x_2,...\in\X^\infty$ on which $f_\mu(x_{1,2,\dots})>0$ 
(the limit $\lim_{n\rightarrow\infty}\frac {\rho(x_{1..n})}{\mu(x_{1..n})}$ 
exists and is finite and positive) we have $\mu(T_\mu)=1$ and $\rho(T_\mu)>0$. Next we will construct a sequence of measures $\mu_k\in\C$, $k\in\N$ such that 
the union of the sets $T_{\mu_k}$ has probability 1 with respect to every $\mu\in\C$, and will show that this is a sequence of measures whose existence is asserted in the theorem statement.

{\em Step 2: a countable cover and the resulting predictor.}
Let $\epsilon_k:=2^{-k}$ and let $m_1:=\sup_{\mu\in\C}\rho(T_\mu)$. Clearly, $m_1>0$. Find any $\mu_1\in\C$ such that $\rho(T_{\mu_1})\ge m_1-\epsilon_1$, and let
$T_1=T_{\mu_1}$. For $k>1$ define $m_k:=\sup_{\mu\in\C}\rho(T_\mu\backslash T_{k-1})$. If $m_k=0$ then define $T_{k}:=T_{k-1}$, otherwise find any $\mu_k$ such 
that $\rho(T_{\mu_k}\backslash T_{k-1})\ge m_k-\epsilon_k$, and let $T_k:=T_{k-1}\cup T_{\mu_k}$. 
Define the predictor $\nu$ as $\nu:=\sum_{k\in\N}w_k\mu_k$.

{\em Step 3: $\nu$ predicts every $\mu\in\C$.}
Since the sets 
$T_1$, $T_2\backslash T_1,\dots, T_k\backslash T_{k-1},\dots$ are disjoint, 
we must have $\rho(T_{k}\backslash T_{k-1})\to0$, so that $m_k\to0$ (since $m_k\le\rho(T_{k}\backslash T_{k-1})+\epsilon_k\to0$).
Let 
$$
T:=\cup_{k\in\N} T_k.
$$ 
Fix any $\mu\in\C$. 
Suppose that  $\mu(T_{\mu} \backslash T)>0$. Since $\mu$ is absolutely continuous
with respect to $\rho$, we must have $\rho(T_{\mu}\backslash T)>0$. Then for every $k>1$ we have
$$m_k=\sup_{\mu'\in\C}\rho(T_{\mu'}\backslash T_{k-1})\ge  \rho(T_{\mu}\backslash T_{k-1})\ge\rho(T_{\mu}\backslash T)>0,$$ which contradicts 
$m_k\rightarrow0$. Thus, we have shown that 
\begin{equation}\label{eq:mt}
\mu(T\cap T_{\mu})=1.
\end{equation}

Let us show that every $\mu\in\C$ is absolutely continuous with respect to $\nu$.
Indeed, fix any $\mu\in\C$ and suppose $\mu(A)>0$ for some $A\in\mathcal F$.
Then from~(\ref{eq:mt}) we have $\mu(A\cap T)>0$, and, by absolute continuity of $\mu$ with respect to $\rho$,
also $\rho(A\cap T)>0$. Since $T=\cup_{k\in\N}T_{k}$, we must have $\rho(A\cap T_k)>0$ for some $k\in\N$.
Since on the set $T_k$ the measure $\mu_k$ has non-zero density $f_{\mu_k}$ with respect to $\rho$, we must have $\mu_k(A\cap T_k)>0$.
(Indeed, $\mu_k(A\cap T_k)=\int_{A\cap T_k}f_{\mu_k}d\rho>0$.)
Hence,  
$$
\nu(A\cap T_k)\ge w_k \mu_k(A\cap T_k)> 0,
$$ so that $\nu(A)>0$.
Thus, $\mu$ is absolutely continuous with respect to $\nu$, and so, by Theorem~\ref{th:bd}, $\nu$ predicts $\mu$ in total variation distance.
\qed\end{proof}

\section{Topological and algebraic characterization of predictability}
In this section we provide a complete characterisation (if-and-only-if conditions) of those sets $\C$ for which there exists a predictor  that attains an asymptotically vanishing total-variation error on every measure from $\C$. These conditions are of two kinds: topological (separability) and algebraic (in terms of the dominance of measures). In both characterisations,  a mixture predictor naturally comes up as a solution.

Before proceeding to characterisations,  let us briefly  recall some facts we know about $\getv$; details can be found, for example, in \cite{Plesner:46}. 
Let $[\P]_{tv}$ denote the set of equivalence classes of $\P$ with respect
to $\getv$, and for $\mu\in  \P_{tv}$ denote $[\mu]$ the equivalence class that contains $\mu$.
 Two elements $\sigma_1,\sigma_2\in[\P]_{tv}$ (or  $\sigma_1,\sigma_2\in\P$) are called disjoint (or singular) if there is no $\nu\in [\P]_{tv}$ such that 
$\sigma_1\getv\nu$ and $\sigma_2\getv\nu$; in this case we write $\sigma_1\perp_{tv}\sigma_2$. 
We write $[\mu_1] +[\mu_2]$ for  $[{1\over2}(\mu_1 + \mu_2)]$.
Every pair $\sigma_1,\sigma_2\in[\P]_{tv}$ has a supremum $\sup(\sigma_1,\sigma_2)=\sigma_1+\sigma_2$.
Introducing into $[\P]_{tv}$ an extra element $0$ such that $\sigma\getv0$ for all $\sigma\in[\P]_{tv}$, 
we can state that for every $\rho,\mu\in  [\P]_{tv}$ there exists a unique pair of elements $\mu_s$ and $\mu_a$
such that $\mu= \mu_a + \mu_s$, $\rho\ge\mu_a$ and $\rho\perp_{tv}\mu_s$. (This is a form of Lebesgue's decomposition.) Moreover, $\mu_a=\inf (\rho,\mu)$. 
Thus, every pair of elements has a supremum and an infimum. 
 Moreover, every bounded set of disjoint elements of $[\P]_{tv}$ is at most countable.

The following lemma, which is an  easy consequence of  \cite{Blackwell:62}, shows that for every pair of measures $\mu$ and $\rho$, if the prediction quality is measured in total variation then either $\rho$ predicts $\mu$ very well (the asymptotic loss is 0) or not at all (the asymptotic loss is 1). 
From this we shall see (Theorem~\ref{th:tv}) that the realizable and the non-realizable cases of the sequence prediction problem essentially collapse to one. 
\begin{lemma}[0-1 law for prediction in total variation]\label{th:01} Let $\mu, \rho$ be two process measures. Then $v(\mu,\rho,x_{1..n})$ converges
to either 0 or 1 with $\mu$-probability~1. 
\end{lemma}
\begin{proof} Assume that $\mu$ is not absolutely continuous with respect to $\rho$ (the other
case is covered by Theorem~\ref{th:bd}).  By the Lebesgue decomposition theorem, the measure $\mu$ admits a representation $\mu=\alpha \mu_a + (1-\alpha)\mu_s$ where 
$\alpha\in[0,1)$ and the measures  $\mu_a$ and $\mu_s$ are such that $\mu_a$ is absolutely continuous with 
  respect to $\rho$ and $\mu_s$ is singular with respect to $\rho$. 
Let $W$ be such a set that $\mu_a(W)=\rho(W)=1$ and $\mu_s(W)=0$.
  Note that we can take $\mu_a=\mu|_{W}$ and $\mu_s=\mu|_{\X^\infty\backslash W}$.
From Theorem~\ref{th:bd} we have $v(\mu_a,\rho,x_{1..n})\to0$ $\mu_a$-a.s., as well as $v(\mu_a,\mu,x_{1..n})\to0$
$\mu_a$-a.s.\ and $v(\mu_s,\mu,x_{1..n})\to0$ $\mu_s$-a.s. Moreover, $v(\mu_s,\rho,x_{1..n})\ge |\mu_s(W|x_{1..n})-\rho(W|x_{1..n})|=1$
so that $\v(\mu_s,\rho,x_{1..n})\to1$ $\mu_s$-a.s. Furthermore,
$$
 v(\mu,\rho,x_{1..n})\le v(\mu,\mu_a,x_{1..n})+v(\mu_a,\rho,x_{1..n})=I
$$
and 
$$
 v(\mu,\rho,x_{1..n})\ge - v(\mu,\mu_s,x_{1..n})+v(\mu_s,\rho,x_{1..n})=II.
$$
We have $I\to0$ $\mu_a$-a.s.\ and hence $\mu|_W$-a.s., as well as  $II\to 1$ $\mu_s$-a.s.\ and hence $\mu|_{\X^\infty\backslash W}$-a.s.
Thus, $\mu(v(\mu,\rho,x_{1..n})\to 0\text{ or }1)\le \mu(W)\mu|_W(I\to0)+\mu(\X^\infty\backslash W)\mu|_{\X^\infty\backslash W}(II\to1)=\mu(W)+\mu(\X^\infty\backslash W)=1$,
which concludes the proof. 
\qed\end{proof}

The following theorem establishes a topological characterization of those sets $\C$ for which a predictor exists (realizable case).

\begin{theorem}\label{th:sep1}
 Let $\C$ be a set of probability measures on $(\X^\infty,\mathcal F)$. There is a measure $\rho$ such that $\rho$ predicts every $\mu\in\C$ in 
total variation if and only if $\mathcal C$ is separable with respect to the topology of total variation distance.
In this case, any measure $\nu$ of the form $\nu=\sum_{k=1}^\infty w_k\mu_k$, where $\{\mu_k: k\in\N\}$ is any dense countable subset of $\C$ and $w_k$ are any positive
weights that sum to 1,  predicts every $\mu\in\C$ in total variation.
\end{theorem}
\begin{proof}
{\em Sufficiency and the mixture predictor.} Let $\C$ be separable in total variation distance, and let $\mathcal D=\{\nu_k:k\in\N\}$ be its dense countable subset.
We have to show that $\nu:=\sum_{k\in\N}w_k\nu_k$, where $w_k$ are any positive real weights that sum to 1, predicts every $\mu\in\C$ in total variation.
To do this, it is enough to show that $\mu(A)>0$ implies $\nu(A)>0$ for every $A\in\mathcal F$ and every $\mu\in\C$. Indeed, 
let $A$ be such that $\mu(A)=\epsilon>0$. Since $\mathcal D$ is dense in $\C$, there is a $k\in\N$ such that $v(\mu,\nu_k)<\epsilon/2$.
Hence $\nu_k(A)\ge\mu(A)-v(\mu,\nu_k)\ge \epsilon/2$ and $\nu(A)\ge w_k\nu_k(A)\ge w_k\epsilon/2>0$.

{\em Necessity.}
 For any $\mu\in\C$, since $\rho$ predicts $\mu$ in total variation, $\mu$ has a density (Radon-Nikodym derivative) $f_\mu$ with respect 
to $\rho$.
We can define $L_1$ distance with respect to $\rho$ as  $L_1^\rho(\mu,\nu)=\int_{\X^\infty}|f_{\mu}-f_\nu|d\rho$.
The set of all measures that have a density with respect to $\rho$, is separable with respect to this distance 
(for example, a dense countable subset can be constructed based on measures whose densities are step-functions,   that take only rational values, 
see, e.g., \cite{Kolmogorov:75});
therefore, its subset $\C$ is also separable. Let $\mathcal D$ be any dense countable  subset of $\C$.
Thus, for every $\mu\in\C$ and every $\epsilon$ there is a $\mu'\in \mathcal D$ such that $L_1^\rho(\mu,\mu')<\epsilon$.
For every measurable set  $A$ we have 
$$
|\mu(A)-\mu'(A)|=\left|\int_A f_\mu d\rho -\int_A f_{\mu'}d\rho\right|\le \int_A|f_\mu-f_{\mu'}|d\rho\le\int_{\X^\infty}|f_\mu-f_{\mu'}|d\rho<\epsilon.
$$
Therefore, $v(\mu,\mu')=\sup_{A\in\mathcal F}|\mu(A)-\mu'(A)|<\epsilon$, and the set $\C$ is separable in total variation distance.
\qed\end{proof}

Finally, including the non-realizable case into the picture, all the preceding results can be combined into the following theorem, which characterizes predictability in terms of the total variation distance.

\begin{theorem}\label{th:tv} Let  $\mathcal C\subset\mathcal P$. The following statements about  $\C$ are equivalent.
\begin{itemize}
  \item[(i)]\ The exists a predictor $\rho$ that predicts every measure in $\C$ in total variation: $l_{tv}(\mu,\rho)=0$ for every $\mu\in\C$.
  \item[(ii)]\ \ There exists a predictor $\rho$  that predicts every measure as well as the best predictor in $\C$ for this measure: $r_{tv}(\C,\rho)=0.$
 \item[(iii)]\ \   $\C$ is upper-bounded with respect to $\getv$.
 \item[(iv)]\ \  There exists a sequence $\mu_k\in\C$, $k\in\N$ such that for some (equivalently, for every) sequence of  weights $w_k\in(0,1]$, $k\in\N$
such that $\sum_{k\in\N}w_k=1$, the measure $\nu=\sum_{k\in\N} w_k\mu_k$ satisfies $\nu\getv\mu$ for every $\mu\in\C$.
 \item[(v)]  $\C$ is separable with respect to the total variation distance.
 \item[(vi)]\ \   Let $ \C^+:=\{\mu\in\mathcal P: \exists \rho\in\C\, \rho\getv\mu\}$. Every disjoint (with respect to $\getv$) subset of $\C^+$ is at most countable.
\end{itemize}
Moreover, predictors in the statements (i) and (ii) are interchangeable, and can be taken to be any upper bound for $\C$.
The sequence $\mu_k$ in the statement (iv)
 can be taken to be any dense (in the total variation distance) countable subset of $\C$ (cf. (v)), or any maximal disjoint (with respect to $\getv$) subset of $\C^+$ of statement
(vi), in which every measure that is not in $\C$ is replaced by any measure from $\C$ that dominates it. %
\end{theorem}
\begin{proof}
The implication $(ii)\Rightarrow(i)$ is obvious. %
The implication $(iii)\Rightarrow(i)$ as well as its  converse (and  hence also $(iv)\Rightarrow(iii)$) are a reformulation of  Theorem~\ref{th:bd}.
 The implication $(i)\Rightarrow(ii)$ follows from the transitivity of $\getv$  and from Lemma~\ref{th:01}  (the 0-1 law for prediction in total variation): indeed, from  Lemma~\ref{th:01} we have $l_{tv}(\nu,\mu)=0$
if $\mu\ge_{tv}\nu$ and $l_{tv}(\nu,\mu)=1$ otherwise. From this and the transitivity of $\ge_{tv}$ it follows that 
 if $\rho\ge_{tv}\mu$ then also $l_{tv}(\nu,\rho)\le l_{tv}(\nu,\mu)$ for all $\nu\in\mathcal P$. %
  The equivalence
of $(iv)$, $(v)$, and $(i)$ follows from Theorem~\ref{th:pq3-1}. The equivalence of $(iii)$ and $(vi)$ was proven in \cite{Plesner:46}.
The concluding statements of the theorem are easy to demonstrate from the  results cited above.
\qed\end{proof}

\noindent{\bf Remark.} Using Lemma~\ref{th:01} we can also define {\em expected} (rather than almost sure) total variation 
loss of $\rho$ with respect to $\mu$, as the $\mu$-probability that $v(\mu,\rho)$ converges to~1: 
$$
l'_{tv}(\mu,\rho):=\mu\{x_1,x_2,\dots\in\X^\infty:v(\mu,\rho,x_{1..n})\to1\}.
$$
The non-realizable case of the sequence prediction problem can be reformulated
 for this notion of loss. However, it is easy to see that for this reformulation Theorem~\ref{th:tv} holds true as well.

Thus, we can see that, for the case of prediction in total variation, all the sequence prediction problems formulated reduce to studying 
the relation of absolute continuity for process measures and those families of measures that are absolutely continuous (have a density) with 
respect to some measure (a predictor). 
 On the one hand, from a statistical 
point of view  such families
are rather large: the assumption that the probabilistic law in question has a density with respect to some (nice) measure
is  a standard one in statistics. It should also be mentioned that  such families can easily be uncountable. (In particular, 
this means that they are large from a computational point of view.) 
On the other hand, even such basic examples as the set of all Bernoulli i.i.d.\ measures does not allow for a predictor 
that predicts every measure in total variation.

 \section{Examples}
A couple of simple  examples illustrate the results of this chapter.
\subsection{ Example: countable classes of measures.} A very simple but  rich  example is obtained by taking any countable family $\C=\{\mu_k: k\in\N\}$ of measures. 
In this case, any mixture predictor $\rho:=\sum_{k\in\N} w_k\mu_k$ predicts all $\mu\in\C$ both in total variation. A particular instance that has gained much attention in the literature is the 
family of all computable measures. Although countable, this family of processes is rather rich. The problem 
of predicting all computable measures was introduced in \cite{Solomonoff:78}, where the Zvonkin-Levin measure \cite{Zvonkin:70}~--- a mixture predictor~--- was proposed as a predictor; see also \cite{Hutter:04uaibook} for a review.

\subsection{ Example: Bernoulli i.i.d.\ processes.} Consider the class $\C_B=\{\mu_p: p\in[0,1]\}$ of 
all Bernoulli i.i.d.\ processes: $\mu_p(x_k=0)=p$ independently for all $k\in\N$. Clearly, this family is uncountable.
Moreover, each set $$T_p:=\{x\in \X^\infty:\text{ the limiting fraction of 0s in $x$ equals }p\},$$ 
 has probability 1 with respect 
to $\mu_p$ and probability 0 with respect to any $\mu_{p'}$ with $p'\ne p$. Since the sets $T_p$, $p\in[0,1]$ are non-overlapping, there is no measure $\rho$ for which $\rho(T_p)>0$ for all $p\in[0,1]$.
That is, there is no measure $\rho$ with respect to which all $\mu_p$ are absolutely continuous.

\chapter{Prediction in KL-divergence}\label{ch:kl}
\section{Realizable Case: Finite-Time  Optimality of Mixture Predictors}\label{s:main}

Theorem~\ref{th:main}, the main result of this  chapter, shows that for any set $\C$ of measures and any predictor $\rho$ there is a mixture predictor over $\C$ that is as good as $\rho$ up to a $O(\log n)$ loss on any measure from $\C$.
As a corollary, one can obtain an asymptotic result  (Corollary~\ref{cl:main}) establishing that  the minimax loss is always achievable and is achieved by  a mixture predictor. These results are obtained  without any assumptions whatsoever on the set $\C$.  

\subsection{Finite-time  upper bound and asymptotic optimality}
\begin{theorem}[upper bound on the best mixture predictor]\label{th:main}Let $\C$ be any set of probability measures on $(\X^\infty,\mathcal F)$,
and let $\rho$ be another probability measure on this space, considered as a predictor. Then there is a mixture (discrete Bayesian) predictor  $\nu$, 
that is, a predictor of the form $\sum_{k\in\N} w_k \mu_k$ where $\mu_k\in\C$ and $w_k\in[0,1]$, such that for every  $\mu\in\C$ we have
\begin{equation}\label{eq:thm}
 L_n(\mu,\nu)- L_n(\mu,\rho) \le  8 \log n + O(\log\log n),
\end{equation}
 where the constants in $O(\cdot)$ are small and are given in \eqref{eq:close}  using the notation defined in \eqref{eq:M}, \eqref{eq:w}, \eqref{eq:B} and~\eqref{eq:k}. The dependence on the alphabet size, $M$, is linear ($M\log\log n$) and the rest of the constants are universal.
\end{theorem}

The proof is deferred to the end of this Section (section~\ref{s:prthmain}).

\begin{corollary}[asymptotic optimality of mixture predictors]\label{cl:main} For any set  $\C$ of probability measures on $(\X^\infty,\mathcal F)$,
there exist a mixutre (discrete Bayesian) predictor $\phi$ such that 
$$
\bd(\C,\phi)= V_\C.
$$
\end{corollary}
\begin{proof}
Note that the statement does not immediately follow from~\eqref{eq:thm}, because $\rho$ in~\eqref{eq:thm} may be such that $\sup_{\mu\in\C}\bar L(\mu,\rho)>V_\C$.
Thus, let  $\gamma_j>V_\C$, $j\in\N$ be a non-increasing sequence such that $\lim_{j\to\infty}\gamma_j = V_\C$. 
By the definition~\eqref{eq:vc} of $V_\C$, it is possible to find  a sequence $\rho_j\in\P$ such that $\bar L(\C,\rho_j)\le\gamma_j$ for all $j\in\N$.
From Theorem~\ref{th:main} we conclude that for each $\rho_j$, $j\in\N$ there is  a probability measure $\nu_j$ 
of the form $\sum_{k\in\N}w_k'\mu_k$, where $\mu_k\in\C$
such that  $\bar L(\C,\nu_j)\le \bar L(\C,\rho_j)$. 
It remains to define $\phi:=\sum_{j\in\N} w_j\nu_j$, where $w_j$ are positive and sum to 1. Clearly, $\phi$ is a discrete Bayesian predictor. Let us show that for every $j\in\N$ it satisfies
 \begin{equation}\label{eq:nuro2}
   \bar L(\C,\phi)\le \bar L(\C,\rho_j).
 \end{equation}

Indeed, for every $\mu\in\C$ and every $j\in\N$
\begin{equation*}
L_n(\mu,\phi)=E_\mu\log\frac{\mu(x_{1..n})}{\phi(x_{1..n})}\le E_\mu\log\frac{\mu(x_{1..n})}{\nu_j(x_{1..n})}-\log w_j,
\end{equation*}
so that $\bar L(\mu,\phi)\le \bar L(\mu,\nu_j)\le \bar L(\mu,\rho_j)\le\gamma_j$, establishing~\eqref{eq:nuro2}. Finally, recall that $\gamma_j\to V_\C$ to obtain the statement of the corollary.
\qed\end{proof}

\subsection{Lower Bound}\label{s:lb}
In this section we establish a lower bound on using a mixture predictor, whether discrete or not, complementing the upper bound of Theorem~\ref{th:main}. The bound leaves a significant gap with respect to the upper bound, but it  shows that the regret of using a Bayesian predictor even with the {\em best} prior for the given set $\C$  cannot be upper-bounded by a constant.
\begin{theorem}\label{th:lb}
 There exists a measurable set of probability measures $\C$ and a probability measure $\rho$, such that for every Bayesian predictor $\nu$ whose prior is concentrated on $\C$, there exists a function $\theta(n)$ which is non-decreasing and goes to infinity with $n$, there exist infinitely many time steps $n_i$ and measures $\mu_i\in\C$ such that $L_{n_i}(\mu_i,\nu)-L_{n_i}(\mu_i,\rho)\ge \theta(n_i)$ for all $i\in\N$.
\end{theorem}
Thus, the lower bound goes to infinity with $n$ but may do so arbitrarily slow. This leaves a gap with respect to the $O(\log n)$ upper bound of Theorem~\ref{th:main}. %

Note that the  theorem compares the regret of the (best) Bayesian with respect to using the best predictor for $\C$~--- but not with using the best predictor for each $\mu\in\C$, which is always $\mu$ itself. 
Note also that this formulation is good enough to be the opposite of Theorem~\ref{th:main}, because the formulation of the latter is strong: Theorem~\ref{th:main} says that {\em for every $\mu$ and for every $n$} (the regret is upper bounded), so, in order to counter that, it is enough to say that {\em there exists $n$ and there exists $\mu$} (such that the regret is lower bounded);  Theorem~\ref{th:lb} is, in fact, a bit stronger, since it establishes that there are  infinitely many such $n$.  However, it does not preclude that for every fixed measure $\mu$ in $\C$ the loss of the Bayesian is upper-bounded by a constant independent of $n$ (but dependent on $\mu$), while the loss of $\rho$ is linear in $n$. This is actually the case in the proof. 
\begin{proof}
Let $\X:=\{0,1\}$. Let $\C$ be the set of Dirac delta measures, that is,  the probability measures  each of which is concentrated on a single deterministic sequence, where the sequences are all the sequences that are 0 from some $n$ on. In particular, introduce 
$S_n:=\{ x_{1,2,\dots}\in\X^\infty : x_i=0\text{ for all }i>n\}$, $S:=\cup_{n\in\N}S_n$. Let $\C_n$ be the set of all probability measures $\mu$ such that $\mu(x)=1$ for some $x\in S_n$ and let $\C:=\cup_{n\in\N}C_n$.   

Observe that the set $\C$ is countable. It is, therefore, very easy to construct a (Bayesian) predictor for this set: enumerate it in any way, say $(\mu_k)_{k\in\N}$ spans all of $\C$, fix a sequence of positive weights $w_k$ that sum to 1,  and let 
\begin{equation}\label{eq:bn}
\nu:=\sum_{k\in\N}w_k \mu_k. 
\end{equation}
Then $L_n(\mu_k,\nu)\le-\log w_k$ for all $k\in\N$. That is, for every $\mu\in\C$ the loss of $\nu$ is upper-bounded by a constant: it depends on $\mu$ but not on the time index $n$.  So, it is good for every $\mu$ for large $n$, but may be bad for some $\mu$ for (relatively) small $n$, which is what we shall exploit.

Observe that, since $\C$ is countable, every Bayesian $\nu$ with its prior over $\C$  must have, by definition, the form~\eqref{eq:bn} for some weights $w_k\in[0,1]$ and some measures $\mu_k\in\C$.  Thus, we fix any Bayesian $\nu$ in this form.

Define $\rho$ to be the Bernoulli i.i.d.\ measure with the parameter 1/2. Note that 
\begin{equation}\label{eq:rho}
L_n(\mu,\rho)=n  
\end{equation}
for every $n$. This is quite a useless predictor; its asymptotic average error is the worst possible, 1.  However, it is minimax optimal for every single time step $n$:
$$
\inf_{\rho'}\sup_{\mu\in\C} L_n(\mu,\rho') = n,  
$$
where the $\inf$ is over all possible probability measures. This is why $\rho$ is hard to compete with--- and, incidentally, why being minimax optimal for each $n$ separately may be useless.

For each $s\in\N$, let $W_s$ be the weight that $\nu$ spends on the measures in the sets $\C_k$ with $k<s$, and let $M_s$ be the set of these measures: 
$$
W_s:=\sum \{w_i: \exists k<s\text{ such that }\mu_i\in \C_k\},
$$
and 
$$
M_s:= \{\mu_i: \exists k<s\text{ such that }\mu_i\in \C_k\}.
$$

By construction, 
\begin{equation}\label{eq:to1}
  \lim_{s\to\infty}W_s=1.
\end{equation}
Next, for each $n\in\N$, let $U_n:=S_{n+1}\setminus S_n$ (these are all the sequences in $S_{n+1}$ with $1$ on the $n$th position). Note that $\mu(U_{n})=0$ for each $\mu\in M_n$, while $|U_n|=2^n$. 
From the latter equality, there exists $x_{1..n}\in\X^{n}$ and $\mu\in U_n\subset S_{n+1}$ such that 
$$ 
 \mu(x_{1..n}=1)\text{ but }\nu(x_{1..n})\le 2^{-n}(1-W_s).
$$
This, \eqref{eq:to1} and~\eqref{eq:rho} imply the statement of the theorem.
\qed\end{proof}

\subsection{Examples}\label{s:exkl}
In this section we consider examples of various sets of processes and the corresponding mixture predictors that are optimal for these sets. The focus is on  asymptotic optimality rather than finite-time bounds. The examples considered can be split  in two groups: the sets of processes for which asymptotically vanishing loss is achievable, and those for which it is not. The former group is rather large, since convergence in expected average KL divergence is  rather week (especially compared to convergence in total variation), and includes, in particular, all stationary ergodic processes. The other group includes sets of processes such as piece-wise i.i.d.\ processes and processes where only some aspects are predictable while others are arbitrary.
\subsubsection{Examples of sets of processes with $V_\C=0$.}\label{s:exvc0}
The case of $V_\C=0$, that is, the case where asymptotically vanishing prediction error may be obtained, is the classical case of the sequence prediction problem, and thus includes the most-studied examples, starting with Bernoulli i.i.d.\ processes. Here we omit the case of countable sets of processes considered in the previous section, as it carries over intact.

{\noindent \bf Bernoulli i.i.d.\ processes.} Consider the class $\C_B=\{\mu_p: p\in[0,1]\}$ of 
all Bernoulli i.i.d.\ processes: $\mu_p(x_k=0)=p$ independently for all $k\in\N$. As we have seen, there is no measure $\rho$ with respect to which all $\mu_p$ are absolutely continuous, and so there is no predictor that would predict any $\mu\in\C_B$  in total variation. However, it is easy to see from the law of large numbers (see also \cite{Krichevsky:93}) that the Laplace predictor~(\ref{eq:lapl}) predicts
every Bernoulli i.i.d. process in expected average KL divergence (and not only).
Hence, Theorem~\ref{th:pq3-1} implies that there is a countable mixture predictor for this family too.
Let us find such a predictor. Let $\mu_q:q\in Q$ be the family of all Bernoulli i.i.d. measures with rational probability of 0, 
and let $\rho:=\sum_{q\in Q} w_q\mu_q$, where $w_q$ are arbitrary positive weights that sum to 1. Let $\mu_p$ be  any Bernoulli i.i.d.
process. Let $h(p,q)$ denote the divergence $p\log(p/q)+(1-p)\log(1-p/1-q)$. For each $\epsilon$ we can find a $q\in Q$ such that $h(p,q)<\epsilon$. Then 
\begin{multline}\label{eq:bern}
 {1\over n} L_n(\mu_p,\rho)={1\over n}\E_{\mu_p}\log\frac{\log\mu_p(x_{1..n})}{\log\rho(x_{1..n})}\le {1\over n}\E_{\mu_p}\log\frac{\log\mu_p(x_{1..n})}{w_q\log\mu_q(x_{1..n})} 
\\=-\frac{\log w_q}{n} +  h(p,q)\le \epsilon+ o(1). 
\end{multline}
Since this holds for each $\epsilon$, we conclude that ${1\over n} L_n(\mu_p,\rho)\to0$ and $\rho$ predicts every $\mu\in\C_B$ in expected average KL divergence.

{\noindent Stationary processes.} In \cite{BRyabko:88}  a predictor $\rho_R$ was constructed which predicts every stationary 
process $\rho\in\C_S$ in expected average KL divergence. (This predictor is obtained as a mixture of
predictors for $k$-order Markov sources, for all $k\in\N$.) Therefore, Corollary~\ref{cl:main} implies that there is also a countable mixture predictor
for this family of processes. Such a predictor can be constructed as follows (the proof in this example is based on the proof in  \cite{BRyabko:06a}, Appendix~1).
 Observe that the family $\C_k$ of
 $k$-order stationary binary-valued Markov
processes is parametrized by $2^k$ $[0,1]$-valued parameters: probability of observing 0 after observing $x_{1..k}$, for each $x_{1..k}\in\X^k$.
For each $k\in\N$ let $\mu^k_q$, $q\in Q^{2^k}$ be the (countable) family of all stationary $k$-order Markov processes with rational 
values of all the parameters. We will show that any predictor $\nu:=\sum_{k\in\N}\sum_{q\in Q^{2^k}} w_k w_q \mu^k_q$, where $w_k$, $k\in\N$ and
$w_q, q\in Q^{2^k}$, $k\in\N$ are any sequences of positive real weights that sum to 1, predicts every stationary $\mu\in\C_S$ in expected average 
KL divergence. 
For $\mu\in\C_S$ and $k\in\N$ define the $k$-order conditional Shannon entropy $h_k(\mu):=-\E_\mu\log\mu(x_{k+1}|x_{1..k})$.
We have $h_{k+1}(\mu)\ge h_k(\mu)$ for every $k\in\N$ and $\mu\in\C_S$, and the limit 
\begin{equation}\label{eq:hlim}
h_\infty(\mu):=\lim_{k\to\infty}h_k(\mu)
\end{equation}
is called the limit Shannon entropy; see, for example, \cite{Gallager:68}.
Fix some $\mu\in\C_S$. It is easy to see  that for  every $\epsilon>0$ and every $k\in\N$ we can find a $k$-order stationary
 Markov measure $\mu^k_{q_\epsilon}$, 
$q_\epsilon\in Q^{2^k}$ with rational values of the parameters, such that 
\begin{equation}\label{eq:muq}
\E_\mu\log\frac{\mu(x_{k+1}|x_{1..k})}{\mu^k_{q_\epsilon}(x_{k+1}|x_{1..k})}<\epsilon.
\end{equation}
We have 
\begin{multline}\label{eq:must}
 {1\over n} L_n(\mu,\nu) \le -\frac{\log w_k w_{q_\epsilon}}{n} + {1\over n} L_n(\mu,\mu^k_{q_\epsilon})\\ 
 =O(k/n) + {1\over n} \E_\mu\log\mu(x_{1..n}) - {1\over n} \E_\mu\log\mu^k_{q_\epsilon}(x_{1..n}) \\ 
= o(1)+ h_\infty(\mu) - {1\over n} \E_\mu\sum_{k=1}^n\log\mu^k_{q_\epsilon}(x_t|x_{1..t-1})
 \\ =  o(1)+ h_\infty(\mu) -  {1\over n} \E_\mu\sum_{t=1}^{k}\log\mu^k_{q_\epsilon}(x_t|x_{1..t-1})  - \frac{n-k}{n}\E_\mu\log\mu^k_{q_\epsilon}(x_{k+1}|x_{1..k})
\\ \le o(1)+ h_\infty(\mu)- \frac{n-k}{n}(h_k(\mu)- \epsilon),
\end{multline}
where the first inequality is derived analogously to~(\ref{eq:bern}), the first equality follows from~(\ref{eq:kl}), the second
equality follows from the Shannon-McMillan-Breiman theorem (e.g., \cite{Gallager:68}), that states that ${1\over n}\log\mu(x_{1..n})\to h_\infty(\mu)$
  in expectation (and a.s.)
 for every $\mu\in\C_S$, and~(\ref{eq:kl}); in the third equality we have used the fact that $\mu^k_{q_\epsilon}$ is $k$-order Markov
and $\mu$ is stationary, whereas the last inequality follows from~(\ref{eq:muq}). Finally, since the choice of $k$ and $\epsilon$ was arbitrary, 
from~(\ref{eq:must}) and~(\ref{eq:hlim}) we
obtain $\lim_{n\to\infty}{1\over n} L_n(\mu,\nu)=0$.

{\noindent \bf Example: weights may matter.} An important difference in the formulations of Theorem~\ref{th:pq3-1} (optimality of mixture prediction for prediction  in total variation) and Corollary~\ref{cl:main} (KL divergence) is that in the former the weights may be arbitrary, since the choice of weights clearly does not affect absolute continuity, but in the latter weights may matter. This can be illustrated by the following example. We will construct a sequence of measures $\nu_k, k\in\N$, a measure $\mu$, and two sequences of positive weights $w_k$ and $w_k'$
with $\sum_{k\in\N} w_k= \sum_{k\in\N} w_k'=1$, for which $\nu:=\sum_{k\in\N}w_k\nu_k$ predicts $\mu$ in expected average KL divergence, but $\nu':=\sum_{k\in\N}w'_k\nu_k$
does not. Let 
$\nu_k$ be a deterministic measure that first outputs $k$ 0s and then only 1s, $k\in\N$.  
Let $w_k=w/k^2$ with $w=6/\pi^2$ and $w_k'=2^{-k}$. Finally, let $\mu$ be a deterministic measure that outputs only 0s. 
We have $L_n(\mu,\nu)=- \log (\sum_{k\ge n} w_k)\le-\log(wn^{-2})=o(n)$, but $L_n(\mu,\nu')=- \log (\sum_{k\ge n} w'_k)=-\log (2^{-n+1})=n-1\ne o(n)$,
proving the claim.

\subsubsection{Examples of sets of processes with $V_\C>0$.}\label{s:vcpos}
The case of $V_\C>0$ is usually shunned by the probabilistic literature on sequence prediction, due, perhaps, to its generality. Where it is considered it is usually in the context of somewhat different settings, where part of the process is assumed non-probabilistic, and thus the formulations are expert-advice-like or mixed (see processes with abrupt changes below).

{\noindent \bf Typical Bernoulli 1/3 Sequences}
We start with an   example which is somewhat artificial, but comes up as a component in more realistic  cases.
 Take the binary $\X$ and consider all sequences $\x\in\X^\infty$ such 
that the limiting number of 1s in $\x$ equals $1/3$. Denote the set of these sequences $S$ and let the set $\C$ consist of all Dirac measures concentrated on sequences from $S$.   Observe that the Bernoulli i.i.d.\ measure $\delta_{1/3}$ with 
probability $1/3$ of 1 predicts measures in $\C$ relatively well:  $\bar L(\C,\delta_{1/3})=h(1/3)$,
where $h$ stands for the binary entropy, and this is also the minimax loss for this set, $V_\C$. 
It might then appear surprising that this loss is achievable by a combination of  countably many measures from $\C$, which
 consists only of deterministic  measures. Let us try to see what such a combination may look like. 
By definition, for any sequence $\x\in S$ and every $\epsilon$ we can find $n_\epsilon(\x)\in\N$ such that, for all $n\ge n_\epsilon(\x)$,
the average number of 1s in $x_{1..n}$ is within $\epsilon$ of $1/3$.  Fix the sequence of indices $n_j:=2^j$, $j\in\N$ and the sequence
of thresholds   $\epsilon_l:=2^{-l}$.  For each $j$ let ${S'}_j^l\subset S$  be the  set of all sequences $\x\in S$ such that $n_{\epsilon_l}(\x)<n_j$. Select then a finite subset $S_j^l$ of ${S'}_j^l$ such that for each $\x'\in {S'}_j^l$ there is $\x\in S$ such 
that $x'_{1..n_j}=x_{1..n_j}$. This is possible, since the set $\X^{n_j}$ is finite. Now for each $\x\in S_j^l$ take 
the corresponding measure $\mu_\x\in\C$ and attach to it the weight $w_lw_j/|S_j^l|$, where, as before, we are using 
the weights $w_k=w/k\log^2k$. Taking these measures for all $j,l\in\N$, we 
obtain our convex combination.  Of course, we did not enumerate all sequences in $S$ (or measures in $\C$) this way; but for 
each sequence $\x\in S$ and for each $n$ there is a sequence among those that we did enumerate that coincides with $\x$ up to the index $n$.   One can then use the theory of types \cite{Csiszar:98} to calculate the sizes of the sets $S_j^l$
and to check that the weights we found give the optimal loss we are after; but for the illustrative purposes of this example 
this is already not necessary.

{\noindent \bf Processes with  Abrupt Changes}
Start with a family of  distributions $S$, for which we have a good predictor: for example, take  $S$ to be the set $B$ of all Bernoulli i.i.d.\ processes, or, 
more generally, a set for which $V_S=0$.  The family $\C_\alpha$ parametrized by $\alpha\in(0,1)$ and $S$  is then  the family of all processes constructed as follows:  there is a sequence of indexes $n_i$ such that $X_{n_i..n_{i+1}}$ is distributed according to $\mu_i$ for some $\mu_i\in S$. Take  then all possible sequences $\mu_i$ and all sequences $n_i$ whose limiting  frequency $\limsup_{i\to\infty}{1\over n}\{i:n_i<n\}$  is bounded by $\alpha$,  to obtain our set $\C_{S,\alpha}$. Thus, we have a family of processes with abrupt changes in 
distribution, where between changes the distribution is from $S$, the changes are assumed to have the frequency bounded by $\alpha$ but are otherwise arbitrary.
This example was considered by \cite{Willems:96} for the case $S=B$, with the goal of minimizing the regret with respect to the predictor that knows where the changes
occur (the value $V_\C$ was not considered directly). The  method proposed in the latter work, in fact, is not limited to the case $S=B$, but is general. The algorithm is based on a prior over all possible sequences $n_i$ of changes; between the changes
the optimal predictor for $B$ is used, which is also a Bayesian predictor with a specific prior. The regret obtained is of order $\log n$. Since for Bernoulli processes themselves  the best achievable average loss up to time $n$ is  ${1\over n}({1\over2}\log n +1)$, for the sequence $1..n_t$ it is 
${1\over n_t}\sum_{i=1}^t({1\over 2}\log(n_i-n_{i-1})+1)$, where  $n_0:=1$. By Jensen's inequality, this sum is maximized when all the segments $n_{i}-n_{i-1}$ are of the same length, $1/\alpha$, so the total average loss is upper-bounded by 
 $\alpha(1-{1\over2}\log\alpha)$. This value is also attainable, and thus gives  $V_{\C_{B,\alpha}}$. 
A similar result can be obtained if we replace   Bernoulli processes with Markov processes, but not with an arbitrary $S$ for which $V_S=0$. For example, if we take $S$ to be all finite-memory 
distributions, then the resulting process may be completely unpredictable ($V_\C=1$): indeed, if the memory of distributions $\mu_i$ grows (with $i$) faster than $\alpha n$,
then there is little one can do.  For such sets $S$ one can make the problem amenable by restricting the way the distributions $\mu_i$ are selected, 
for example, imposing an ergodicity-like condition that the average distribution has a limit. Another way (often considered in the literature in 
slightly different settings, see \cite{Gyorgy:12} and references) is to have $\alpha\to0$, although in this case one recovers $V_{\C_S}=0$ provided 
$\alpha$ goes to 0 slowly enough (and, of course, provided $V_S=0$).

{\noindent \bf Predictable Aspects} 
The preceding example can be thought of as an instantiation of the general class of processes in which  some aspects are predictable while 
others are not. Thus, in the considered example changes between the distributions were unpredictable, but between the changes the distributions were predictable. %
Another example of this kind is that of processes predictable on some scales but not on others. 
 Imagine that it is  possible to predict, for example, large fluctuations of the process but not small fluctuations (or the other way around). More formally, consider now an alphabet $\X$ with $|\X|>2$, and
let $Y$ be a partition of $\X$. For any sequence ${x_1,\dots,x_n,\dots}$ there is an associated sequence $y_1,\dots,y_n,\dots$ where 
$y_i$ is defined as $y\in Y$ such that $x_i\in y$.  Here again we can obtain examples of sets $\C$ of processes with $V_\C\in(0,1)$ by restricting the distribution of $y_1,\dots,y_n,\dots$ to a set $B$ with $V_B=0$. The interpretation is that, again, we can model 
the $y$ part (by processes in $B$) but not the rest, which we then allow to be arbitrary.

Yet another example is that of  processes  predictable only after certain kind of events: such as a price drop; or  a rain. At other times, the process is unpredictable: it can, again, be an arbitrary  sequence.   More formally, let a set $A\subset \X^*:=\cup_{k\in\N}\X^k$ be measurable. Consider for each sequence  $\x={x_1,\dots,x_n,\dots}$
  another (possibly finite) sequence $\x'={x'_1,\dots,x'_n,\dots}$ given by $x'_i:=(x_{n_i+1})_{i\in\N}$ where $n_i$ are all indexes such that $x_{1..n_i}\in A$. 
We now form the set $\C$ as the set of all processes $\mu$ such that $\x'$ belongs ($\mu$-a.s.) to some pre-defined set $B$;
 for this set $B$ we may have  $V_B=0$.  This means that we can model what happens after events in $A$~--- by processes in $B$, but not the rest of the times, on which we say the process may be arbitrary. For different $A$ and $B$ we then obtain  examples where $V_\C\in(0,1)$.
In relation to this it is worth mentioning the work \cite{Lattimore:11} which explores the possibility that a Bayesian predictor may fail to predict some subsequences.

\subsection{Proof of Theorem~\ref{th:main}}\label{s:prthmain}
Before giving the proof of the theorem, let us briefly expose the main ideas behind it. 
Assume for a moment  that, for each $\mu\in\C$, the limit $\lim_{n\to\infty}{1\over n}\log\frac{\mu(x_{1..n})}{\rho(x_{1..n})}$ exists 
for $\mu$-almost all $\x=x_1,\dots,x_n,\dots\in\X^\infty$, where $\rho$ is the predictor given to compare to. Then we could define ($\mu$-almost everywhere) the function  $f_\mu(\x)$ whose value at $\x$ equals this limit. Let us call it the ``log-density'' function. 
What we would be looking for thence is to find a countable dense subset of the set of  log-densities of all probability measures from $\C$. The measures $\mu$ corresponding to each log-density in this countable set would then constitute the sequence whose existence the theorem asserts. To find such a dense countable subset we could employ a standard procedure: approximate all log-densities by step functions with finitely many steps. The main technical argument is then to show that, for each level of the step functions,  there are not too many of these functions whose steps are concentrated on different sets of  non-negligible probability,  for otherwise the requirement that $\rho$ attains $V_\C$ would be violated. Here ``not too many'' means  exponentially many with the right exponent (the one corresponding to the step of the step-function with which we approximate the density), 
and ``non-negligible probability'' means a probability bounded away (in $n$) from 0. In reality, what we do instead in the proof is use the step-functions approximation at each time step $n$. Since there are only countably many time steps, the result is still a countable set of measures $\mu$ from $\C$.
Before going  further, note that constructing a predictor for each $n$ does not mean constructing the best predictors up to this time step: in fact, taking a predictor that is minimax optimal up to $n$,  for each $n$,  and summing these predictors up (with weights) for all $n\in\N$ may result in the worst possible predictor overall, and in particular, a one much worse than the predictor $\rho$ given. An example of this behaviour is given in the proof of Theorem~\ref{th:lb} (the lower bound). The objective for each $n$ is different, and it is to approximate  the measure $\rho$ up to this time step with measures from $\C$.   For each $n$, we consider a covering of the set $\X^n$ with subsets, each of which is associated with a measure $\mu$ from $\C$. These latter measures are then those the prior is concentrated on (that is, they are summed up with weights). The covering is constructed as follows. The log-ratio function $\log\frac{\mu(x_{1..n})}{\rho(x_{1..n})}$, where $\rho$ is the predictor whose performance we are trying to match, is approximated with a step function for each $\mu$, and for each size of the step. The cells of the resulting partition are then ordered with respect to their $\rho$ probability. The main part of the proof is then to show that not too many cells are needed to cover the set $\X^n$ this way up to a small probability. Quantifying the ``not too many'' and ``small'' parts results in the final bound.

\begin{proof}[of Theorem~\ref{th:main}.]
Define the weights $w_k$ as follows: $w_1:=1/2$, and, for $k>1$
\begin{equation}\label{eq:w} 
w_k:=w/k\log^2k,
\end{equation}
 where $w$ is the normalizer such that $\sum_{k\in\N}w_k=1$.
Replacing $\rho$ with $1/2(\rho+\delta)$ if necessary, where $\delta$ is the i.i.d.\ probability measure with equal probabilities of outcomes, i.e.\ $\delta(x_{1..n})=M^{-1/n}$ for all $n\in\N, x_{1..n}\in\X^n$,
we shall assume, without loss of generality,
\begin{equation}\label{eq:boundedness}
  -\log\rho(x_{1..n})\le nM+1  \text{ for all $n\in\N$
and $x_{1..n}\in\X^n$}.                      
\end{equation}
The replacement is without loss of generality as it adds at most $1$ to the final bound (to be accounted for).
 Thus, in particular, 
\begin{equation}\label{eq:boundedness33}
L_n(\mu,\rho)\le nM+1\text{ for all $\mu$}. 
 \end{equation}

The first part of the proof is the following covering construction. %

For each $\mu\in\C$, $n\in\N$ define the sets 
\begin{equation}\label{eq:tt}
T_{\mu}^n:=\left\{x_{1..n}\in \X^n:  \frac{\mu(x_{1..n})}{\rho(x_{1..n})}\ge {1\over n}\right\}.
\end{equation} 
From Markov inequality, we obtain 
\begin{equation}\label{eq:tm}
\mu(\X^n\backslash T_{\mu}^n)\le 1/n.
\end{equation}

For each $k>1$ let $U_k$ be the partition of  $[-\frac{\log n}{n},M+{1\over n}]$ into   $k$ intervals defined as follows.
 $U_k:=\{u_k^i:i=1..k\}$, where
\begin{equation}\label{eq:cover}
 u^i_k=\left\{ \begin{array}{ll}
                \left[-\frac{\log n}{n},{iM\over k}\right] & i=1, \\
 		\left(\frac{(i-1)M}{k},\frac{iM}{k}\right] & 1<i<k, \\
 		\left(\frac{(i-1)M}{k},M+\frac{1}{n}\right] & i=k.
               \end{array}
  \right.
\end{equation}
Thus, $U_k$ is a partition of $[0,M]$ into $k$ equal intervals but for  some  padding that we added to the leftmost and the rightmost intervals:
on the left we added $[-\frac{\log n}{n},0)$ and on the right $(M,M+1/n]$.

For each $\mu\in\C$, $n,k>1$, $i=1..k$ define the sets 
\begin{equation}\label{eq:t}
T_{\mu,k,i}^n:=\left\{x_{1..n}\in \X^n: {1\over n}\log \frac{\mu(x_{1..n})}{\rho(x_{1..n})}\in u_k^i\right\}.
\end{equation} 
Observe that,  for every  $\mu\in\C,$ $k,n>1$, these sets  constitute a partition of %
$T_{\mu}^n$ into $k$ disjoint sets: indeed, on the left we have ${1\over n}\log \frac{\mu(x_{1..n})}{\rho(x_{1..n})}\ge -{1\over n}\log n$ 
 by definition~\eqref{eq:tt}
of $T_{\mu}^n$, and on the right we have ${1\over n}\log \frac{\mu(x_{1..n})}{\rho(x_{1..n})}\le M+1/n$ from~\eqref{eq:boundedness}.
In particular, from this definition,    %
for all $x_{1..n}\in T^n_{\mu,k,i}$ we have
 \begin{equation}\label{eq:cops1}
 \mu(x_{1..n}) \le 2^{{iM\over k}n+1}\rho(x_{1..n}). %
 \end{equation}
 For every $n,k\in\N$ and $i\in\{1..k\}$ consider the following construction. 
 Define 
$$m_{1}:=\max_{\mu\in\C}\rho(T_{\mu,k,i}^n)$$ (since $\X^n$ are finite all suprema are reached). 
 Find any $\mu_{1}$ such that $\rho(T_{\mu_1,k,i}^n)=m_{1}$ and let
$T_{1}:=T^n_{\mu_1,k,i}$. For $l>1$, let 
$$m_{l}:=\max_{\mu\in\C}\rho(T_{\mu,k,i}^n\backslash T_{l-1}).$$
 If $m_l>0$, let $\mu_l$ be any $\mu\in\C$ such 
that $\rho(T^n_{\mu_l,k,i}\backslash T_{l-1})=m_l$, and let $T_l:=T_{l-1}\cup T^n_{\mu_l,k,i}$; otherwise let $T_l:=T_{l-1}$ and $\mu_l:=\mu_{l-1}$. 
Note that,  for each $x_{1..n}\in T_{l}$ there is $l'\le l$ such that $x_{1..n}\in T^n_{\mu_{l'},k,i}$ and thus from~\eqref{eq:t} we get      %
 \begin{equation}\label{eq:cops2}
  2^{{(i-1)M\over k}n-\log n}\rho(x_{1..n})\le \mu_{l'}(x_{1..n}).  %
 \end{equation}
 Finally, define 
\begin{equation}\label{eq:nun}
\nu_{n,k,i}:=\sum_{l=1}^{\infty} w_l\mu_l.
\end{equation}
(Notice that 
for every $n,k,i$ there is only a finite number of positive $m_l$,
since the set $\X^n$ is finite; thus the sum in the last definition is effectively finite.)
Finally, define the predictor $\nu$ as 
\begin{equation}\label{eq:nu}
\nu:={1\over2}\sum_{n,k\in\N}w_nw_k{1\over k}\sum_{ i=1}^k\nu_{n,k,i} +{1\over2}r,
\end{equation}
where $r$ is a regularizer defined so as to have for each $\mu'\in\C$ and $n\in\N$
\begin{equation}\label{eq:boundedness2}
  \log\frac{\mu'(x_{1..n})}{\nu(x_{1..n})}\le nM - \log w_n +1  \text{\ \ for all  $x_{1..n}\in\X^n$};                      
\end{equation}
this and the stronger statement  \eqref{eq:boundedness} for $\nu$  can be obtained analogously to the latter inequality in the case the i.i.d.\ measure $\delta$ is in $\C$; otherwise, we  need to define $\nu$ as a combination of probability measures from $\C$ only; 
we postpone this step till the end of this proof. 

Next, let us show that the measure $\nu$ is the predictor whose existence is claimed in the statement. %

Introduce the notation 
$$
L_n|_A(\mu,\nu):=\sum_{x_{1..n}\in A}\mu(x_{1..n})\log\frac{\mu(x_{1..n})}{\rho(x_{1..n})};
$$
with this notation, for any set $A\subset\X^n$ we have
$$
L_n(\mu,\nu)= L_n|_A(\mu,\nu)+L_n|_{\X^n\setminus A} (\mu,\nu).
$$

First we want to show that, for each $\mu\in\C$, for each fixed $k,i$,  the sets $T^n_{\mu,k,i}$ are covered by sufficiently few 
sets  $T_l$, where ``sufficiently few'' is, in fact, exponentially many with the right exponent. 
By definition, for each $n,i,k$ the sets $T_l\backslash T_{l-1}$ are disjoint (for different $l$) and have non-increasing (with $l$) $\rho$-probability. Therefore, $\rho(T_{l+1}\backslash T_{l})\le 1/l$  for all $l\in\N$. Hence, from the definition of $T_l$, we must also have $\rho(T^n_{\mu,k,i}\backslash T_{l})\le 1/l$   for all $l\in\N$.
From the latter inequality and~\eqref{eq:cops1} we obtain 
$$
\mu(T^n_{\mu,k,i}\backslash T_{l})\le{1\over l} 2^{{iM\over k}n+1}.
$$
 Take $l_i:=\lceil{ kn}2^{{iM\over k}n+1}\rceil$ to obtain 
\begin{equation}\label{eq:li}
 \mu(T^n_{\mu,k,i}\backslash T_{l_i})\le {1\over kn}.
\end{equation}

Moreover, for every $i=1..k$, for each  $x_{1..n}\in T_{l_i}$,  %
there is $l'\le l_i$ such that $x_{1..n}\in T^n_{\mu_{l'},k,i}$ and thus   the following chain holds
\begin{multline}\label{eq:fnu}
 \nu(x_{1..n})\ge {1\over2}w_nw_k{1\over k} \nu_{n,k,i} \ge {1\over2}w_nw_k{1\over k} w_{kn\,2^{{iM\over k}n+1}}\mu_{l'}(x_{1..n}) 
\\ 
 \ge {w^3\over 4n^2k^3\log^2n\log^2k (\log k+\log n+1+ nMi/k)^2} 2^{-{iM\over k}n}\mu_{l'}(x_{1..n}) 
\\ 
 \ge {w^3\over 4(M+1)^2n^4k^3\log^2n\log^2k} 2^{-{iM\over k}n}\mu_{l'}(x_{1..n}) 
 \\
 \ge {w^3\over 4(M+1)^2n^5k^3\log^2n\log^2k} 2^{-{M\over k}n}\rho(x_{1..n}) 
= B_n  2^{-{M\over k}n}\rho(x_{1..n}),
\end{multline}
where the first inequality is from~\eqref{eq:nu}, the second from~\eqref{eq:nun} with  $l=l_i$, the third is 
by definition of $w_l$, the fourth uses $i\le k$ for the exponential term,  as well as  $(\log n +\log k)\le n-1$ for $n\ge3$, which will be justified by the choice of $k$ in the following~\eqref{eq:k},
  the fifth inequality   uses~\eqref{eq:cops2},  and the final equality introduces $B_n$ defined as  
\begin{equation}\label{eq:B}
 B_n:={w^3\over 4(M+1)^2n^5k^3\log^2n\log^2k}.
\end{equation}

We have
\begin{equation}\label{eq:br}
 L_n(\mu,\nu)= \left(\sum_{i=1}^kL_n|_{T_{l_i}}(\mu,\nu)\right) +  L_n|_{\X^n\setminus \cup_{i=1}^kT_{l_i} }(\mu,\nu).
\end{equation}
For the first term, from~\eqref{eq:fnu} we obtain
\begin{multline}\label{eq:lmn1}
 \sum_{i=1}^kL_n|_{T_{l_i}}(\mu,\nu)
 \le  \sum_{i=1}^kL_n|_{T_{l_i}}(\mu,\rho) +Mn/k - \log B_n
 \\ = L_n(\mu,\rho)-L_n|_{\X^n\setminus\cup_{i=1}^kT_{l_i}}(\mu,\rho) +Mn/k - \log B_n.
\end{multline}
For the second term in~\eqref{eq:br}, we recall that $T^n_{\mu,k,i}$, $i=1..k$ is a partition of $T^n_\mu$, and  decompose 
\begin{equation}\label{eq:3sets}
 \X^n\setminus \cup_{i=1}^kT_{l_i}\subseteq\left(\cup_{i=1}^k(T^n_{\mu,k,i}\setminus T_{l_i})\right)\cup (\X^n\setminus T^n_\mu).  %
\end{equation}
Next, using~\eqref{eq:boundedness2} and an upper-bound for the $\mu$-probability of each of the two sets in~\eqref{eq:3sets}, namely, \eqref{eq:li} and~\eqref{eq:tm}, as well as $k\ge 1$, we obtain 
\begin{equation}\label{eq:sb}
  L_n|_{\X^n\setminus \cup_{i=1}^kT_{l_i} }(\mu,\nu) \le (nM - \log w_n+1 ) {2\over n}.
\end{equation}

Returning to~\eqref{eq:lmn1},
observe that for every $n\in\N$ and every set $A\subset \X^n$, using Jensen's inequality we can obtain
\begin{multline}\label{eq:jen}
L_n|_A(\mu,\rho)=
\\-\sum_{x_{1..n}\in A}\mu(x_{1..n})\log\frac{\rho(x_{1..n})}{\mu(x_{1..n})}
=  -\mu(A)\sum_{x_{1..n}\in A}{1\over\mu(A)}\mu(x_{1..n})\log\frac{\rho(x_{1..n})}{\mu(x_{1..n})}
\\
\ge -\mu(A)\log{\rho(A)\over\mu(A)} \ge -\mu(A)\log\rho(A) -{1\over2}. 
\end{multline}

Therefore, using~\eqref{eq:boundedness33}, similarly to~\eqref{eq:sb} we obtain
\begin{equation}\label{eq:sb2}
 -L_n|_{\X^n\setminus\cup_{i=1}^kT_{l_i}}(\mu,\rho) \le  (nM +1 ) {2\over n}+{1\over2}.
\end{equation}

Combining~\eqref{eq:br} with~\eqref{eq:lmn1}, ~\eqref{eq:sb} and~\eqref{eq:sb2} we derive 
\begin{equation}\label{eq:close}
 L_n(\mu,\nu)
\le L_n(\mu,\rho)+Mn/k-\log B_n +4M -{2\over n}(\log w_n - 1)+1/2;
\end{equation}
setting 
\begin{equation}\label{eq:k}
 k:=\lceil n/\log\log n\rceil 
\end{equation}
 we obtain the statement of the theorem.

It remains to come back to~\eqref{eq:boundedness2} and define the regularizer $r$ as a combination of measures from $\C$ for this inequality to hold.
For each $n\in\N$, denote
$$
A_n:=\{x_{1..n}\in \X^n: \exists\mu\in\C\ \mu(x_{1..n})\ne0\},
$$ and let, for each $x_{1..n}\in \X^n$, the probability measure $\mu_{x_{1..n}}$ be any probability measure from $\C$ such that $\mu_{x_{1..n}}(x_{1..n})\ge{1\over2}\sup_{\mu\in\C}\mu(x_{1..n})$.
Define 
$$
 r_n'%
   :={1\over |A_n|}\sum_{x_{1..n}\in A_n}\mu_{x_{1..n}}%
$$ for each
 $n\in\N$, and let  $r:=\sum_{n\in\N}w_n r'_n$. 
For every $\mu\in\C$ we have 
$$
r(x_{1..n})\ge w_n|A_n|^{-1} \mu_{x_{1..n}}(x_{1..n})\ge{1\over2} w_n |\X|^{-n} \mu(x_{1..n})
$$ for
every $n\in\N$ and every $x_{1..n}\in A_n$,  establishing~(\ref{eq:boundedness2}).
\qed\end{proof}
\section{Conditions on $\C$ sufficient for vanishing loss ($V_\C=0$)}\label{s:pq3+}
In this section we exhibit some sufficient conditions on the class
 $\C$, under which a predictor for all measures in $\C$ exists in the sense of KL divergence, that is, sets for which $V_\C=0$.
Unlike for total variation, for KL divergence there are no known non-trivial if-and-only-if conditions, neither topological (such as separability) nor algebraic (dominance). Still, some interesting sufficient conditions can be identified. 

 It is important to note that none of these conditions relies on a parametrization of any kind.
The conditions presented are of 
two types: conditions on asymptotic behaviour of measures in $\C$, and on their
local (restricted to first $n$ observations) behaviour. Conditions of the first type
concern separability of $\C$ with respect to the total variation distance and the expected average KL divergence.
The conditions of the second kind concern the ``capacity'' of the sets $\C^n:=\{\mu^n:\mu\in\C\}$, $n\in\N$,
where $\mu^n$ is the measure $\mu$ restricted to the first $n$ observations.
Intuitively, if $\C^n$ is small (in some sense), then prediction
is possible. We measure the capacity of $\C^n$ in two ways. The first way is 
to find the maximum probability given to each sequence $x_1,\dots,x_n$
by some measure in the class, and then take a sum over $x_1,\dots,x_n$.
Denoting the obtained  quantity  $c_n$, one can show that  it grows polynomially in $n$ for 
some important classes of processes, such as i.i.d. or Markov processes.
We show that, in general, if $c_n$ grows subexponentially then a predictor
exists that predicts any measure in  $\C$ in expected average KL divergence.
On the other hand, exponentially growing $c_n$ are not sufficient for prediction. 
 A more refined way to measure the capacity of $\C^n$
is using a concept of channel capacity from information  theory,
which was developed for a closely related problem of finding optimal 
codes for a class of sources. We extend corresponding results from
information theory to show that sublinear growth of channel capacity 
is sufficient for the existence of a predictor, in the sense of expected average divergence.
Moreover, the obtained bounds on the divergence are optimal up to an additive logarithmic term.

\subsection{Separability}\label{s:sep}
Knowing that a mixture of a countable subset gives a predictor if there is one, a notion that naturally comes to mind, 
when trying to characterize families of processes for which  a predictor exists, is separability.
Can we say that there is a predictor for a class $\C$ of measures if and only if $\C$ is separable? 
Of course, to talk about separability we need a suitable topology on the space of all measures, or at least on $\C$. If the formulated 
questions were to have a positive answer, we would need  a different topology for each of the notions 
of predictive quality that we consider. Sometimes these measures of predictive quality indeed define  a nice
enough structure of a probability space, but sometimes they do not. The question  whether there exists a 
topology on $\C$, separability with respect to which is equivalent to the existence of a predictor, is already more vague and less appealing.
As we have seen in the previous chapter (Theorem~\ref{th:sep1}), in the case of total variation distance the suitable topology is that  of total variation distance. 
In the case of expected average KL divergence the situation is different. While one can introduce a topology based on it, 
separability with respect to this topology turns out to be a sufficient but not a necessary condition for the existence of a predictor, 
as is shown in Theorem~\ref{th:sep2}.

\begin{definition}[asymptotic KL ``distance'' $D$] 
Define  asymptotic expected average KL divergence between measures $\mu$ and $\rho$ as 
\begin{equation}\label{eq:akl2}
 D(\mu,\rho)=\limsup_{n\rightarrow\infty} {1\over n}  L_n(\mu,\rho).
\end{equation}
\end{definition}

\begin{theorem}\label{th:sep2}
For any set $\C$ of probability measures on $(\X^\infty,\mathcal F)$, separability with respect to the asymptotic expected average KL divergence $D$ is a sufficient but not a  necessary condition
for the existence of a predictor:
\begin{itemize}
 \item[(i)] If there exists a countable set $\mathcal D:=\{\nu_k:k\in\N\}\subset\C$, such that  for every $\mu\in\C$ and every $\epsilon>0$ there is a measure 
$\mu'\in\mathcal D$, such  that $D(\mu,\mu')<\epsilon$, then every measure  $\nu$ of the form $\nu=\sum_{k=1}^\infty w_k\mu_k$, where  $w_k$ are any positive
weights that sum to 1,  predicts every $\mu\in\C$ in expected average KL divergence.
\item[(ii)]  There is an uncountable set $\C$ of measures, and a measure $\nu$, such that $\nu$ predicts every $\mu\in\C$ in expected average KL divergence,
but $\mu_1\ne\mu_2$ implies $D(\mu_1,\mu_2)=\infty$ for every $\mu_1,\mu_2\in\C$; in particular, $\C$ is not separable with respect to $D$.
\end{itemize}
\end{theorem}
\begin{proof}
 {\em (i)} Fix $\mu\in\C$. For every $\epsilon>0$ pick $k\in\N$ such that $D(\mu, \nu_k)<\epsilon$. We have 
$$ 
 L_n(\mu,\nu) = \E_\mu\log\frac{\mu(x_{1..n})}{\nu({x_{1..n}})} \le \E_\mu\log\frac{\mu(x_{1..n})}{w_k\nu_k{(x_{1..n}})}  = -\log w_k + L_n(\mu,\nu_k) \le n\epsilon + o(n).
$$ Since this holds for every $\epsilon$, we conclude ${1\over n} L_n(\mu,\nu)\to0$.

{\em (ii)} Let $\C$ be the set of all deterministic sequences (measures concentrated on just one sequence) such that the number of 0s in the 
first $n$ symbols is less than $\sqrt{n}$. Clearly, this set is uncountable. It is easy to check that $\mu_1\ne\mu_2$ implies $D(\mu_1,\mu_2)=\infty$ for every $\mu_1,\mu_2\in\C$, but 
the predictor $\nu$, given by $\nu(x_n=0):=1/n$ independently for different $n$, predicts every $\mu\in\C$ in expected average KL divergence.
\qed\end{proof}

\subsubsection{Examples} 
For Bernoulli i.i.d. and $k$-order Markov processes, the (countable) sets of processes that have rational 
values of the parameters, considered in the previous section, are dense both in the topology of the parametrization
and with respect to the asymptotic average divergence $D$. It is also easy to check from the arguments presented in the corresponding example
of Section~\ref{s:exvc0}, that the family of all $k$-order stationary Markov processes with rational values
of the parameters, where we take all $k\in\N$, is dense with respect to $D$ in the set $\S$ of all stationary processes, so
that $\S$ is separable with respect to $D$. Thus, the sufficient but not necessary condition of separability is satisfied in this case.
On the other hand, neither of these latter families is separable with respect to the topology of total variation distance.

\subsection{Conditions based on the local behaviour of measures} \label{s:loc}
Next we provide some sufficient conditions for the existence of a predictor based 
on local characteristics of the class of measures, that is, measures truncated to the first $n$ observations.
First of all, it must be noted that necessary and sufficient conditions cannot be obtained this way. The basic
example is that of a family $\mathcal D$ of all deterministic sequences that are 0 from some time on. This is a countable
class of measures which is very easy to predict. Yet, the class of measures on $\X^n$, obtained by truncating
all measures in $\mathcal D$ to the first $n$ observations, coincides with what would be obtained by truncating all deterministic
measures to the first $n$ observations, the latter class being obviously not predictable at all (see also examples below). 
Nevertheless, considering this kind of local behaviour of measures, one can obtain not only sufficient conditions 
for the existence of a predictor, but also rates of convergence of the prediction error. It also gives some
ideas of how to construct  predictors, for the cases when the sufficient conditions obtained are met.  

For  a class   $\C$ of  stochastic processes   and a sequence $x_{1..n}\in\X^n$ introduce
the coefficients
\begin{equation}\label{eq:sup}
 c_{x_{1..n}}(\C):=\sup_{\mu\in\C}\mu(x_{1..n}).
\end{equation}
Define also the normalizer 
\begin{equation}\label{eq:nor}
 c_n(\C):=\sum_{x_{1..n}\in\X^n}c_{x_{1..n}}(\C).
\end{equation}
\begin{definition}[NML estimate]
  The  normalized maximum likelihood estimator $\lambda$ is defined  (e.g., \cite{Krichevsky:93}) as 
\begin{equation}\label{eq:nml}
\lambda_\C(x_{1..n}):= \frac{1}{c_n(\C) } c_{x_{1..n}}(\C),
\end{equation}
for each $x_{1..n}\in\X^n$.
\end{definition}
The family $\lambda_\C(x_{1..n})$ (indexed by $n$) in general does  not immediately define a stochastic process over $\X^\infty$ ($\lambda_\C$ are not consistent for different $n$);
thus, in particular, using average KL divergence for measuring prediction quality would not make sense, since
$$L_n(\mu(\cdot|x_{1..n-1}),\lambda_\C(\cdot|x_{1..n-1}))$$ can be negative, as the following example shows.

{\noindent\bf Example: negative $L_n$ for NML estimates}. Let the processes $\mu_i$, $i\in\{1,\dots,4\}$ be defined on
the steps $n=1,2$ as follows. $\mu_1(00)=\mu_2(01)=\mu_4(11)=1$, while $\mu_3(01)=\mu_3(00)=1/2$. 
We have $\lambda_\C(1)=\lambda_\C(0)=1/2$, while  $\lambda_\C(00)=\lambda_\C(01)=\lambda_\C(11)=1/3$. If we 
define $\lambda_\C(x|y)=\lambda_\C(yx)/\lambda_\C(y)$, we obtain $\lambda_\C(1|0)=\lambda_\C(0|0)=2/3$. Then
$d_2(\mu_3(\cdot|0),\lambda_\C(\cdot|0))=\log3/4<0$.

Yet, by taking an appropriate mixture, it is still possible to construct a predictor (a stochastic process) based on $\lambda$, 
that predicts all the measures in the class. 
\begin{definition}[predictor $\rho_c$]\label{def:nml}
Let $w:=6/\pi^2$  and let $w_k:=\frac{w}{ k^2}$.
Define a measure $\mu_k$ as follows. On the first $k$ steps it is defined as $\lambda_\C$, and 
for $n>k$ it outputs only zeros with probability 1; so, $\mu_k(x_{1..k})=\lambda_\C(x_{1..k})$ and 
$\mu_k(x_n=0)=1$ for $n>k$. 
Define the measure $\rho_c$ as
\begin{equation}\label{eq:r2}
\rho_c=\sum_{k=1}^\infty w_k\mu_k.
\end{equation}
\end{definition}
Thus, we have taken the normalized maximum likelihood estimates $\lambda_n$ for each $n$ and continued
them arbitrarily (actually, by a deterministic sequence) to obtain a sequence of measures on $(\X^\infty, \mathcal F)$ that can be summed.

\begin{theorem}\label{th:ml} For  any set $\C$ of probability measures on $(\X^\infty,\mathcal F)$, the predictor  $\rho_c$ defined above satisfies
\begin{equation}\label{eq:mlb}
{1\over n} L_n(\mu,\rho_c)\le \frac{\log c_n(\C)}{n} + O\left(\frac{\log n}{n}\right);
\end{equation}
in particular, if \begin{equation}\label{eq:cond1}
\log c_n(\C) =o(n),
\end{equation}
then $\rho_c$ predicts every $\mu\in\C$ in expected average KL divergence.
\end{theorem}
\begin{proof}
Indeed, 
\begin{multline}\label{eq:mlproof}
 {1\over n} L_n(\mu,\rho_c) 
  = \frac{1}{n}\E \log \frac {\mu(x_{1..n})}{\rho_c(x_{1..n})} 
  \le \frac{1}{n} \E \log \frac {\mu(x_{1..n})}{w_n \mu_n(x_{1..n})}\\
  \le \frac{1}{n} \log\frac{c_n(\C)}{w_n} = \frac{1}{n}(\log c_n(\C) + 2\log n + \log w).
\end{multline}
\qed\end{proof}

\subsubsection{Examples}
{\noindent\bf I.i.d., finite-memory.} To illustrate the applicability of the theorem we first consider 
the class of i.i.d. processes ${\mathcal B}$ over the binary alphabet $\X=\{0,1\}$. 
It is easy to see that, for each $x_1,\dots,x_n$,
$$
\sup_{\mu\in{\mathcal B}} \mu(x_{1..n})=(k/n)^k(1-k/n)^{n-k},
$$
where $k=\#\{i\le n: x_i=0\}$ is the number of 0s in $x_1,\dots,x_n$. 
For the constants $c_n(\C)$ we can derive
\begin{multline*}
 c_n(C)=\sum_{x_{1..n}\in \X^n}\sup_{\mu\in{\mathcal B}}\mu(x_{1..n})=\sum_{x_{1..n}\in \X^n} (k/n)^k(1-k/n)^{n-k}\\=\sum_{k=0}^n{n\choose k}(k/n)^k(1-k/n)^{n-k}\le 
\sum_{k=0}^n\sum_{t=0}^n{n\choose k}(k/n)^t(1-k/n)^{n-t}=n+1,
\end{multline*}
so that  $c_n(C)\le n+1$.

In general, for the class $\mathcal M_k$ of {\bf processes with memory $k$} over a finite
space $\X$ we can get polynomial coefficients $c_n(\mathcal M_k)$  (see, for example, \cite{Krichevsky:93}).
Thus, with respect to  finite-memory processes, the conditions 
of Theorem~\ref{th:ml} leave ample space for the growth of $c_n(\C)$, since~(\ref{eq:cond1}) allows  subexponential growth of $c_n(\C)$.
Moreover, these conditions are tight, as the following example shows.

{\noindent \bf Exponential coefficients are not sufficient.} 
Observe that the condition~(\ref{eq:cond1}) cannot 
be relaxed further, in the sense that exponential coefficients $c_n$ are 
not sufficient for prediction. Indeed, for the class of all deterministic 
processes (that is,  each process from the class produces some fixed sequence 
of observations with probability 1) we have $c_n=2^n$, while obviously 
for this class a predictor does not exist. 

{\noindent \bf Stationary processes.} For the set of all stationary processes we can obtain $c_n(C)\ge2^n/n$ (as is easy to 
see by considering periodic $n$-order Markov processes, for each $n\in\N$), so
that the conditions of Theorem~\ref{th:ml} are not satisfied.  This cannot be fixed, since  
uniform rates of convergence cannot be obtained for this family of processes, as was shown
in \cite{BRyabko:88}.

\subsection{Optimal rates of uniform convergence}
 A natural question that arises with respect to the bound~(\ref{eq:mlb}) is whether it can be
matched by a lower bound. This question is closely related to the optimality
of the normalized maximum likelihood estimates used in the construction of the predictor. In general,
since NML estimates are not optimal, neither are the rates of convergence in~(\ref{eq:mlb}).
To obtain (close to) optimal rates one has to consider a different measure of capacity.
 
To do so, we make the following connection to a problem 
in information theory. Let $\mathcal P(\X^\infty)$ be the set of all stochastic processes (probability measures) on
the space $(\X^\infty,\mathcal F)$, and let $\mathcal P(\X)$ be the set of probability distributions over a (finite) set $\X$. For a class $\C$ of measures we are interested in a predictor
that has a small (or minimal) worst-case (with respect to the class $\C$) probability of error.
Thus, we are interested in the quantity
\begin{equation}\label{eq:infsup}
\inf_{\rho\in\mathcal P(\X^\infty)} \sup_{\mu\in\C} D(\mu,\rho),
\end{equation}
where the infimum is taken over all stochastic processes $\rho$, and $D$ is the 
asymptotic expected average KL divergence~(\ref{eq:akl2}).
(In particular, we are interested in the conditions under which the quantity~(\ref{eq:infsup})
equals zero.) This problem has been studied for the case when the probability measures
are over a finite set $\X$, and $D$ is replaced simply by the KL divergence $d$ between 
the measures. 
Thus, the problem was to find the probability measure $\rho$ (if it exists) on which the following 
minimax is attained 
\begin{equation}\label{eq:infsup2}
R(A):=\inf_{\rho\in\mathcal P(\X)}\sup_{\mu\in A} d(\mu,\rho),
\end{equation}
where $A\subset\mathcal P(\X)$.
This problem is closely related to the problem of finding the best code for the class of sources
$A$, which was its original motivation. 
The normalized maximum likelihood distribution considered above does not in general
lead to the optimum solution for this problem. 
The optimum solution is obtained through the  result that relates the minimax~(\ref{eq:infsup2})  to the so-called channel capacity.
\begin{definition}[Channel capacity]
For a set $A$ of measures on a finite set $\X$ the {\em channel capacity} of
$A$ is defined as 
\begin{equation}\label{eq:cc}
 C(A):=\sup_{P\in \mathcal P_0(A)} \sum_{\mu\in S(P)} P(\mu) d(\mu,\rho_P),
\end{equation} where $\mathcal P_0(A)$ is the set of all probability distributions on $A$ that have a finite support, $S(P)$ is the (finite) support
 of a distribution $P\in\mathcal P_0(A)$, and
$\rho_P=\sum_{\mu\in S(P)} P(\mu)\mu$.
\end{definition}
 It is shown in \cite{BRyabko:79,Gallager:76} that 
$
C(A)=R(A),
$
thus reducing the problem of finding a minimax to an optimization problem.
For probability measures over infinite spaces this result ($R(A)=C(A)$)
was generalized by \cite{Haussler:97}, but the divergence between probability distributions is measured
by KL divergence (and not asymptotic average KL divergence), which gives  infinite $R(A)$ e.g.  already for the class of i.i.d. processes. 

However, truncating measures in a class $\C$ to the first $n$ observations, we can
use the results about channel capacity to analyse the predictive properties of the class.
Moreover, the rates of convergence that can be obtained along these lines are close to optimal.
In order to pass from measures minimizing the divergence for each individual $n$ to a process that minimizes
the divergence for all $n$ we use the same idea as when constructing the process~$\rho_c$.

\begin{theorem}\label{th:cc} Let $\C$ be a set of measures on $(\X^\infty,\mathcal F)$, and let $\C^n$ be the class
of measures from $\C$ restricted to $\X^n$. 
There exists a measure $\rho_C$ such that
\begin{equation}\label{eq:ccb}
 {1\over n} L_n(\mu,\rho_C)\le \frac{ C(\C^n)}{n} + O\left(\frac{\log n}{n}\right);
\end{equation}
in particular, if  $C(\C^n)/n\rightarrow0$, then  $\rho_C$ predicts every $\mu\in\C$ in expected average KL divergence.
Moreover, for any measure $\rho_C$ and every $\epsilon>0$ there exists $\mu\in\C$ such that
$$
 {1\over n} L_n(\mu,\rho_C)\ge \frac{ C(\C^n)}{n} -\epsilon.
$$
\end{theorem}
\begin{proof} 
As shown in \cite{Gallager:76}, for each $n$ there exists a  sequence $\nu^n_k$, $k\in\N$
of measures on $\X^n$ such that 
$$
 \lim_{k\rightarrow\infty} \sup_{\mu\in \C^n} L_n(\mu,\nu^n_k)\rightarrow C(\C^n).
$$ 
For each $n\in\N$ find an index $k_n$ such that  
$$ 
|\sup_{\mu\in \C^n} L_n(\mu,\nu^n_{k_n}) - C(\C^n)|\le 1.
$$
Define the measure $\rho_n$ as follows. On the first $n$ symbols it coincides with $\nu^n_{k_n}$ and
$\rho_n(x_m=0)=1$ for $m>n$. Finally, set $\rho_C=\sum_{n=1}^\infty w_n\rho_n$, where $w_k=\frac{w}{ n^2}, w=6/\pi^2$.
We have to show that $\lim_{n\rightarrow\infty}{1\over n} L_n(\mu,\rho_C)=0$ for every $\mu\in\C$.
Indeed, similarly to~(\ref{eq:mlproof}), we have
\begin{multline}\label{eq:ccproof}
{1\over n} L_n(\mu,\rho_C)=\frac{1}{n}\E_\mu\log\frac{\mu(x_{1..n})}{\rho_C(x_{1..n})} \\
 \le \frac{\log w_k^{-1}}{n} +\frac{1}{n} \E_\mu\log\frac{\mu(x_{1..n})}{\rho_n(x_{1..n})} 
 \le
 \frac{\log w +2\log n}{ n}+\frac{1}{n} L_n(\mu,\rho_n)\\
\le o(1) + \frac{C(\C^n)}{n}.
\end{multline} 

The second statement follows from the fact \cite{BRyabko:79,Gallager:76} that \ $C(\C^n)=R(\C^n)$ \ (cf.~(\ref{eq:infsup2})).
\qed\end{proof}  

Thus, if the channel capacity  $C(\C^n)$ grows sublinearly, a predictor 
can be constructed for the class of processes $\C$. In this case 
the problem of constructing the predictor is reduced to finding the channel 
capacities for different $n$ and finding the corresponding measures on which 
they are  attained or approached. 

\subsubsection{Examples} For the class of all Bernoulli i.i.d. processes,
the channel  capacity   $C(\mathcal B^n)$   is known to be $O(\log n)$ \cite{Krichevsky:93}.
For the family of all stationary processes it is $O(n)$, so that the conditions of Theorem~\ref{th:cc} are satisfied
for the former but not for the latter. 

We also remark that the requirement of a sublinear channel capacity cannot be relaxed,
in the sense that a linear channel capacity is not sufficient for prediction, since
it is the maximal possible  capacity for a set of measures on $\X^n$, achieved, for example, on the set of all measures, or on 
the set of all deterministic sequences.

\section{Non-Realizable Case and  Suboptimality of Mixture Predictors}\label{s:not}
In the non-realizable case, that is, if we consider the given set $\C$ as a set of comparison predictors (experts) and seek to minimize the regret with respect to the best predictor in the set, the situation is principally different. The main result of this section establishes that  it may happen that every Bayesian  combination (including discrete mixture predictors) of distributions in a set  is suboptimal~--- even  asymptotically.  
In addition, some sufficient conditions on the existence of a (mixture) predictor with asymptotically vanishing regret are analysed, and some examples of process families for which such predictors exist or not are considered.
\subsection{Suboptimality of Mixture Predictors}
\begin{theorem}\label{th:not}
 There exists a set $\C$ for which  $U_\C=0$ and this value is attainable (that is, there exists $\rho\in\mathcal P$ for which $R(\C,\rho)=0$),
yet for some constant $c>0$ every Bayesian predictor $\phi$ with a prior concentrated on $\C$ must have
$$
R(\C,\rho)>c>0.
$$
In other words, any  Bayesian predictor must have a linear regret, while there exists a predictor with a sublinear regret.

\end{theorem}
Intuitively, the reason why any Bayesian  
predictor does not work in the counterexample of the proof given below is  as follows. The set $\C$ considered is so large that any Bayesian  predictor has to attach an exponentially decreasing a-posteriori weight to each element in $\C$. At the same time, by construction, each measure in $\C$ already attaches too little  weight to the part 
of the event space on which it is a good predictor. In other words, the likelihood of the observations with respect to  each predictor in $\C$ is too small to allow for any added penalty. To combine predictors in $\C$ one has to {\em boost} the likelihood, rather than attach a penalty.
While this result is stated for Bayesian predictors, 
from the argument above it is clear that the  example used in the proof  is applicable to any combination of predictors in $\C$ one might think of, including, for example, MDL \cite{Rissanen:89} and expert-advice-style predictors  (e.g.,\cite{Cesa:06}). Indeed, if one has to boost the likelihood for some classes of predictors, it 
clearly breaks the predictor for other classes. In other words, there is no way to combine the prediction of the experts short of disregarding them and doing something else instead.

\begin{remark}[Countable $\C$]\label{r:count}
Note that any set $\C$ satisfying the theorem must necessarily be uncountable.
Indeed, for any countable set $\C=(\mu_k)_{k\in\N}$, take the Bayesian predictor $\phi:=\sum_{k\in\N}w_k\mu_k$,
where $w_k$ can be, for example, $\frac{1}{k(k+1)}$. Then, for any $\nu$ and any $n$, from~\eqref{eq:kl} we obtain
$$L_n(\nu,\phi)\le-\log w_k + L_n(\nu,\mu_k).$$
That is to say, the regret of $\phi$ with respect to any $\mu_k$ is independent of $n$, and thus for every $\nu$ we have
$\bar R^\nu(C,\phi)=0$.
\end{remark}

\begin{proof}[of Theorem~\ref{th:not}.]
Let the alphabet $\X$  be  ternary $\X=\{0,1,2\}$. 
 For $\alpha\in(0,1)$ denote $h(\alpha)$ the binary entropy $h(\alpha):= -\alpha\log \alpha-(1-\alpha)\log(1-\alpha)$.
Fix an arbitrary  $p\in(0,1/2)$  %
 and let $\beta_p$ be the Bernoulli i.i.d.\ measure (produces only 0s and 1s) with parameter $p$. 
Let $S$ be the set of sequences in $\X^\infty$ that have no $2$s and such that the frequency of $1$ is close to $p$:
\begin{equation*}
 S:=\left\{{\bf x}\in\X^\infty:  x_i\ne2 \forall i, \text { and}   \left|{1\over t}|\{i=1..t:x_i=1\}| - p\right|\le f(t) \text{ from some $t$ on}\right\},
\end{equation*}
where $f(t)=\log t/\sqrt{t}$. Clearly, $\beta_p(S)=1$. 

 Define the set $D_S$ as the set of all Dirac probability measures concentrated on a sequence
from $S$, that is $D_S:=\{\nu_\x: \nu_\x(\x)= 1,\ \x\in S\}$. Moreover, for each $\x\in S$ define the probability measure $\mu_\x$ as follows: 
$\mu_\x(X_{n+1}|X_{1..n})=p$ coincides with $\beta_p$ (that is, 1 w.p. $p$ and 0 w.p. $1-p$) if $X_{1..n}=x_{1..n}$, 
 and outputs 2 w.p.~1 otherwise: $\mu_\x(2|X_{1..n})=1$ if $X_{1..n}\ne x_{1..n}$.  That is, $\mu_\x$ behaves
as $\beta_p$ only on the sequence $\x$, and on all other sequences it just outputs 2 deterministically. 
This  means, in particular, that many sequences have probability 0, and some probabilities above are defined conditionally on 
zero-probability events,  but this is not a problem; see the remark in the end of the proof. 

Finally, let $\C:=\{\mu_\x: \x\in S\}$. Next we define the predictor $\rho$ that predicts 
well all measures in $\C$. First, introduce the probability measure $\delta$ that is going to take care of 
all the measures that output 2 w.p.1 from some time on. For each $a\in\X^*$ let $\delta_a$ be the 
probability measure that is concentrated on the sequence that starts with $a$ and then consists of all 2s.
Define $\delta:=\sum_{a\in\X^*}w_a\delta_a$, where $w_a$ are arbitrary positive numbers that sum to 1. 
Let also the probability measure  $\beta'$ be  i.i.d.\ uniform over $\X$. 
 Finally, define    
\begin{equation}\label{eq:3}
\rho:=1/3(\beta_p+\beta'+\delta).
\end{equation}

Next, let us show that, for every $\nu$, the measure  $\rho$ predicts $\nu$ as well as any measure in $\C$:
its loss is an additive constant factor. 
In fact, it is enough to see this for all $\nu\in D_S$, and for all measures that output all 2s w.p.1\ from some $n$ on.
For each $\nu$ in the latter set, from~\eqref{eq:3} the loss of $\rho$ is upper-bounded by $\log 3  -\log {w_a}$, where $w_a$ is 
the corresponding weight. This is a constant (does not depend on $n$). For the former set, again from the definition~\eqref{eq:3}
for every $\nu_\x\in D_S$  we have  (see also Remark~\ref{r:count})
$$
L_n(\nu_\x,\rho)\le \log 3 + L_n(\nu_\x,\beta_p)= nh(p)+o(n),$$  while 
$$\inf_{\mu\in C} L_n(\nu_\x,\mu)= L_n(\nu_\x,\mu_\x) = nh(p)+o(n).$$ 
Therefore, for all $\nu$ we have 
$$
 R^\nu_n(C,\rho)=o(n)\text{ and }\bar R^\nu(C,\rho)=0.
$$
Thus, we have shown that for 
every $\nu\in S$ there is a reasonably good predictor in $\C$ (here ``reasonably good'' means
that its loss is linearly far from  that of random guessing), and, moreover, there is a predictor
$\rho$ whose asymptotic regret is zero with respect to $\C$. 

Next we need to show that any Bayes predictor has $2nh(p) + o(n)$ loss on at least some measure, which 
is double that of $\rho$, and which can be as bad as random guessing  (or worse; depending on $p$).
 We show something stronger:  any Bayes predictor has asymptotic  average loss  of $2nh(p)$ {\em on average} 
over all measures in $S$. So there will be many measures on which it is bad, not just one. 

  Let $\phi$ be any Bayesian predictor with its prior concentrated on  $\C$. Since $\C$ is parametrized by $S$,
for any $x_{1..n}\in\X^n, n\in\N$ we can write  $\phi(x_{1..n})= \int_S \mu_\y(x_{1..n})dW(\y)$
where $W$ is some probability measure  over $S$ (the prior).
Moreover, using the notation $W(x_{1..k})$ for the $W$-measure of all sequences in $S$ that 
start with $x_{1..k}$, from the definition of the measures $\mu_\x$, for every $\x\in S$  we have  
\begin{equation}\label{eq:unint}
 \int_S \mu_\y(x_{1..n})dW(\y)  = \int_{{\bf y}\in S: y_{1..n}=x_{1..n}} \beta_p(x_{1..n})dW(\y) =  \beta_p(x_{1..n}) W(x_{1..n}).
\end{equation}

 Consider 
the average 
$$E_U \limsup {1\over n} L_n(\nu_x,\phi) dU(\x),$$
 where the expectation is taken with respect to 
the probability measure $U$ defined as the measure $\beta_p$ 
 restricted to $S$; in other words,  $U$ is approximately   uniform  over this set.
Fix any $\nu_\x\in S$. Observe that  $L_n(\nu_{\bf x},\phi)=-\log\phi(x_{1..n})$.
For the asymptotic regret, we can assume w.l.o.g.\  that the loss  $L_n(\nu_\x,\phi)$  is upper-bounded, say,
 by $n\log|\X|$ at least from some $n$ on
(for otherwise the statement already holds for $\phi$). This allows us to 
 use  Fatou's lemma to bound 
\begin{multline}\label{eq:f}
 E_U \limsup {1\over n} L_n(\nu_\x,\phi)  \ge  \limsup {1\over n} E_U  L_n(\nu_\x,\phi)  
=    \limsup - {1\over n}  E_U \log\phi(\x) \\ = \limsup - {1\over n}  E_U \log  \beta_p(x_{1..n}) W(x_{1..n}),
\end{multline} 
where in the last equality we used~\eqref{eq:unint}. 
Moreover, 
\begin{multline}
- E_U \log \beta_p(x_{1..n}) W(x_{1..n})\\= -  E_U \log \beta_p(x_{1..n})  + E_U \log \frac{U(x_{1..n})}{ W(x_{1..n})} - E_U \log  U(x_{1..n}) \ge 2h(p)n + o(n),
\end{multline}
where in the inequality we have used the fact that KL divergence is non-negative and the definition of $U$ (that is, that $U=\beta_p|_S$).
From this and~\eqref{eq:f} we obtain the statement of the theorem.

Finally, we remark that all the considered probability measures can be made non-zero everywhere by simply combining them with 
the uniform i.i.d.\  over  $\X$ measure $\beta'$, that is, taking for each measure $\nu$ the combination ${1\over 2}(\nu +\beta')$.
This way all losses up to time $n$ become bounded by $n\log|\X|+1$, but the result still holds with a different constant.
\qed\end{proof}
\subsection{Some sufficient conditions on a set $\C$ for the existence of a predictor with vanishing regret}\label{s:sufreg}
We have seen that, in the case of prediction in total variation, separability with respect to the topology of this distance
can be used to characterize the existence of a predictor with either vanishing loss or regret. 
In the case of expected average KL divergence the situation is somewhat different, since, first of all, (asymptotic average) KL divergence 
is not a metric. While one can introduce a topology based on it, 
separability with respect to this topology turns out to be a sufficient but not a necessary condition for the existence of a predictor, 
as is shown in the next theorem.

\begin{definition}\label{def:dinf} Define the distance $d_\infty(\mu_1,\mu_2)$ on process measures as 
follows 
 \begin{equation}\label{eq:dinf}
  d_\infty(\mu_1,\mu_2)=\limsup_{n\to\infty}\sup_{x_{1..n}\in\X^n}\frac{1}{n}\left|\log\frac{\mu_1(x_{1..n})}{\mu_2(x_{1..n})}\right|,
 \end{equation} 
where we assume $\log0/0:=0$.
\end{definition}
Clearly, $d_\infty$ is symmetric and satisfies the triangle inequality, but it is not exact.
Moreover, for every $\mu_1,\mu_2$ we have
\begin{equation}
 \limsup_{n\to\infty}\frac{1}{n}L_n(\mu_1,\mu_2)\le d_\infty(\mu_1,\mu_2).
\end{equation}
The distance $d_\infty(\mu_1,\mu_2)$ measures the difference in behaviour of $\mu_1$ and $\mu_2$ on all individual sequences.
Thus, using this distance to analyse the non-realizable case is most close to the traditional approach to this case, which is 
formulated in terms of predicting  individual deterministic sequences.
\begin{theorem}\label{th:dinf}  
\begin{itemize}
  \item[(i)] Let $\C$ be a set of process measures. If $\C$ is separable with respect to $d_\infty$ then 
there is a predictor with vanishing regret with respect to $\C$ 
for the case of prediction in expected average KL divergence. 
 \item[(ii)] There exists a set of process measures $\C$ such that $\C$ is not separable with respect to $d_\infty$, but 
there is a predictor with vanishing regret with respect to $\C$ 
for the case of prediction in expected average KL divergence.
\end{itemize}
\end{theorem}
\begin{proof}
For the first statement,  let $\C$ be separable and let $(\mu_k)_{k\in\N}$ be a dense countable subset of $\C$. Define $\nu:=\sum_{k\in\N}w_k\mu_k$, 
where $w_k$ are any positive summable weights. Fix any measure $\tau$ and any $\mu\in\C$.
We will show that $\limsup_{n\to\infty}{1\over n}L_n(\tau,\nu)\le \limsup_{n\to\infty}{1\over n}L_n(\tau,\mu)$.
For every $\epsilon$, find such a $k\in\N$ that $d_\infty(\mu,\mu_k)\le\epsilon$. We have
\begin{multline*}
L_n(\tau,\nu)\le L_n(\tau,w_k\mu_k) =\E_\tau\log\frac{\tau(x_{1..n})}{\mu_k(x_{1..n})}-\log w_k
\\= \E_\tau\log\frac{\tau(x_{1..n})}{\mu(x_{1..n})} + \E_\tau\log\frac{\mu(x_{1..n})}{\mu_k(x_{1..n})} -\log w_k\\
\le L_n(\tau,\mu)+\sup_{x_{1..n}\in\X^n}\log\left|\frac{\mu(x_{1..n})}{\mu_k(x_{1..n})}\right| -\log w_k.
\end{multline*}
From this, dividing by $n$ taking $\limsup_{n\to\infty}$ on both sides, we conclude 
$$
\limsup_{n\to\infty}{1\over n}L_n(\tau,\nu)\le \limsup_{n\to\infty}{1\over n}L_n(\tau,\mu) + \epsilon.
$$
Since this holds for every $\epsilon>0$ the first statement is proven.

The second statement is proven by the following example. 
 Let $\C$ be the set of all deterministic sequences (measures concentrated on just one sequence) such that the number of 0s in the 
 first $n$ symbols is less than $\sqrt{n}$, for all $n\in\N$. Clearly, this set is uncountable. 
It is easy to check that $\mu_1\ne\mu_2$ implies $d_\infty(\mu_1,\mu_2)=\infty$ for every $\mu_1,\mu_2\in\C$, but 
 the predictor $\nu$, given by $\nu(x_n=0)=1/n$ independently for different $n$, predicts every $\mu\in\C$ in expected average KL divergence. Since all elements
of $\C$ are deterministic, $\nu$ is also a solution to the non-realizable case for $\C$.
\qed\end{proof}

As we have seen in the statements above, the set of all deterministic measures  $\mathcal D$ 
plays an important role in the analysis of the predictors in the non-realizable case.
Therefore, an interesting question is to characterize those sets $\mathcal C$ of measures 
for which there is a predictor $\rho$ that predicts {\em every individual sequence} at least
as well as any measure from $\C$.  Such a characterization can be obtained in terms of Hausdorff dimension  using a result of \cite{BRyabko:86}
that shows that Hausdorff dimension of a set characterizes the optimal prediction error that can be attained by any predictor.

For a set $A\subset \X^\infty$ denote $H(A)$ its Hausdorff dimension (see, for example, \cite{Billingsley:65} for its definition).

\begin{theorem}\label{th:hau}
 Let $\C\subset\P$. The following statements are equivalent.
\begin{itemize}
 \item[(i)] There is a measure $\rho\in\P$ that predicts every individual sequence at least as well as the best measure from $\C$:  
  for every $\mu\in\C$ and  every sequence $x_1,x_2,\dots\in\X^\infty$ we have
  \begin{equation}\label{eq:ha}
   \liminf_{n\to\infty}-{1\over n}\log\rho(x_{1..n})\le \liminf_{n\to\infty}-{1\over n}\log\mu(x_{1..n}).
  \end{equation}
 \item[(ii)] For every $\alpha\in[0,1]$ the Hausdorff dimension of the set of sequences on which the average prediction error 
 of the best measure in $\C$ is not greater than $\alpha$ is bounded by $\alpha/\log|\X|$:
 \begin{equation}\label{eq:ha2}
  H(\{x_1,x_2,\dots\in\X^\infty:\inf_{\mu\in\C}\liminf_{n\to\infty}-{1\over n}\log\mu(x_{1..n})\le\alpha\})\le \alpha/\log|\X|.
 \end{equation}
\end{itemize}
\end{theorem}
\begin{proof}
 The implication $(i)\Rightarrow(ii)$ follows directly from \cite{BRyabko:86} where it is shown that for every 
measure $\rho$ one must have 
$$
H(\{x_1,x_2,\dots\in\X^\infty:\liminf_{n\to\infty}-{1\over n}\log\rho(x_{1..n})\le\alpha\})\le \alpha/\log|\X|.
$$

To show the opposite implication, we again refer to \cite{BRyabko:86}: for every set $A\subset\X^\infty$ there is a measure $\rho_A$ such 
that   
\begin{equation}\label{eq:hh}
\liminf_{n\to\infty}-{1\over n}\log\rho_A(x_{1..n})\le H(A)\log|\X|.
\end{equation}
 For each $\alpha\in[0,1]$ define
$
 A_\alpha:= \{x_1,x_2,\dots\in\X^\infty:\inf_{\mu\in\C}\liminf_{n\to\infty}-{1\over n}\log\mu(x_{1..n})\le \alpha\}).
$ 
By assumption, $H(A_\alpha)\le\alpha/\log|X|$, so that from~(\ref{eq:hh}) for all $x_1,x_2,\dots\in A_\alpha$ we obtain
\begin{equation}\label{eq:hh2}
\liminf_{n\to\infty}-{1\over n}\log\rho_A(x_{1..n})\le \alpha.
\end{equation}
Furthermore, define $\rho:=\sum_{q\in Q}w_q\rho_{A_q}$, 
where $Q=[0,1]\cap\mathbb Q$ is the set of rationals in $[0,1]$ and $(w_q)_{q\in  Q}$ is any sequence of positive reals
 satisfying $\sum_{q\in Q}w_q=1$. For every $\alpha\in[0,1]$ let $q_k\in Q$, $k\in\N$ be such a sequence that $0\le q_k -\alpha\le1/k$. 
 Then, for every $n\in\N$ and every $x_1,x_2,\dots\in A_{q_k}$ we have 
$$
 -{1\over n}\log\rho(x_{1..n})\le -{1\over n}\log\rho_q(x_{1..n}) -\frac{\log w_{q_k}}{n}.
$$
From this and~(\ref{eq:hh2}) we get
$$
\liminf_{n\to\infty}-{1\over n}\log\rho(x_{1..n})\le \liminf_{n\to\infty} \rho_{q_k}(x_{1..n})+1/k\le q_k+1/k.
$$ Since this holds 
for every $k\in\N$, it follows that for all $x_1,x_2,\dots\in\cap_{k\in\N} A_{q_k}=A_\alpha$ we have
$$
 \liminf_{n\to\infty}-{1\over n}\log\rho(x_{1..n})\le \inf_{k\in\N}(q_k +1/k)=\alpha,
$$
which completes the proof of the implication $(ii)\Rightarrow(i)$.
\qed\end{proof}

\subsubsection{Examples: finite memory and stationary processes}
Although simple, Theorem~\ref{th:dinf} can be used to establish the existence of a solution to the non-realizable case of the sequence prediction problem for an important 
class of process measures: that of all processes with  finite memory, as the next theorem shows. Results similar to Theorem~\ref{th:mark} are known 
in different settings, e.g., \cite{Ziv:78, BRyabko:84, Cesa:99} and others. 

\begin{theorem}\label{th:mark}
 There exists a mixture predictor $\nu$ with asymptotically vanishing regret with respect to the  set of all finite-memory process measures $\mathcal M:=\cup_{k\in\N}\mathcal M_k$: 
$$
R(\mathcal M,\nu)=0.
$$.
\end{theorem}
\begin{proof}
 We will show that the set $\mathcal M$ is separable with respect to $d_\infty$. Then the statement will follow from Theorem~\ref{th:dinf}.
It is enough to show  that each set $\mathcal M_k$ is separable with respect to $d_\infty$. 

For simplicity, assume that the alphabet is binary ($|\X|=2$; the general case is analogous).
 Observe that the family $\mathcal M_k$ of $k$-order stationary binary-valued Markov
processes is parametrized by $|\X|^{k}$ $[0,1]$-valued parameters: probability of observing $0$ after observing $x_{1..k}$, for each $x_{1..k}\in\X^k$.
Note that this parametrization is continuous (as a mapping from the parameter space with the Euclidean topology to $\mathcal M_k$ with the
topology of $d_\infty$). 
Indeed,   for any $\mu_1,\mu_2\in\mathcal M_k$ 
and every $x_{1..n}\in\X^n$ such that $\mu_i(x_{1..n})\ne 0$, $i=1,2$, it is easy to see  that 
\begin{equation}\label{eq:m3}
{1\over n}\left |\log\frac{\mu_1(x_{1..n})}{\mu_2(x_{1..n})}\right|\le \sup_{x_{1..k+1}}{{1\over k+1} \left|\log\frac{\mu_1(x_{1..k+1})}{\mu_2(x_{1..k+1})}\right|},
\end{equation}
so that the right-hand side of~(\ref{eq:m3}) also upper-bounds $d_\infty(\mu_1,\mu_2)$, implying continuity of the parametrization.

It follows that  the set  $\mu^k_q$, $q\in Q^{|\X|^k}$  of all stationary $k$-order Markov processes with rational 
values of all the parameters ($Q:=\mathbb Q\cap[0,1]$) 
is dense in $\mathcal M_k$, proving the separability of the latter set.
\qed\end{proof}

Another important example is the set of all stationary process measures $\S$. This example also illustrates the difference between the prediction problems 
that we consider. For this set  a predictor with a vanishing loss is presented in Section~\ref{s:exvc0} (based on \cite{BRyabko:06a}). 
In contrast, here we show that there is no predictor with a vanishing regret for this set. A stronger impossibility result is obtained in Chapter~\ref{ch:p2}; however,  the proof of the statement  below is simpler and therefore is still instructive.
\begin{theorem}\label{th:stno1}
 There is no predictor with asymptotically vanishing regret with respect to the set of all stationary processes $\S$. 
\end{theorem}
\begin{proof}
 This   proof  is based on a construction similar to the one used in \cite{BRyabko:88}  (see also \cite{Gyorfi:98})
to demonstrate impossibility of consistent prediction  of stationary processes without Cesaro averaging. 

Let $m$ be a Markov chain with states $0,1,2,\dots$ and state transitions defined as follows.
From each sate $k\in\N\cup\{0\}$ the chain passes to the state $k+1$ with probability 2/3 and to the state 0 with probability 1/3. 
It is easy to see that this chain possesses a unique stationary distribution on the set of states (see, e.g., \cite{Shiryaev:96}); taken as the initial distribution 
it defines a stationary ergodic process with  values in $\N\cup\{0\}$. Fix the ternary alphabet $\X=\{a,0,1\}$. 
For each sequence $t=t_1,t_2,\dots\in\{0,1\}^\infty$ define the process $\mu_t$ as follows. It is a deterministic function 
of the chain $m$. If the chain is in the state 0 then the process $\mu_t$ outputs $a$; if the chain $m$ is in the state $k>0$ then 
the process outputs $t_k$. That is,  we have defined a hidden Markov process which in the state 0 of the underlying 
Markov chain always outputs  $a$, while in other states it outputs either $0$ or $1$ according to the sequence~$t$. 

To show that there is no predictor whose regret vanishes with respect to  $\S$, we will show that there is no predictor whose regret vanishes with respect to the smaller
set $\C:=\{\mu_t: t\in\{0,1\}^\infty\}$. Indeed, for any $t\in\{0,1\}^\infty$ we have $L_n(t,\mu_t)= n\log 3/2 + o(n)$. Then if $\rho$ 
is a solution to the non-realizable problem for $\C$ we should have $\limsup_{n\to\infty}{1\over n} L_n(t,\rho)\le \log 3/2<1$ for every $t\in\mathcal D$,
 which contradicts Lemma~\ref{th:disc}.
\qed\end{proof}

From the proof of Theorem~\ref{th:stno1} one can see that, in fact, the statement that is proven is stronger: there is no 
predictor with vanishing regret for a smaller set of processes:  for the set of all functions of stationary ergodic countable-state Markov chains. 
In Chapter~\ref{ch:p2} an even stronger result is proven using a similar construction.

\chapter{Decision-Theoretic Interpretations}\label{s:dt}
Classical decision theory is concerned with single-step games. Among its key  results 
are the complete class and minimax theorems. 
The asymptotic formulation of the (infinite-horizon) prediction problem considered in this book can be also viewed a single-shot game, where one player (which one may call the Nature or adversary) selects a measure $\nu$ that generates an infinite sequence of data, and the other player (the statistician) selects a predictor $\rho$. The infinite game is then played out step-by-step, with the resulting payout being the asymptotic loss, as measured by either expected average KL divergence $\bar L(\nu,\rho)$ or the  asymptotic total variation loss $l_{tv}(\nu,\rho)$. Here we disregard the finite-time losses and only consider this final payout.
In this chapter we consider the realizable case. Thus, the strategies of the statistician are unrestricted, $\rho\in\mathcal P$, and the strategies $\nu$ of the adversary are limited to a given set of measures $\C\subset\mathcal P$. The non-realizable case would require the introduction of a third player and this falls out of scope of the decision-theoretic framework that we consider here.

The considered problem presents  both differences and similarities with respect to the standard game-theoretic problem. 
A distinction worth mentioning at this point is that the results presented here are obtained  under no assumptions whatsoever, whereas
the results in decision theory we refer to always have a number of conditions; on the other hand, here we are concerned with  specific
loss functions:  the KL divergence (mainly), and the total variation distance (which turns out to be less interesting), rather than general losses that are common in decision theory.
The terminology in this section is mainly after \cite{Ferguson:14}.

We start with the asymptotic KL divergence, as it turns out to be more interesting in this context, and then in the end of the chapter summarize which facts carry over to the total-variation case.
\section{Players and Strategies}
Thus, the predictors $\rho\in\P$ are called  {\em  strategies  of the statistician}. 
The probability measures $\mu\in\C$ are now the basic {\em  strategies of the opponent} (a.k.a.\ Nature), and the first thing  we need to do is to extend these to  randomized strategies. 
To this end, denote $\C^*$ the set of all probability distributions over measurable subsets of $\C$. Thus, the opponent selects a randomized
strategy $W\in\C^*$ and the statistician (predictor) $\rho$ suffers the loss 
\begin{equation}\label{eq:ew}
  E_{W(\mu)} \bar L(\mu,\rho),
\end{equation}
 where the notation $W(\mu)$ means that $\mu$ is drawn
according to $W$. Note a distinction with the combinations we considered before. A combination of the kind $\nu=\int_\C\alpha dW(\alpha)$ is itself
a probability measure over  one-way infinite sequences, whereas a probability measure $W\in \C^*$ is a probability measure over $\C$.

\section{Minimax} Generalizing the definition~\eqref{eq:vc} of $V_\C$, we can now introduce the {\em minimax} or the {\em upper value}
\begin{equation}\label{eq:bvc}
   \bar V_\C:=\inf_{\rho\in\P} \sup_{W\in\C^*} E_{W(\mu)} \bd(\mu,\rho).
\end{equation}
Furthermore, the {\em maximin} or the {\em lower value} is defined as 
\begin{equation}\label{eq:ubvc}
   {\underline V}_\C:= \sup_{W\in\C^*} \inf_{\rho\in\P} E_{W(\mu)} \bd(\mu,\rho).
\end{equation}

The so-called minimax theorems in decision theory  (e.g., \cite{Ferguson:14}) for single-step games and general loss functions state that,
 under certain conditions, $\bar V_\C=\underline V_\C$ and the 
statistician has a minimax strategy, that is, there exists $\rho$ on which $\bar V_\C$ is attained. 
Minimax theorems generalize the classical result of von~Neumannn~\cite{neumann:28}, and provide sufficient conditions of various generality for it to hold.
A rather  general sufficient condition is the existence of a topology with respect to which the set of all strategies of the statistician, $\P$  in our case, 
is compact, and the risk, which  is $\bar L(\mu,\rho)$ in our case, is lower semicontinuous. 
Such a condition seems nontrivial to verify. For example, a (meaningful) topology with respect to 
which $\P$ is compact is that of the so-called distributional distance \cite{Gray:88} (in our case it coincides with the topology of the 
weak${}^*$ convergence), but $\bar L(\mu,\rho)$ is not  (lower) semicontinuous with respect to it. 
Some other (including non-topological) sufficient conditions  are given in \cite{Sion:58,Lecam:55}.
Other related  results for KL divergence (expected log loss) include \cite{BRyabko:79,Gallager:76,Haussler:97}.

In our setup, it is easy to see that, for every $\C$, $$\bar V_\C = V_\C$$
and so Corollary~\ref{cl:main} holds for $\bar V_\C$. 
Thus, using decision-theoretic terminology, we can state the following.
\begin{corollary}[partial minimax theorem]
 For every set $\C$ of strategies of the opponent the statistician has a minimax strategy.
\end{corollary}

However, the question  of whether
 the upper and the lower values coincide remains open. 
That is, we are taking the worst possible distribution over $\C$, and ask what is the best possible 
predictor  with the knowledge of  this distribution ahead of time. The question is whether ${\underline V}_\C = V_\C$.
A closely related question is  whether  there is a worst possible strategy for the opponent.
 This latter would be somehow a maximally spread-out (or maximal entropy) distribution over $\C$.
In general, measurability issues  seem to be very relevant here, especially for the maximal-entropy distribution part.

\section{Complete Class} In this section we shall see that Corollary~\ref{cl:main} can be interpreted as a complete-class theorem for asymptotic average loss, 
as well as some principled differences between the cases $V_\C>0$ and $V_\C=0$.

For a set of probability measures (strategies of the opponent) $\C$, a  predictor  %
$\rho_1$ is said to be  {\em as good as}  a predictor $\rho_2$
if $\bar L(\mu,\rho_1)\le \bar L(\mu,\rho_2)$ for all $\mu\in\C$. A predictor $\rho_1$ is {\em better (dominates)} $\rho_2$ if $\rho_1$ is as good as $\rho_2$
and $\bar L(\mu,\rho_1) < \bar L(\mu,\rho_2)$ for some $\mu\in\C$. 
A predictor $\rho$ is {\em  admissible} (also called {\em Pareto optimal}) if there is no predictor $\rho'$ which is better than $\rho$; 
otherwise it is called {\em inadmissible}.
Similarly, a set of predictors $D$ is called a {\em complete class} if for every $\rho'\notin D$ there  is $\rho\in D$ such that 
$\rho$ is better than $\rho'$. A set of of predictors $D$ is called an {\em  essentially complete class} if 
for every $\rho'\notin D$ there  is $\rho\in D$ such that 
$\rho$ is as good as $\rho'$. 
An (essentially) complete class is called {\em minimal} if none of its proper subsets is (essentially) complete.

Furthermore, in  decision-theoretic terminology, a predictor $\rho$ is called a {\em Bayes rule} for a prior $W\in\C^*$ 
if it is optimal for $W$, that is, if it attains $\inf_{\rho\in\P} E_{W(\mu)}\bar L(\mu,\rho)$. 
Clearly, if $W$ is concentrated on a finite or countable set then any mixture over this set with full support
is a Bayes rule, and the value of the $\inf$ above is 0; so the use of this terminology is non-contradictory here.

In decision theory, the complete class theorem \cite{Wald:50,Lecam:55} (see also  \cite{Ferguson:14}) states that, under certain conditions similar to those
above for the minimax theorem, the set of Bayes rules is complete and the admissible Bayes rules  form a minimal complete class.

An important difference in our set-up is that all strategies are inadmissible (unless $V_\C=0$), and one cannot speak about 
minimal (essentially) complete classes. However, the set of all Bayes rules 
is still essentially complete, and an even stronger statement holds: it is enough to consider all Bayes rules with countable priors:

\begin{corollary}[Complete class theorem]\label{th:comcla}
 For every set $\C$, the  set  of those Bayes rules %
 whose priors are concentrated on at most countable sets is essentially complete.
 There is no admissible rule (predictor) and no minimal essentially complete class
 unless $V_\C=0$. In the latter case, every 
predictor $\rho$ that attains this value is admissible and the set $\{\rho\}$ is minimal essentially complete. 
\end{corollary}
\begin{proof}
The first statement is a reformulation of Corollary~\ref{cl:main}.
 To prove the second statement, consider any $\C$ such that $V_\C>0$, take a predictor $\rho$ that attains this value (such a predictor exists by Theorem~\ref{th:main}),
and a probability measure $\mu$ such that $\bar L(\mu,\rho)>0$. Then for a predictor $\rho':=1/2(\rho+\mu)$ we have $\bar L(\mu,\rho')=0$.
Thus,  $\rho'$ is better than $\rho$: its loss is strictly smaller on one measure, $\mu$, and is at least the same on all the rest of the measures in $\C$. Therefore,  $\rho$ is inadmissible.  The statement about minimal essentially complete class is proven analogously: indeed, take any essentially complete class, $D$, and any predictor $\rho\in D$. Take then the predictor $\rho'$ constructed as above. Since $\rho'$ is better than $\rho$ and $D$ is essentially complete, there must be another predictor $\rho''\in D$,  such that $\rho''$ is as good as $\rho'$. Therefore, $D\setminus\{\rho\}$ is essentially complete and $D$ is not minimal.
The statements about the case $V_\C=0$ are self-evident.
\end{proof}
\section{Total Variation}
The total-variation case turns out to be less interesting from the decision-theoretic point of view, as the value $V_\C$ is either 0 or 1, as follows from Theorem~\ref{th:01}. Thus, it is easy to see that, for every $\C$, we have $\bar V_\C=\underline V_\C\in\{0,1\}$. The statistician has a minimax strategy in either case: in the case $ V_\C=0$ a minimax discrete mixture predictor that attains this value can be found (Theorem~\ref{th:tv}), while the other case ($V_\C=1$) is degenerate and any strategy of statistician attains this value. Thus, we have a (rather trivial) version of the minimax theorem for this case. As for the complete class theorem, it can be easily seen to carry over from the KL case (Corollary~\ref{th:comcla}) without changes.

\chapter{Middle-case: combining predictors whose loss  vanishes}\label{ch:p2}
The realizable case requires us to make a modelling assumption on the mechanism that generates the data: we assume that the latter comes from a given model set $\C$. This is rather bothering if we do not know much about the actual data source. 
So far we have two ways of dealing with this problem: one is to make the model $\C$ so large as to insure that it does include the mechanism that generates the data; we have seen that in this case an optimal predictor can still be found as a mixture of measures in $\C$, thereby hinting at a path to the solution. A different way is to be agnostic with respect to the mechanism generating the data and only consider the given set $\C$ as a comparison set of predictors. While attractive from the philosophical point of view, we have seen that in this case the set $\C$ may actually turn out to be useless, as no combination of the measures in it is optimal. 

In this chapter we consider a middle ground: we do make some assumption on the mechanism that generates the data, but only make it as strong as is necessary for the prediction problem at hand. Namely, we assume that at least one of the predictors in a given set $\C$ has a vanishing loss. 
This somehow lessens the ``belief'' that we need to have in our model set with respect to the realizable case; yet it is a stronger assumption as compared to the agnosticism of the non-realizable case. 

\begin{svgraybox}
\begin{definition}[Middle-case] Given a set of predictors $\C$, find a predictor $\rho$ such that the loss of $\rho$  vanishes on every $\nu$ on which the loss of at least some predictor in $\C$ vanishes: 
$$
\forall \nu\in\mathcal P [ (\exists\mu\in\C \ \ \bar L(\nu,\mu)=0) \Rightarrow \bar  L(\nu,\rho)=0].
$$
If such a measure $\rho$ exists then we say that it is a solution to the middle case for $\C$.
\end{definition}
\end{svgraybox}

Of course, an analogous definition can be formulated for prediction in total variation. However, as we have seen in Chapter~\ref{ch:tv}, already the realizable and the non-realizable cases of the sequence prediction problem are equivalent for this notion of loss. Therefore, it is of no surprise that the middle case is also equivalent to these. Indeed, it is an easy exercise to check that the following statement holds true. 
\begin{proposition}[Middle case for prediction in total variation]
The following statement can be added to the list of equivalent statements in Theorem~\ref{th:tv}: 
 There exists a predictor $\rho$ whose total-variation loss vanishes as long as the loss of some predictor in $\C$ vanishes:
$$
\text{\it{(vii)}} \ \ \exists\rho\forall \nu\in\mathcal P [ (\exists\mu\in\C \ \ l_{tv}(\nu,\mu)=0) \Rightarrow  l_{tb}(\nu,\rho)=0].
$$
Moreover, the same predictors as in the statements (i) and (ii) of the Theorem may be used.
\end{proposition}

We thus return our attention to the case of prediction in KL divergence.
In this case, the middle case sits firmly in-between the realizable and the non-realizable cases of the sequence prediction problem, as the next statement demonstrates.
 \begin{proposition}\label{th:comp} 
\begin{itemize}
 \item [(i)]
There exists a set $\C_1\subset\mathcal P$  
 for which there is a measure $\rho$ whose loss vanishes on every $\mu\in\C_1$ (there is a solution to the realizable case), but  there is no solution to the middle case for $\C_1$.
 \item [(ii)]   There is a set $\C_2\subset\mathcal P$ for which there is a solution to the middle case problem, 
 but there is no predictor whose regret vanishes with respect to $\C_2$. 
\end{itemize}

 \end{proposition}
\begin{proof}
 We have to provide two examples.
For the first example, it suffices to take $\C_1$ to be the set of all stationary processes $\S$; this will be demonstrated below in Theorem~\ref{th:stno}.

For the second example, first fix the binary alphabet $\X=\{0,1\}$. For each deterministic sequence $t=t_1,t_2,\dots\in\mathcal D$ construct the process measure $\gamma_t$ as follows:
$\gamma'_t(x_{n}=t_n|t_{1..n-1}):=2/3$ and for $x_{1..n-1}\ne t_{1..n-1}$ let $\gamma'_t(x_{n}=0|x_{1..n-1})=1/2$, for all $n\in\N$.
It is easy to see that $\gamma$ is a solution to the middle-case problem for the set $\C_2:=\{\gamma'_t: t\in \X^\infty\}$.
Indeed, if $\nu\in\mathcal P$ is such that $L_n(\nu,\gamma')=o(n)$ then we must have $\nu(t_{1..n})=o(1)$. From 
this and the fact that $\gamma$ and $\gamma'$ coincide (up to $O(1)$) on all other sequences we conclude $L_n(\nu,\gamma)=o(n)$.
 However, there is no predictor with vanishing regret with respect to $\C_2$.
 Indeed, for every $t\in\mathcal D$ we have $L_n(t,\gamma'_t)=n \log3/2+o(n)$. Therefore, if $\rho$ has a vanishing regret with respect to $\C_2$ 
 then $\limsup {1\over n} L_n(t,\rho)\le \log 3/2 <1$ which contradicts Lemma~\ref{th:disc}.
\qed\end{proof}

A key observation concerning the middle case is that, unlike for the non-realizable case, and similarly to the realizable case, if a solution to the middle case exists for a set of processes $\C$, then it can be obtained in the form of a mixture predictor, and the mixutre is over the set $\C$.
This is formalized in the next theorem.

To formulate it, it is useful to introduce the following relation of dominance of measures, which can be reminiscent of the relation of absolute continuity that we used to characherize the existence of a solution to all of the prediction problems in total variation. 
Let us say that $\rho$ ``dominates'' $\mu$ if $\rho$ predicts every $\nu$ such that $\mu$ predicts $\nu$, and
 Denote this relation by $\geklz$:
\begin{definition}[$\geklz$]
We write $\rho\geklz\mu$ if for every $\nu\in\mathcal P$ the equality $\bar L(\nu,\mu)=0$ implies $\bar L(\nu,\rho)=0$.
\end{definition}
The relation $\geklz$ has some similarities with $\getv$. First of all, $\geklz$ is also transitive (as can be easily seen  from the definition).
Moreover, similarly to $\getv$, one can show that for any $\mu,\rho$ any strictly convex combination $\alpha\mu+(1-\alpha)\rho$ is
a supremum of $\{\rho,\mu\}$ with respect to~$\geklz$. 

\begin{theorem}\label{th:pq4-2}[asymptotic optimality of mixture predictors for the middle-case problem] Let $\C$ be a set of probability measures on $\Omega$. If there is a measure $\rho$ such that $\rho\geklz\mu$ for every $\mu\in\C$ ($\rho$ is 
a solution to the middle-case problem)
then there is a sequence $\mu_k\in\C$, $k\in\N$, such that $\sum_{k\in\N} w_k\mu_k\geklz\mu$ for  every $\mu\in\C$, where $w_k$ are some positive weights.
\end{theorem}
The proof  is deferred to Section~\ref{s:prpq4} in the end of this chapter. This proof uses similar techniques to that of Theorem~\ref{th:main}, but is less refined since only an asymptotic result is obtained; on the other hand, we have to take specific care of measures outside of $\C$ that are predicted by processes in $\C$.

\section{Examples and an Impossibility Result for Stationary Processes}
Most of the examples of the preceding sections carry over here, since any solution to the non-realizable case for a set $\C$, that is, any predictor whose regret with respect to $\C$ is  vanishing, is also a solution to the middle-case problem. Thus, there is a solution to the middle-case problem for the set $\mathcal M$ of all finite-memory processes. 
The more interesting case is the set $\S$ of stationary processes, since, as we have seen in the previous chapter, there is no predictor whose regret vanishes with respect to this set. It turns out that this impossibility result can be strengthened, as the following theorem shows.

 Define $\Sp$ to be the set of all processes that are predicted by at least some stationary process:
$$
\Sp:=\{\mu\in\mathcal P: \exists \nu\in\S\ d(\nu,\mu)=0\}.
$$
\begin{theorem}[It is not possible to predict every process for which a stationary predictor exist]\label{th:stno}
 For any predictor $\rho\in\mathcal P$ there is a measure $\mu\in\Sp$ such that $d(\mu,\rho)\ge 1$.
\end{theorem}
\begin{proof}
 We will show that the set $\Sp$ includes the set $\mathcal D$ of all Dirac measures, that is, 
of all measures concentrated on one deterministic sequence.  The statement of the theorem follows
directly from this, since for any $\rho$ one can find a sequence $x_1,\dots,x_n,\dots\in \X^\infty$
such that $\rho(x_n|x_{1..n-1}) \le 1/2$ for all $n\in\N$.

To show that $\mathcal D\subset\Sp$, we will construct, for any given sequence $x:=x_1,\dots,x_n,\dots\in \X^\infty$,
a measure $\mu_x$ such that $d(\delta_x,\mu_x)=0$ where $\delta_x$ is the Dirac measure concentrated on~$x$.
These measures are constructed as functions of a stationary  Markov chain with a countably infinite set of states.
The construction is similar to the one used in the proof of Theorem~\ref{th:stno1} (based on the one in \cite{BRyabko:88})  but is more involved. 

The Markov chain $M$ has the set $\N$ of states. From each state $j$ it transits to the state $j+1$
with probability $p_j:={j^2}/{(j+1)^2}$ and to the state $1$ with the remaining probability, $1-p_j$.
Thus, $M$ spends most of the time around the state $1$, but takes rather long runs towards outer states: long, 
since $p_j$ tends to 1 rather fast. We need to show that it does not ``run away'' too much; more precisely, 
we need to show $M$ has a stationary distribution. For this, it is enough to show that the state $1$ is positive
recurrent (see, e.g., \cite[Chapter VIII]{Shiryaev:96} for the definitions and facts about Markov chains used
here). This can be verified directly as follows. Denote  $f_{11}^{(n)}$ the probability that starting from the state 
1 the chain returns to the state 1 for the first time in exactly $n$ steps.
 We have  
$$f_{11}^{(n)}=(1-p_n)\prod_{i=1}^{n-1}p_i=\left(1-\frac{n^2}{(n+1)^2}\right)\frac{1}{n^2}.$$
To show that the state 1 is positive recurrent we need $\left(\sum_{n=0}^\infty n f_{11}^{(n)}\right)^{-1}>0$.
Indeed, $n f_{11}^{(n)} %
<{3}/{n^2}$  which is summable.
It follows that $M$ has a stationary distribution, which we  call~$\pi$.

For   a given sequence $x:=x_1,\dots,x_n,\dots\in A^\infty$,
the measure $\mu_x$ is constructed as a function $g_x$ of the chain $M$ taken with 
its stationary distribution as the initial one. We define $g_x(j)=x_j$ for all $j\in\N$.
Since $M$ is stationary, so is~$\mu_x$. It remains to show that $d(\delta_x,\mu_x)=0$.
Indeed, we have 
\begin{multline*}
L_n(\delta_x,\mu_x)=-\log\mu_x(x_1,\dots,x_n)\le -\log\left(\pi_1\prod_{j=1}^np_j\right)
 = -\log\pi_1 + 2\log(n+1)=o(n),
\end{multline*}
which completes the proof.
\qed\end{proof}

\section{Proof of Theorem~\ref{th:pq4-2}}\label{s:prpq4}
\begin{proof}

Define the sets $C_\mu$ as the set of all measures 
$\tau\in\P$ such that $\mu$ predicts $\tau$ in expected average KL divergence.
Let $\C^+:=\cup_{\mu\in\C} C_\mu$. For each $\tau\in\C^+$ let $p(\tau)$ be any (fixed) 
$\mu\in\C$ such that $\tau\in C_\mu$. In other words,  $\C^+$ is the set of all measures 
that are predicted by some of the measures in $\C$, and for each measure $\tau$ in $\C^+$ we 
designate one ``parent'' measure $p(\tau)$ from $\C$ such that $p(\tau)$ predicts $\tau$.

Define the weights $w_k:=1/k(k+1)$, for all $k\in\N$. %

\noindent{\em Step 1.} 
 For each $\mu\in \C^+$ let $\delta_n$ be any monotonically increasing function such that $\delta_n(\mu)=o(n)$ and $L_n(\mu,p(\mu))=o(\delta_n(\mu))$.
Define the sets 
\begin{equation}\label{eq:U}
U_\mu^n:=\left\{x_{1..n}\in \X^n: \mu(x_{1..n})\ge{1\over n}\rho(x_{1..n})\right\},
\end{equation} 
\begin{equation}\label{eq:V}
V_\mu^n:=\left\{x_{1..n}\in \X^n: p(\mu)(x_{1..n})\ge2^{-\delta_n(\mu)}\mu(x_{1..n})\right\},
\end{equation} 
and
\begin{equation}\label{eq:t-2}
T_\mu^n:=U_\mu^n\cap V_\mu^n.
\end{equation} 
We will upper-bound $\mu(T_\mu^n)$.
First, using Markov's inequality, we derive 
\begin{equation}\label{eq:markk-2}
\mu(\X^n\backslash U_\mu^n) 
 = \mu \left(\frac {\rho(x_{1..n})}{\mu(x_{1..n})} > n\right)\le {1\over n} E_\mu \frac {\rho(x_{1..n})}{\mu(x_{1..n})}={1\over n}.
\end{equation}
Next, observe that for every $n\in\N$ and every set $A\subset \X^n$, using Jensen's inequality we can obtain
\begin{multline}\label{eq:jen-2}
-\sum_{x_{1..n}\in A}\mu(x_{1..n})\log\frac{\rho(x_{1..n})}{\mu(x_{1..n})}
=  -\mu(A)\sum_{x_{1..n}\in A}{1\over\mu(A)}\mu(x_{1..n})\log\frac{\rho(x_{1..n})}{\mu(x_{1..n})}
\\
\ge -\mu(A)\log{\rho(A)\over\mu(A)} \ge -\mu(A)\log\rho(A) -{1\over2}. 
\end{multline}
Moreover,
\begin{multline*}\label{eq:anoth-2}
 L_n(\mu,p(\mu))=   -\sum_{x_{1..n}\in\X^n\backslash V_\mu^n }\mu(x_{1..n})\log\frac{p(\mu)(x_{1..n})}{\mu(x_{1..n})}  
\\
-\sum_{x_{1..n}\in V_\mu^n}\mu(x_{1..n})\log\frac{p(\mu)(x_{1..n})}{\mu(x_{1..n})} \ge \delta_n(\mu_n)\mu(\X^n\backslash V_\mu^n)-1/2,
\end{multline*}
where in the inequality we have used~(\ref{eq:V}) for the first summand  and~(\ref{eq:jen-2}) for the second.
Thus,
\begin{equation}\label{eq:del-2}
\mu(\X^n\backslash V_\mu^n) \le \frac{L_n(\mu,p(\mu))+1/2}{\delta_n(\mu)}=o(1).
\end{equation}
From~(\ref{eq:t-2}), (\ref{eq:markk-2}) and~(\ref{eq:del-2}) we conclude
\begin{equation}\label{eq:mark-2}
\mu(\X^n\backslash T_\mu^n) \le \mu(\X^n\backslash V_\mu^n) + \mu(\X^n\backslash U_\mu^n) =o(1).
\end{equation}

{\em Step 2n: a countable cover, time $n$.}
Fix an $n\in\N$. Define $m^n_1:=\max_{\mu\in\C}\rho(T_\mu^n)$ (since $\X^n$ are finite all suprema are reached). 
 Find any $\mu^n_1$ such that $\rho^n_1(T_{\mu^n_1}^n)=m^n_1$ and let
$T^n_1:=T^n_{\mu^n_1}$. For $k>1$, let $m^n_k:=\max_{\mu\in\C}\rho(T_\mu^n\backslash T^n_{k-1})$. If $m^n_k>0$, let $\mu^n_k$ be any $\mu\in\C$ such 
that $\rho(T_{\mu^n_k}^n\backslash T^n_{k-1})=m^n_k$, and let $T^n_k:=T^n_{k-1}\cup T^n_{\mu^n_k}$; otherwise let $T_k^n:=T_{k-1}^n$. Observe that 
(for each $n$) there is only a finite number of positive $m_k^n$,
since the set $\X^n$ is finite; let $K_n$ be the largest index $k$ such that $m_k^n>0$. Let 
\begin{equation}\label{eq:nun-2}
\nu_n:=\sum_{k=1}^{K_n} w_kp(\mu^n_k).
\end{equation}
As a result of this construction, for every $n\in\N$ every $k\le K_n$ and every  $x_{1..n}\in T^n_k$ 
using the definitions~(\ref{eq:t-2}), (\ref{eq:U}) and~(\ref{eq:V})  we obtain
\begin{equation}\label{eq:ext-2}
\nu_n(x_{1..n})\ge w_k{1\over n}2^{-\delta_n(\mu)}\rho(x_{1..n}).
\end{equation}

{\em Step 2: the resulting predictor.}
Finally, define 
\begin{equation}\label{eq:nu-2}
\nu:={1\over 2}\gamma+{1\over2}\sum_{n\in\N}w_n\nu_n,
\end{equation}
 where $\gamma$ is the i.i.d.\ measure with equal probabilities of all $x\in\X$ 
(that is, $\gamma(x_{1..n})=|\X|^{-n}$ for every $n\in\N$ and every $x_{1..n}\in\X^n$). 
We will show that $\nu$  predicts every $\mu\in\C^+$, and 
then in the end of the proof (Step~r) we will show how to replace $\gamma$ by a combination of a countable set of elements of $\C$ (in fact, $\gamma$ 
is just a regularizer which ensures that $\nu$-probability of any word is never too close to~0). 

{\em Step 3: $\nu$ predicts every $\mu\in\C^+$.}
Fix any $\mu\in\C^+$. 
Introduce the parameters $\epsilon_\mu^n\in(0,1)$, $n\in\N$, to be defined later, and let $j_\mu^n:=1/\epsilon_\mu^n$.
Observe that $\rho(T^n_k\backslash T^n_{k-1})\ge \rho(T^n_{k+1}\backslash T^n_k)$, for any $k>1$ and any $n\in\N$, by definition of these sets.
Since the sets $T^n_k\backslash T^n_{k-1}$, $k\in\N$ are disjoint, we obtain $\rho(T^n_k\backslash T^n_{k-1})\le 1/k$. Hence,  $\rho(T_\mu^n\backslash T_j^n)\le \epsilon_\mu^n$ for some $j\le j_\mu^n$,
since  otherwise $m^n_j=\max_{\mu\in\C}\rho(T_\mu^n\backslash T^n_{j_\mu^n})> \epsilon_\mu^n$ so that  $\rho(T_{j_\mu^n+1}^n\backslash T^n_{j_\mu^n}) > \epsilon_\mu^n=1/j_\mu^n$, which is a contradiction. 
Thus,   
\begin{equation}\label{eq:tm-2}
\rho(T_\mu^n\backslash T_{j_\mu^n}^n)\le \epsilon_\mu^n.
\end{equation}
We can upper-bound $\mu(T_\mu^n\backslash T^n_{j^n_\mu})$ as follows. 
First, observe that
\begin{multline}\label{eq:mut-2}
L_n(\mu,\rho) 
=   -\sum_{x_{1..n}\in T^n_\mu\cap T^n_{j^n_\mu}}\mu(x_{1..n})\log\frac{\rho(x_{1..n})}{\mu(x_{1..n})} 
\\
-\sum_{x_{1..n}\in T^n_\mu\backslash  T^n_{j^n_\mu}}\mu(x_{1..n})\log\frac{\rho(x_{1..n})}{\mu(x_{1..n})} \\- \sum_{x_{1..n}\in \X^n\backslash T^n_\mu}\mu(x_{1..n})\log\frac{\rho(x_{1..n})}{\mu(x_{1..n})}
\\
=
I+II+III.
\end{multline}
Then, from~(\ref{eq:t-2}) and~(\ref{eq:U}) we get 
\begin{equation}\label{eq:e1-2}
I\ge -\log n.
\end{equation}
From~(\ref{eq:jen-2}) and~(\ref{eq:tm-2})
we get 
\begin{equation}\label{eq:e2-2}
II
\ge  -\mu(T_\mu^n\backslash T^n_{j^n_\mu}) \log\rho(T_\mu^n\backslash T^n_{j^n_\mu})- 1/2
\ge -\mu(T_\mu^n\backslash T^n_{j^n_\mu}) \log \epsilon_\mu^n - 1/2.
\end{equation}
Furthermore,  
\begin{multline}\label{eq:e3-2}
III\ge \sum_{x_{1..n}\in \X^n\backslash T^n_\mu}\mu(x_{1..n})\log\mu(x_{1..n}) \\
\ge \mu(\X^n\backslash T^n_\mu)\log\frac{\mu(\X^n\backslash T^n_\mu)}{|\X^n\backslash T^n_\mu|}\ge -{1\over2} - \mu(\X^n\backslash T^n_\mu)n\log|\X|,
\end{multline} 
where the first inequality is obvious, in the second inequality we have used the fact that entropy is maximized when all events are equiprobable 
and in the third one we used $|\X^n\backslash T^n_\mu|\le|\X|^n$.
Combining~(\ref{eq:mut-2}) with the bounds (\ref{eq:e1-2}), (\ref{eq:e2-2}) and~(\ref{eq:e3-2})  we obtain 
\begin{equation*}
L_n(\mu,\rho) \ge -\log n  -\mu(T_\mu^n\backslash T^n_{j^n_\mu}) \log \epsilon_\mu^n  - 1 - 
  \mu(\X^n\backslash T^n_\mu)n\log|\X|,
\end{equation*}
so that
\begin{equation}\label{eq:mu2-2}
 \mu(T_\mu^n\backslash T^n_{j^n_\mu}) \le {1\over-\log \epsilon_\mu^n}\Big(L_n(\mu,\rho) +\log n  +1 + 
     \mu(\X^n\backslash T^n_\mu)n\log|\X| \Big).
\end{equation}
From the fact that $L_n(\mu,\rho)=o(n)$ and~(\ref{eq:mark-2}) it follows that the term in brackets is $o(n)$, so that 
 we can define the parameters $\epsilon^n_\mu$ in such a way that $-\log \epsilon^n_\mu=o(n)$ while
at the same time  the bound~(\ref{eq:mu2-2}) gives $\mu(T_\mu^n\backslash T^n_{j^n_\mu})=o(1)$. Fix such a choice of $\epsilon^n_\mu$.
Then,   using~(\ref{eq:mark-2}),  we  conclude
\begin{equation}\label{eq:xt-2}
\mu(\X^n\backslash T^n_{j^n_\mu})\le \mu(\X^n\backslash T^n_{\mu})+ \mu(T^n_{\mu}\backslash T^n_{j^n_\mu}) =o(1).
\end{equation}

We proceed with the proof of $L_n(\mu,\nu)=o(n)$. 
For any $x_{1..n}\in T^n_{j_\mu^n}$  we have
\begin{equation}\label{eq:i-2}
\nu(x_{1..n})\ge {1\over 2}w_n\nu_n(x_{1..n})
\ge{1\over 2}w_n w_{j_{\mu}^n} {1\over n}2^{-\delta_n(\mu)}\rho(x_{1..n})
\ge\frac{w_n}{4n}(\epsilon_\mu^n)^22^{-\delta_n(\mu)}\rho(x_{1..n}),
\end{equation}
where the first inequality follows from~(\ref{eq:nu-2}), the second from~(\ref{eq:ext-2}), and in the third we have used $w_{j_{\mu}^n}=1/(j_{\mu}^n)(j_{\mu}^n+1)$
and  $j_\mu^n=1/\epsilon^\mu_n$.
 Next we use the decomposition
\begin{equation}\label{eq:12-2}
L_n(\mu,\nu)= -\sum_{x_{1..n}\in T^n_{j_\mu^n}}\mu(x_{1..n})\log\frac{\nu(x_{1..n})}{\mu(x_{1..n})} %
- \sum_{x_{1..n}\in \X^n\backslash T^n_{j_\mu^n}}\mu(x_{1..n})\log\frac{\nu(x_{1..n})}{\mu(x_{1..n})}  = I + II.
\end{equation}
From~(\ref{eq:i-2})  we find 
\begin{multline}\label{eq:1-2}
I\le -\log\left(\frac{w_n}{4n}(\epsilon_\mu^n)^22^{-\delta_n(\mu)}\right)  - \sum_{x_{1..n}\in T^n_{j_\mu^n}}\mu(x_{1..n})\log\frac{\rho(x_{1..n})}{\mu(x_{1..n})}\hfill\\
=
(o(n) - 2\log\epsilon_\mu^n + \delta_n(\mu)) %
 +\left(L_n(\mu,\rho)+ \sum_{x_{1..n}\in\X^n\backslash T^n_{j_\mu^n}}\mu(x_{1..n})\log\frac{\rho(x_{1..n})}{\mu(x_{1..n})}\right)
\\
\le o(n) -  \sum_{x_{1..n}\in\X^n\backslash T^n_{j_\mu^n}}\mu(x_{1..n})\log\mu(x_{1..n})\\ 
\le o(n)+\mu(\X^n\backslash T^n_{j_\mu^n})n\log|\X|=o(n),
\end{multline}
where in the second inequality we have used $-\log\epsilon_\mu^n=o(n)$, $L_n(\mu,\rho)=o(n)$ and $\delta_n(\mu)=o(n)$, in the last inequality we have again used the fact that the entropy is maximized when all events are equiprobable, 
while the last equality follows from~(\ref{eq:xt-2}). 
Moreover, from~(\ref{eq:nu-2}) we find
\begin{equation}\label{eq:2-2}
II\le \log 2 - \sum_{x_{1..n}\in\X^n\backslash  T^n_{j_\mu^n}}\mu(x_{1..n})\log\frac{\gamma(x_{1..n})}{\mu(x_{1..n})}
\le 1 +n\mu(\X^n\backslash T^n_{j_\mu^n})\log|\X|=o(n),
\end{equation}
where in the last inequality we have used $\gamma(x_{1..n})=|\X|^{-n}$ and $\mu(x_{1..n})\le 1$, and the last equality follows from~(\ref{eq:xt-2}).

From~(\ref{eq:12-2}), (\ref{eq:1-2}) and~(\ref{eq:2-2}) we conclude ${1\over n}L_n(\nu,\mu)\to0$.

{\em Step r: the regularizer $\gamma$}. It remains to show that the  i.i.d.\ regularizer $\gamma$ in the definition of $\nu$~(\ref{eq:nu-2}), can be replaced by a convex combination of a countably many elements from $\C$.
This can be done exactly as in the corresponding step (Step r) of the proof of Theorem~\ref{th:main} (Section~\ref{s:prthmain}).
\qed\end{proof}

\chapter{Conditions under which one measure is a predictor for another }\label{ch:other}
In this chapter we consider some questions aimed towards generalizing the results presented in the other chapters towards different losses.
One of the main features of the losses considered, bot total variation and expected average KL divergence, is that predictive quality is preserved under summation. That is, if a measure $\rho$ predicts a measure $\mu$, then any mixture predictor that involves the measure $\rho$, in the simplest case, $1/2\rho+1/2\chi$ where $\chi$ is any other measure, also predicts $\mu$. For prediction in expected average KL divergence this follows from its equivalence with absolute continuity (Theorem~\ref{th:bd}): absolute continuity is preserved under summation with arbitrary measure as follows directly from its definition (Definition~\ref{d:ac}).  For KL divergence, it is guaranteed by~\ref{eq:tric}.
Here we show that for other losses this is not necessarily the case. 

Thus, the first question
we consider in this chapter is the following: suppose that a measure $\rho$
predicts $\mu$ (in some sense), and let $\chi$ be some other
probability measure. %
 Does the measure $\rho'=\odt(\rho+\chi)$ still predict
$\mu$? That is, we ask to which prediction quality criteria does
the idea of taking a Bayesian sum generalize.

Another way to look for generalizations of the results presented is by asking the following question.
What are the conditions under which a measure $\rho$ is a good predictor 
for a measure $\mu$?
Towards this end, we start with the
following observation. For a  mixture predictor
$\rho$ of a countable class of measures $\nu_i$, $i\in\N$, we
have $\rho(A)\ge w_i\nu_i(A)$ for any $i$ and any measurable set
$A$, where $w_i$ is a constant. This condition is stronger than
the assumption of absolute continuity (Definition~\ref{d:ac}) and is sufficient for
prediction in a very strong sense. Since we are interested in prediction in 
 a weaker sense, let us make a weaker assumption.
\begin{definition}[Dominance with decreasing coefficients] Say
that {\em a measure $\rho$ dominates a measure $\mu$ with
coefficients $c_n>0$} if
\begin{equation} \label{eq:dom}
  \rho(x_1,\dots,x_n) \;\geq\; c_n \mu(x_1,\dots,x_n)
\end{equation}
for all $x_1,\dots,x_n$.
\end{definition}

Thus, the second question we consider in this chapter 
 is: under what conditions on $c_n$ does (\ref{eq:dom})
imply that $\rho$ predicts $\mu$? Observe that if
$\rho(x_1,\dots,x_n)>0$ for any $x_1,\dots,x_n$ then any measure
$\mu$ is {\em locally} absolutely continuous with respect to
$\rho$ (that is, the measure $\mu$ restricted to the first $n$
trials $\mu|_{\X^n}$ is absolutely continuous w.r.t.\
$\rho|_{\X^n}$ for each $n$), and moreover, for any measure $\mu$
some constants $c_n$ can be found that satisfy (\ref{eq:dom}). For
example, if $\rho$ is Bernoulli i.i.d.\ measure with parameter
$\odt$ and $\mu$ is any other measure, then (\ref{eq:dom}) is
(trivially) satisfied with $c_n=2^{-n}$. Thus we know that if
$c_n\equiv c$ then $\rho$ predicts $\mu$ in a very strong sense,
whereas exponentially decreasing $c_n$ are not enough for
prediction. Perhaps somewhat surprisingly, we will show that
dominance with any subexponentially decreasing coefficients is
sufficient for prediction in expected average KL divergence. Dominance with any polynomially decreasing
coefficients, and also with coefficients decreasing (for example)
as $c_n=\exp(-\sqrt{n}/\log n)$, is sufficient for (almost sure)
prediction on average (i.e.\ in Cesaro sense). However, for
prediction on every step we have a negative result: for any
dominance coefficients that go to zero there exists a pair of
measures $\rho$ and $\mu$ which satisfy~(\ref{eq:dom}) but $\rho$
does not predict $\mu$ in the sense of almost sure convergence of
probabilities. Thus the situation is similar to that for
predicting any stationary measure: prediction is possible in the
average but not on every step.

Note also that for Laplace's measure $\rho_L$ it can be shown that
$\rho_L$ dominates any i.i.d.\ measure $\mu$ with linearly
decreasing coefficients $c_n={1\over n+1}$; a generalization of
$\rho_L$ for predicting all measures with memory $k$ (for a given
$k$) dominates them with polynomially  decreasing coefficients.
Thus dominance with decreasing coefficients generalizes (in a
sense) predicting countable classes of measures (where we have
dominance with a constant), absolute continuity (via local
absolute continuity), and predicting i.i.d.\ and finite-memory
measures.

Thus, in this chapter,  we address the following two questions for several different losses apart from those considered in the rest of this volume. Is dominance with
decreasing coefficients sufficient for prediction in some sense,
under some  conditions on the coefficients (Section~\ref{sec:dom})? And, if a measure
$\rho$ predicts a measure $\mu$ in some sense, does the measure
$\odt(\rho+\chi)$ also predict $\mu$ in the same sense, where
$\chi$ is an arbitrary measure (Section~\ref{sec:sum})?

The rest of this chapter is organized as follows.
 Section~\ref{s:pmore}
introduces the  measures of divergence of
probability measures (losses) that we will consider. Section~\ref{sec:dom} addresses the question
of whether dominance with decreasing coefficients is sufficient
for prediction, while in Section~\ref{sec:sum} we consider the
problem of summing a predictor with an arbitrary measure. 
\section{Measuring performance of prediction}\label{s:pmore}

In addition to those introduced in Section~\ref{s:kl}, for two measures $\mu$ and $\rho$ define the following measures of divergence
\begin{itemize}%
\item[($\delta$)]{ Kullback-Leibler (KL) divergence}\\ $\displaystyle  \delta_n(\mu,\rho|x_{<n})=\sum_{x\in\X}\mu(x_n=x|x_{<n})\log\frac{\mu(x_n=x|x_{<n})}{\rho(x_n=x|x_{<n})}$,
\item[($\bar d$)]{ average KL divergence} $\displaystyle  \bar d_n(\mu,\rho|x_{1..n})= {1\over n}  \sum_{t=1}^n \delta_t(\mu,\rho|x_{<t})$; \\ note that this is our expected average KL divergence $L_n$~\eqref{eq:kl} without the expectation, i.e.~${1\over n}L_n(\mu,\rho)=\E_\mu\bar d_n(\mu,\rho|x_{1..n})$
\item[($a$)]{ absolute distance} $\displaystyle a_n(\mu,\rho|x_{<n})=\sum_{x\in\X}|\mu(x_n=x|x_{<n})-\rho(x_n=x|x_{<n})|$,
\item[($\bar a$)]{ average absolute distance} $\displaystyle \bar a_n(\mu,\rho|x_{1..n})={1\over n}\sum_{t=1}^n a_t(\mu,\rho|x_{<t})$.
\end{itemize}

\begin{definition}\label{def:conv}
We say that $\rho$ predicts $\mu$
\begin{itemize}\itemindent=2ex
\item[$(d)$\ ] \     in (non-averaged) {KL divergence} if $\delta_n(\mu,\rho|x_{<n})\rightarrow0$   $\mu$-a.s. as $t\rightarrow\infty$,
\item[$(\bar d)$\ ] \  in (time-average) {average KL divergence} if $\bar d_n(\mu,\rho|x_{1..n})\rightarrow 0$ $\mu$-a.s., 
\item[$(\E\bar d)$]  \ \ \ \ in {expected average KL divergence} if $\bar L(\mu,\rho)=0$, %
\item[$(a)$\ ]   \   in {absolute distance} if $a_n(\mu,\rho|x_{<n})\rightarrow0$  $\mu$-a.s., 
\item[$(\bar a)$\ ]\  in {average absolute distance} if $\bar a_n(\mu,\rho|x_{1..n})\rightarrow 0$ $\mu$-a.s., 
\item[$(\E\bar a)$]\ \ \  in {expected average absolute distance} if $\E_\mu\bar a_n(\mu,\rho|x_{1..n})\rightarrow 0$,
\item[$(tv)$] \ \ in { total variation} if $v(\mu,\rho,x_{1..n})\to0\mu-\as$  (Definition~\ref{d:ctv}.)
\end{itemize}
\end{definition}

The argument $x_{1..n}$ will be often left implicit in our
notation.
 The following implications hold (and are
complete and strict): 
\begin{equation}
 \begin{array}{ccccccc}
     &             & \delta & \Rightarrow & \bar d &             & \E\bar d \\
     &         & \Downarrow  &   & \Downarrow &             & \Downarrow \\
  tv & \Rightarrow & a & \Rightarrow & \bar a & \Rightarrow & \E\bar a \\
\end{array}%
\end{equation}
to be understood as e.g.: if $\bar d_n\to 0$ a.s.\ then $\bar a_n\to 0$
a.s, or, if $\E\bar d_n\to 0$ then $\E\bar a_n\to 0$. The
horizontal implications $\Rightarrow$ follow immediately from the
definitions,
and the $\Downarrow$ follow from the following Lemma:

\begin{lemma}\label{th:da}
For all measures $\rho$ and $\mu$ and sequences $x_{1..\infty}$ we
have: $a_t^2\leq 2\delta_t$ and $\bar a_n^2\leq 2\bar d_n$ and $(\E\bar
a_n)^2\leq 2\E\bar d_n$.
\end{lemma}

\begin{proof}
Pinsker's inequality \cite[Lem.3.11$a$]{Hutter:04uaibook} implies
$a_t^2\leq 2\delta_t$. Using this and Jensen's inequality for the average
$\odn\sum_{t=1}^n [...]$ we get
\begin{equation}
  2\bar d_n
  \;=\; {1\over n}\sum_{t=1}^n 2\delta_t
  \;\geq\;   {1\over n}\sum_{t=1}^n a_t^2
  \;\geq\;  \left({1\over n}\sum_{t=1}^n a_t\right)^2
  \;=\; \bar a_n^2
\end{equation}
Using this and Jensen's inequality for the expectation $\E$ we get
$2\E\bar d_n\geq \E\bar a_n^2 \geq (\E\bar a_n)^2$.
\end{proof}

\section{Preservation of the predictive ability under summation with an arbitrary measure}\label{sec:sum}

Now we turn to the question whether, given a measure $\rho$ that
predicts a measure $\mu$ in some sense, the ``contaminated''
measure $(1-\epsilon)\rho+\eps\chi$ for some $0<\epsilon<1$ also predicts
$\mu$ in the same sense, where $\chi$ is an arbitrary probability
measure. Since most considerations are independent of the choice
of $\eps$, in particular the results in this section, we set
$\eps=\odt$ for simplicity. We define

\begin{definition}\label{def:cont}
By ``$\rho$ contaminated with $\chi$'' we mean
$\rho':=\odt(\rho+\chi)$, where $\rho$ and $\chi$ are probability
measures.
\end{definition}

Positive results  can be obtained for  convergence in expected
average KL divergence. The statement of the next proposition in a
different form was mentioned in \cite{BRyabko:06a,Hutter:06usp}.
Since the proof is simple we present it here for the sake of
completeness; it is based on the same ideas as the proof of
Theorem~\ref{th:dom}.

\begin{proposition}\label{th:expaverklsum}
Let $\mu$ and $\rho$ be two measures on $\X^\infty$ and suppose
that $\rho$ predicts $\mu$ in expected average KL divergence. Then
so does the measure $\rho'=\odt(\rho +\chi)$ where $\chi$
is any other measure on $\X^\infty$.
\end{proposition}

\begin{proof}
The statement follows directly from~\eqref{eq:tric}.
\end{proof}

Next we consider some  negative results. An example of measures
$\mu$, $\rho$ and $\chi$ such that $\rho$ predicts $\mu$ in
absolute distance (or KL divergence) but $\odt(\rho+\chi)$
does not, can be constructed similarly to the example from
\cite{Kalai:92} (of a measure $\rho$ which is a sum of
distributions arbitrarily close to $\mu$ yet does not predict it).
The idea is to take a measure $\chi$ that predicts $\mu$ much
better than $\rho$ on almost all steps, but on some steps  gives
grossly wrong probabilities.

\begin{proposition}\label{th:nosumad}
There exist measures $\mu$, $\rho$ and $\chi$ such that $\rho$
predicts $\mu$ in absolute distance (KL divergence) but
$\odt(\rho+\chi)$ does not predict $\mu$ in absolute
distance (KL divergence).
\end{proposition}

\begin{proof}
Let $\mu$ be concentrated on the sequence $11111\dots$ (that is
$\mu(x_n=1)=1$ for any $n$), and let $\rho(x_n=1)={n\over n+1}$
with probabilities independent on different trials. Clearly,
$\rho$ predicts $\mu$ in both absolute distance and KL divergence.
Let $\chi(x_n=1)=1$ for all $n$ except on the sequence
$n=n_k=2^{2^k}=n_{k-1}^2$, $k\in\SetN$ on which
$\chi(x_{n_k}=1)=n_{k-1}/n_k=2^{-2^{k-1}}$. This implies that
$\chi(1_{1..n_k})=2/n_k$ and
$\chi(1_{1..n_k-1})=\chi(1_{1..n_{k-1}})=2/n_{k-1}=2/\sqrt{n_k}$. It
is now easy to see that $\odt(\rho+\chi)$ does not predict
$\mu$, neither in absolute distance nor in KL divergence. Indeed
for $n=n_k$ for some $k$ we have
\beqn
  \odt(\rho\!+\!\chi)(x_{n}=1|1_{<n})
  \;=\; \frac{\rho(1_{1..n})+\chi(1_{1..n})}{\rho(1_{<n})+\chi(1_{<n})}
  \;\leq\; \frac{1/(n\!+\!1) + 2/n} {1/n \;+\; 2/\sqrt{n}}
  \;\rightarrow\; 0.
\eeqn
\end{proof}

For the (expected) average absolute distance the negative result also holds:

\begin{proposition}\label{th:nosumavad}
There exist such measures $\mu$, $\rho$ and $\chi$ that $\rho$
predicts $\mu$ in  average absolute distance but
$\odt(\rho+\chi)$ does not predict $\mu$ in (expected)
average absolute distance.
\end{proposition}

\begin{proof}
Let $\mu$ be Bernoulli 1/2 distribution and let $\rho(x_n=1)=1/2$
for all $n$ (independently of each other) except for some sequence
$n_k$, $k\in\SetN$ on which $\rho(x_{n_k}=1)=0$. Choose $n_k$
sparse enough for $\rho$ to predict $\mu$ in the average absolute
distance. Let $\chi$ be Bernoulli  1/3. Observe that $\chi$
assigns non-zero probabilities to all finite sequences, whereas
$\mu$-a.s.\ from some $n$ on  $\rho(x_{1..n})=0$. Hence
$\odt(\rho+\chi)(x_{1..n})=\odt\chi(x_{1..n})$ and so
$\odt(\rho+\chi)$ does not predict $\mu$. 
\end{proof}

Thus  for the question of whether predictive ability is preserved
when an arbitrary measure is added to the predictor, we
have the following table of answers.

\vskip 2mm
\begin{center}
\begin{tabular}{| c|c |c|c|c|c|c|}\hline
 $\E \bar d_n $ &  $\bar d_n \vphantom{\bar{\hat d}} $ &  $\delta_n$ &  $\E\bar a_n$& $\bar a_n$& $a_n$ \\\hline
 + & ? &  $-$ & $-$ & $-$ & $-$ \\\hline
\end{tabular}
\end{center}
\vskip 2mm

As it can be seen, there remains one open question: whether this
property is preserved under almost sure convergence of the average
KL divergence.

It can be inferred from the example in
Proposition~\ref{th:nosumad} that contaminating   a predicting
measure $\rho$ with a measure $\chi$ spoils $\rho$ if $\chi$ is
better than $\rho$ on almost every step. It thus can be
conjectured that adding a measure can only spoil a predictor on
sparse steps, not affecting the average.

\section{Dominance with decreasing coefficients}\label{sec:dom}

First we consider the question whether property (\ref{eq:dom}) is
sufficient for prediction.

\begin{definition}\label{def:dom}
We say that a measure $\rho$ dominates a measure $\mu$
with coefficients $c_n>0$ if
\begin{equation}
 \rho(x_{1..n}) \;\geq\; c_n \mu(x_{1..n})
\end{equation}
 for all $x_{1..n}$.
\end{definition}

Suppose that $\rho$ dominates $\mu$ with decreasing coefficients
$c_n$. Does $\rho$ predict $\mu$ in (expected, expected average)
KL divergence (absolute distance)?
First let us give an example.

\begin{proposition}\label{prop:Laplace}
Let $\rho_L$ be the Laplace measure, given by
$\rho_L(x_{n+1}=a|x_{1..n})=\frac{k+1}{n+|\X|}$ for any $a\in\X$
and any $x_{1..n}\in\X^n$, where $k$ is the number of occurrences
of $a$ in $x_{1..n}$ (this is also well defined for $n=0$). Then
\begin{equation}
  \rho_L(x_{1..n}) \;\geq\; \frac{n!}{(n+|\X|-1)!} \; \mu({x_{1..n}})
\end{equation}
for any measure $\mu$ which generates independently and
identically distributed symbols. %
The equality is attained for some choices of $\mu$.
\end{proposition}

\begin{proof}
We will only give the proof for $\X=\{0,1\}$, the general case is
analogous. To calculate $\rho_L(x_{1..n})$ observe that it only
depends on the number of 0s and 1s in $x_{1..n}$ and not on their
order. Thus we compute $\rho_L(x_{1..n})=\frac{k!(n-k)!}{(n+1)!}$
where $k$ is the number of 1s. For any measure $\mu$ such that
$\mu(x_n=1)=p$ for some $p\in[0,1]$ independently for all $n$, and
for Laplace measure $\rho_L$ we have
\bqan
  \frac{\mu(x_{1..n})}{\rho_L(x_{1..n})}
  & = & \frac{(n+1)!}{k!(n-k)!}p^k(1-p)^{n-k}
  \;=\; (n+1){n\choose k}p^k(1-p)^{n-k} \\
  & \le & (n+1) \sum_{k=0}^n{n\choose k}p^k(1-p)^{n-k}
  \;=\; n+1,
\eqan
for any $n$-letter word $x_1,\dots,x_n$  where $k$ is the number of 1s in it.
The equality in the bound is attained when $p=1$, so that $k=n$, $\mu(x_{1..n})=1$,
and $\rho_L(x_{1..n})=\frac{1}{n+1}$. 
\end{proof}

Thus for Laplace's measure $\rho_L$ and binary $\X$ we have
$c_n=\O(\frac{1}{n})$. As mentioned above, in
general, exponentially decreasing  coefficients $c_n$ are not
sufficient for prediction, since (\ref{eq:dom}) is satisfied with
$\rho$ being a Bernoulli i.i.d.\ measure and $\mu$ any other
measure. On the other hand, the following proposition shows that
in a weak sense of convergence in expected average KL divergence
(or absolute distance) the property~(\ref{eq:dom}) with
subexponentially decreasing $c_n$ is sufficient. It is also worth reminding that
that  if $c_n$ are bounded from below then prediction in the
strong sense of total variation is possible.

\begin{theorem}\label{th:dom}
Let $\mu$ and $\rho$ be two measures on $\X^\infty$ and suppose that
$
 \rho(x_{1..n})\ge c_n \mu(x_{1..n})
$
for any $x_{1..n}$, where $c_n$ are positive constants satisfying
$\frac{1}{n}\log c_n^{-1}\rightarrow0$. Then $\rho$ predicts
$\mu$ in expected average KL divergence $L(\mu,\rho)=0$ and in expected average absolute
distance $\E_\mu \bar a_n(\mu,\rho)\rightarrow0$.
\end{theorem}

\begin{proof}
For convergence in average expected KL divergence, using~(\ref{eq:kl}) we derive
\begin{equation*}
    \frac{1}{n}L_n(\mu,\rho)= \frac{1}{n}\E \log \frac {\mu(x_{1..n})}{\rho(x_{1..n})}
  \le \frac{1}{n}  \log c_n^{-1}\rightarrow0.
\end{equation*}

 The statement for expected average distance follows from this  and \\
Lemma~\ref{th:da}. 
\end{proof}

With a stronger condition on $c_n$ prediction in average KL
divergence can be established.

\begin{theorem}\label{th:dom2}
Let $\mu$ and $\rho$ be two measures on $\X^\infty$ and suppose that
$
 \rho(x_{1..n})\ge c_n \mu(x_{1..n})
$
for every $x_{1..n}$,
where $c_n$ are positive constants satisfying
\beq\label{eq:domsum}
 \sum_{n=1}^{\infty} \frac {(\log c_n^{-1})^2}{n^2} \;<\; \infty.
\eeq
Then $\rho$ predicts $\mu$  in  average KL divergence $\bar
d_n(\mu,\rho)\rightarrow0$ $\mu$-a.s.\ and in  average absolute
distance $\bar a_n(\mu,\rho)\rightarrow0$ $\mu$-a.s.
\end{theorem}

In particular, the condition~(\ref{eq:domsum}) on the coefficients
is satisfied for polynomially decreasing coefficients, or for
$c_n=\exp(-\sqrt{n}/\log n)$.
\begin{proof}
Again the second statement (about absolute distance) follows from
the first one and Lemma~\ref{th:da}, so that we only have to prove
the statement about KL divergence.

Introduce the symbol $\E^n$ for $\mu$-expectation over $x_n$
conditional on $x_{<n}$. Consider random variables
$l_n=\log\frac{\mu(x_n|x_{<n})}{\rho(x_n|x_{<n})}$ and $\bar l_n=
{1\over n} \sum_{t=1}^n l_t$. Observe that $\delta_n=\E^n l_n$, so that
the random variables $m_n=l_n-\delta_n$ form a martingale difference
sequence  (that is, $\E^n m_n=0$) with respect to the standard
filtration defined by $x_1,\dots,x_n,\dots$. Let also $\bar m_n=
{1\over n} \sum_{t=1}^n m_t$. We will show that $\bar
m_n\rightarrow0$ $\mu$-a.s.\ and $\bar l_n\rightarrow0$ $\mu$-a.s.
which implies $\bar d_n\rightarrow 0$ $\mu$-a.s.

Note that
\beqn
 \bar l_n \;=\; {1\over n}\log \frac{\mu(x_{1..n})}{\rho(x_{1..n})}
 \;\leq\; \frac{\log c_n^{-1}}{n} \;\rightarrow\; 0.
\eeqn Thus to show that $\bar l_n$ goes to $0$ we need to bound it
from below. It is easy to see that $n\bar l_n$ is ($\mu$-a.s.)
bounded from below by a constant, since $\frac{\rho (x_{1..n})}{\mu
(x_{1..n})}$ is a positive $\mu$-martingale whose expectation is 1,
and so it converges to a finite limit $\mu$-a.s.\ by Doob's
submartingale convergence theorem, see e.g.\
\cite[p.508]{Shiryaev:96}.

Next we will show that $\bar m_n\rightarrow0$ $\mu$-a.s.
 We have
\bqan
  m_n & = & \log\frac{\mu(x_{1..n})}{\rho(x_{1..n})}
          - \log\frac{\mu(x_{<n})}{\rho(x_{<n})}
      - \E^n\log\frac{\mu(x_{1..n})}{\rho(x_{1..n})}
      + \E^n\log\frac{\mu(x_{<n})}{\rho(x_{<n})}
 \\
      & = & \log\frac{\mu(x_{1..n})}{\rho(x_{1..n})}
      - \E^n\log\frac{\mu(x_{1..n})}{\rho(x_{1..n})}.
\eqan
Let $f(n)$ be some function monotonically increasing to infinity such that
\beq\label{eq:f2}
  \sum_{n=1}^\infty\frac{(\log c_n^{-1}+f(n))^2}{n^2} \;<\; \infty
\eeq
(e.g.\ choose $f(n)=\log n$ and exploit
$(\log c_n^{-1}+f(n))^2 \leq 2(\log c_n^{-1})^2+2f(n)^2$ and (\ref{eq:domsum}).)
For a sequence of random variables  $\lambda_n$ define
\beqn
  (\lambda_n)^{+(f)} \;=\; \left\{ \begin{array}{ll} \lambda_n & \mbox{ if }
    \lambda_n\ge - f(n) \\ 0 & \mbox{ otherwise }\end{array}\right.
\eeqn
and $\lambda_n^{-(f)}=\lambda_n - \lambda_n^{+(f)}$.
Introduce also
\beqn
  m^+_n \;=\; \left (\log \frac{\mu(x_{1..n})}{\rho(x_{1..n})}\right)^{+(f)}
         - \E^n\left(\log\frac{\mu(x_{1..n})}{\rho(x_{1..n})}\right)^{+(f)},
\eeqn $m_n^-=m_n-m_n^+$ and the averages $\bar m^+_n$ and $\bar
m_n^-$. Observe that $m_n^+$ is a martingale difference sequence.
Hence to establish the convergence $\bar m^+_n\rightarrow0$ we can
use the martingale strong law of large numbers
\cite[p.501]{Shiryaev:96}, which states that, for a martingale
difference sequence $\gamma_n$, if $\E(n\bar \gamma_n)^2<\infty$
and $\sum_{n=1}^\infty \E \gamma_n^2/n^2<\infty$ then $\bar
\gamma_n\rightarrow0$ a.s. Indeed, for $m^+_n$ the first condition
is trivially satisfied (since the expectation in question is a
finite sum of finite numbers), and the second follows from the
fact that $|m_n^+|\le \log c_n^{-1}+f(n)$ and~(\ref{eq:f2}).

Furthermore, we have
\beqn
  m_n^- \;=\; \left(\log \frac{\mu(x_{1..n})}{\rho(x_{1..n})}\right)^{-(f)}
        - \E^n\left(\log\frac{\mu(x_{1..n})}{\rho(x_{1..n})}\right)^{-(f)}.
\eeqn
As it was mentioned before, $\log
\frac{\mu(x_{1..n})}{\rho(x_{1..n})}$ converges $\mu$-a.s.\ either
to (positive) infinity or to a finite number. Hence $\left(\log
\frac{\mu(x_{1..n})}{\rho(x_{1..n})}\right)^{-(f)}$ is non-zero
only a  finite  number of times, and so  its average goes to zero.
To see that 
$$\E^n\left(\log
\frac{\mu(x_{1..n})}{\rho(x_{1..n})}\right)^{-(f)}\rightarrow0
$$ we
write
\bqan
  \E^n\left(\log \frac{\mu(x_{1..n})}{\rho(x_{1..n})}\right)^{-(f)}
  & =& \sum_{x_n\in\X}\mu(x_n|x_{<n}) \left( \log \frac{\mu(x_{<n})}{\rho(x_{<n})}
        + \log\frac{\mu(x_n|x_{<n})}{\rho(x_n|x_{<n})}\right)^{-(f)}
\\
  & \ge & \sum_{x_n\in\X}\mu(x_n|x_{<n}) \left( \log \frac{\mu(x_{<n})}{\rho(x_{<n})}
        + \log\mu(x_n|x_{<n})\right)^{-(f)}
\eqan and note that the first term in brackets is bounded from
below, and so for the sum in brackets to be less than $-f(n)$
(which is unbounded) the second term $\log\,\mu(x_n|x_{<n})$ has
to go to $-\infty$, but then the expectation goes to zero since
$\lim_{u\rightarrow0} u\log u=0$.

Thus we conclude that $\bar m_n^-\rightarrow0$ $\mu$-a.s., which
together with $\bar m_n^+\rightarrow0$ $\mu$-a.s.\ implies $\bar
m_n\rightarrow0$ $\mu$-a.s., which, finally, together with $\bar
l_n\rightarrow0$ $\mu$-a.s.\ implies $\bar d_n\rightarrow0$
$\mu$-a.s. 
\end{proof}

Next, we show that no form of dominance with decreasing coefficients is
sufficient for prediction in absolute distance or KL divergence,
as the following impossibility result states.

\begin{proposition}\label{th:nodom}
For each sequence of positive numbers $c_n$ that goes to 0 there
exist measures $\mu$ and $\rho$ and a number $\epsilon>0$ such
that
$
 \rho(x_{1..n})\ge c_n \mu(x_{1..n})
$
for all $x_{1..n}$, yet $a_n(\mu,\rho|x_{1..n})>\epsilon$ and
$\delta_n(\mu,\rho|x_{1..n})>\epsilon$ infinitely often $\mu$-a.s.
\end{proposition}

\begin{proof}
Let $\mu$ be concentrated on the sequence $11111\dots$ (that is
$\mu(x_n=1)=1$ for all $n$), and let $\rho(x_n=1)=1$ for all $n$
except for a subsequence of steps $n=n_k$, $k\in\SetN$ on which
$\rho(x_{n_k}=1)=1/2$ independently of each other. It is easy to
see that choosing $n_k$ sparse enough we can make
$\rho(1_1\dots1_n)$ decrease to $0$ arbitrary slowly; yet
$|\mu(x_{n_k})-\rho(x_{n_k})|=1/2$ for all $k$. 
\end{proof}

Thus for the  question of whether dominance with some
coefficients decreasing  to zero is sufficient for prediction, we
have the following table of  questions and answers, where, in
fact, positive answers for $a_n$ are implied by positive answers
for $\delta_n$ and vice versa for the negative answers: \vskip 2mm
\begin{center}
\begin{tabular}{|c|c|c|c|c|c|}\hline
  $\E \bar d_n  $ &  $\bar d_n \vphantom{\bar{\hat d}} $ &  $\delta_n$ &  $\E\bar a_n$& $\bar a_n$& $a_n$ \\\hline
  + & + &  $-$ & + & + & $-$\\\hline
\end{tabular}
\end{center}
However, if we take into account the conditions on the
coefficients,  we see some open problems left, and different
answers for $\bar d_n$ and $\bar a_n$ may be obtained. Following
is the table of conditions on dominance coefficients and answers
to the questions whether these conditions are sufficient for
prediction (coefficients bounded from below are included for the
sake of completeness).
\begin{center}
\begin{tabular}{|c|c|c|c|c|c|c|}\hline
                                                      & $\E \bar d_n $ &  $\bar d_n \vphantom{\bar{\hat d}} $ &  $\delta_n$ &  $\E\bar a_n$& $\bar a_n$& $a_n$\\\hline
 $\log c_n^{-1}=o(n)$                                 &  +             &  ?                                   &   $-$  &       +      &  ?         & $-$  \\\hline
 $\sum_{n=1}^\infty\frac{\log c_n^{-1}}{n^2}< \infty$ &  +             &  +                                   &   $-$ &         +     &  +         & $-$  \\\hline\hline
 $c_n\ge c>0$                                         &  +             &  +                                   &   $+$ & + & + &+ \\\hline
\end{tabular}
\end{center}
We know from Proposition~\ref{th:nodom} that the condition $c_n\ge
c>0$ for convergence in $\delta_n$ can not be improved; thus the open
problem left is to find whether  $\log\, c_n^{-1}=o(n)$ is
sufficient for prediction in $\bar d_n$ or at least in $\bar a_n$.

Another open problem is to find out whether any conditions on
dominance coefficients are necessary for prediction; so far we
only have some sufficient conditions. On the one hand, the
results presented suggest that some form of dominance with
decreasing coefficients may be necessary for prediction, at least
in the sense of convergence of averages. On the other hand, the
condition~(\ref{eq:dom}) is uniform over all sequences and this is 
probably not necessary for prediction. As for prediction in the
sense of almost sure convergence, perhaps more subtle behaviour of
the ratio $\frac{\mu(x_{1..n})}{\rho(x_{1..n})}$ should be analysed,
since dominance with decreasing coefficients is not sufficient for
prediction in this sense.

\chapter{Conclusion and outlook}
A statistician facing an unknown stochastic phenomenon has a large, nonparametric model class at hand that 
she has reasons to believe captures some aspects of the phenomenon. Yet other aspects remain completely unknown, 
and there is no hope that the process generating the data indeed comes from the model class.
For this reason, the statistician may be content with having non-zero error no matter how much data may become available 
now or in the future, but she would still like to make some use of the model. There are now two rather distinct ways 
to proceed. One  is to say that the data may come from an arbitrary sequence, and to try to construct 
a predictor that minimizes the regret with respect to every distribution in the model class, on every  sequence. Thus, one would be treating the model class as a set of experts. The other way is to try to enlarge
the model class, in particular, by allowing that all there is unknown in the process may be arbitrary (that is, 
an arbitrary  sequence). Having done this, one may safely assume that the probability measure that generates the data belongs to the model class. This second  way may be more difficult precisely on the modelling step. Yet, 
the conclusion one can draw from the results exposed in this volume  is that this is the way to follow, for in this case one can be sure that it is possible to make statistical inference by standard available tools, specifically, Bayesian forecasting. Indeed,   even if 
the best achievable asymptotic error is non-zero, it is attained by a mixture predictor (with a discrete prior).
Finding such a prior is a separate problem, but it is a one with  which Bayesians are familiar. Here, modelling the parts of the environment that are completely unknown
should not be very difficult:  a good prior distribution  over all sequences is just the Bernoulli i.i.d.\ measure with equiprobable outcomes.
(Note that it is not necessary to look for priors concentrated on  countable sets.) 
On the other hand, for the regret-minimization route, the statistician cannot use an arbitrary model class; indeed, 
she would first need to make sure that regret minimization is viable  at all for the model class at hand: it may happen that every combination of distributions in the model is suboptimal. 

It is worth noting that the conclusions summarized above  are not about Bayesian versus non-Bayesian inference. 
Rather, Bayesian inference is used as a generic approach to construct predictors for general (uncountable) model
classes.  At this level of generality, the choice of  alternative  approaches is very limited, but it would 
be interesting to see which predictors constructed for more specific settings can be generalized to arbitrary model classes, and whether a universality result akin to those presented in this book for mixture predictors  holds for such generalizations.

Yet another possible way to address the modelling problem is somewhere in-between the two outlined above: rather than making the model set large enough as to include the data-generating mechanism, and rather than allowing the data sequence to be arbitrary and only considering the model set as a set of experts, one may wish to assume that, while none of the measures in the model set $\C$ is  good enough to describe the data-generating mechanism completely, at least one of them is good enough for prediction, in the sense that its loss vanishes in asymptotic. This is the ``middle case'' of Chapter~\ref{ch:p2}, and, as we have seen in that chapter, for this problem as well the optimal (asymptotic) performance can be obtained by a mixture predictor.

Next we take a look at  potential generalizations and interesting directions for future research.
These can be divided into following groups: direct generalizations of the results presented; related and more general questions about sequential prediction;  generalizations beyond sequential prediction and links with other problems in learning and statistics. 
\section{Generalizations: infinite alphabets and different losses}
The results exposed in this book call for direct generalizations along several directions.
{\noindent\bf  Infinite alphabets.}
 The first question that comes to mind is about 
generalizing the results on mixture predictors to infinite  alphabets $\X$. 
Of particular interest in this respect are Theorem~\ref{th:main} and the corresponding asymptotic result, Corollary~\ref{cl:main}. Namely, 
it has been shown that, in the realizable setting, the performance of any predictor can be matched by that of a mixture predictor. Can a similar result be obtained for infinite alphabets?  There are some important differences with the finite-alphabet case. First of all, for an infinite alphabet, the very first observation may provide an infinite amount of information. One can easily come up with examples where the first observation identifies with certainty the measure that generates the whole sequence, even though the set $\C$ of possible measures is infinite. Thus it may be needed to pose restrictions on the set $\C$ (which is undesirable), or perhaps relax the prediction goal, for example, estimating only the mean of the next outcome rather than its probability distribution. These difficulties are already seen in the i.i.d.\ case, see \cite{Dembo:94}. The case of countably infinite alphabets is not exempt from this consideration either. It is particularly easy to  see if the loss is measured with KL divergence.  Indeed, it is easy to check (and is well-known in information theory) that for every predictor there is a distribution such that the KL divergence between the true and predicted probabilities is infinite already on the first step (note that the distribution may be i.i.d., as its behaviour beyond the first time-step is irrelevant). This also shows that different losses may be better suited for infinite alphabets.
{\noindent\bf  Different loss functions.}
The second, closely related, potential generalization is to different losses. The fact that the main results in this book are established with respect to two very different losses (total variation and expected average KL divergence) strongly  suggests that these results are generalizable. The main property of the loss that we used is that predictive quality is not affected (in asymptotic; or almost not affected for finite time) by taking mixtures. This property breaks for some other losses, as shown in Section~\ref{sec:sum}. Thus, it appears interesting to find exactly what properties of the loss are required to establish universality of mixture predictors and obtain a corresponding generalization. The results of interest in this respect are, again, Theorem~\ref{th:main} and Corollary~\ref{cl:main}, as well as the middle-case, Theorem~\ref{th:pq4-2}.
{\noindent\bf  Non-asymptotic results for the middle-case problem.}
It is also worth noting that for the middle-case problem, unlike for the realizable case, only an asymptotic version of the result concerning universality of mixture predictors has been obtained (Theorem~\ref{th:pq4-2}). It would be interesting to see whether a non-asymptotic version, akin to Theorem~\ref{th:main}, can be established for this problem. 

{\noindent\bf  The gap between the upper and lower bound.}
Finally, it is worth mentioning that the upper (Theorem~\ref{th:main}) and lower (Theorem~\ref{th:lb}) bounds on the performance of the best mixture predictor have a $O(\log n)$ gap to them. This gap is significant, since the best achievable cumulative loss for some important classes of processes, such as finite-memory processes, is of the same order. It may thus be interesting to look whether the bound can be improved specifically for sets of processes for which this is the case.
\section{The general sequence-prediction questions}
The questions that we attempted to answer in this book and the results obtained can be seen as a part of a bigger picture concerning sequence-prediction problems. In light of these results, the questions outlined in the Introduction~(Section~\ref{s:gm}) can be updated as follows.

\begin{svgraybox}
 {\bf Sequence prediction: realizable case.} A set of probability measures $\C$ is given. %
It is required to find a predictor $\rho$ that attains minimax asymptotic loss: 
$$
L(\mu,\rho)\le V_\C
$$
for all $\mu\in\C$.
\end{svgraybox}
We know (Corollary~\ref{cl:main}) that such a predictor exists for every set $\C$ whatsoever,  and it can be found in the form of a mixture predictor. However, we still do not know, in general, how to find~it.

{\noindent\bf  Characterization of sets $\C$ with $V_\C=\alpha$.}
Another closely related question is to find characterisations of the sets $\C$ for which $V_\C=\alpha$ in terms of the parameter $\alpha$. Finding such generalizations may be of help when constructing predictors. As we have seen (Theorem~\ref{th:tv}), this is the case for prediction in total variation, where a complete characterization may be obtained in terms of separability: in this latter case, taking a mixture over any dense subset results in a predictor we are looking for. Based on the results of \cite{BRyabko:86} and in light of Theorem~\ref{th:hau}, it can be suggested that the Hausdorff dimension may be instrumental for this in case $\alpha>0$. On the other hand, of course, the case  $V_\C=0$ remains to be of special interest, and finding characterizations of sets $\C$ for which $V_\C=0$ is an interesting problem. 

For the  middle-case problem (Chapter~\ref{ch:p2}) we can formulate a similar question:
\begin{svgraybox}
 {\bf Sequence prediction: middle case.} A set of probability measures $\C$ is given. 
 It is required to find a predictor $\rho$ that attains asymptotically vanishing loss on any measure $\mu$ for which at least one measure from $\C$ is a good predictor (its loss vanishes):
$$
L(\mu,\rho) =0\text{ for all $\mu\in\C^+,$}
$$
 where %
$
\C^+:=\{\mu\in\mathcal P: \exists \nu\in\C\ d(\nu,\mu)=0\}.
$
\end{svgraybox}
While this might look like a special case of the  previous question (substitute $\C$ with $\C^+$ and take the special case $V_\C=0$), what makes this formulation  interesting is that, as we know from Theorem~\ref{th:pq4-2}, it is sufficient to look for a solution in terms of mixtures over $\C$ and not over $\C^+$.

Finally, for the non-realizable case, we can ask the following question.
\begin{svgraybox}
 {\bf Sequence prediction: non-realizable case.} A set of probability measures $\C$ is given and is considered as a set of predictors to compare to. The unknown data-generating mechanism $\mu$ can be arbitrary.   It is required to construct a predictor $\rho$ that attains optimal regret with respect to $\C$: 
$$
R(\C,\rho)=U_\C.
$$
\end{svgraybox}
A major difference with respect to the previous two problems is that here we do not know where to look for a predictor,  as any Bayesian mixture over $\C$ may be suboptimal (Theorem~\ref{th:not}). Still, we know that mixture predictors do attain the minimal regret for some sets of processes, such as Bernoulli or finite-memory processes (Theorem~\ref{th:mark}). Thus, it appears interesting to first answer the following question.

\begin{svgraybox}
\noindent{  Characterize the sets $\C$} for which the { minimal regret} can be  { attained by a mixture} predictor.
\end{svgraybox}
\noindent{\bf Minimax, decision-theoretic questions.} Finally, one cannot help noticing a strong decision-theoretic flavour of the questions above. Indeed, as we have seen in Section~\ref{s:dt}, some of the main results can be seen as complete-class or (partial) minimax results. As mentioned in that section, it remains open to see whether the minimax theorem holds for $V_\C$. It is also worth mentioning that, while minimax and complete-class theorems constitute an old and well-studied part of decision theory, and the results on sufficient conditions for these kind of theorems to hold are abound, they turn out to be not directly applicable (at least, to the best of this author's knowledge) to the problems at hand. It thus appears interesting to look at these questions more generally, that is,  beyond the specific loss used for the problem of prediction (KL divergence), trying to find out what conditions are sufficient for these theorems to hold that would include the specific cases of interest here.
\section{Beyond sequential prediction}
Sequential prediction is one of the most basic problems  of  statistics, machine learning and information theory, which forms a basis of and is otherwise linked to various other problems. Here we attempt to look at what other problems can benefit from the results exposed in this volume. 

{\noindent\bf  Active learning problems: reinforcement learning, bandits.}
A basic implicit assumptions in sequential prediction is that the environment is {\em passive}, in the sense that the predictions made do not affect future outcomes. Allowing the environment to depend on the predictions leads to the problem known as reinforcement learning or sequential decision making (e.g., \cite{Szepesvari:2010,dimitrakakis2018decision}). Here, the goal is usually not just to emit predictions but to optimize some utility. The latter may come as special observations provided by the environment, called rewards, or may be a function of the observations. While mixture predictors and Bayesian methods are also used for these problems (e.g., \cite{Hutter:04uaibook}), this problem also exhibits other distinct features which are not easily solved by mixtures. One of these problems can be identified with the so-called {\em traps} in the environment: some actions made by the predictor may turn out  to be fatal, in the sense that no recovery is possible, and thus it is not possible to make any meaningful inference any more. This problem and some possible solutions are explored in \cite{Ryabko:08ao++}.  The problem of  traps is absent in a special case of the reinforcement learning problem called the bandit problem. 
Here, instead of one sequence of observations, there are several sequences (more generally, the set of sequences may be infinite), but on each time step only one sequence is observed. Which one is observed depends on the action of the learner.   The goal is also to optimize some utility function (e.g., rewards).  This formulation puts the light on another major feature not present in the problem of sequence prediction: the observations are not fully observable. Thus, the main problem~--- besides predicting the probabilities of future observations~--- is optimal {\em exploration}, that is, selecting which sequences to observe in order to get enough information needed to take optimal decisions. It should be noted that both the bandit and the reinforcement learning problem are typically considered under rather narrow probabilistic assumptions, namely, Markov (for reinforcement learning) and i.i.d.\ outcomes (for bandits). The questions pertaining to the generalizations of the results presented here, of course, only concern the general settings where no such assumptions are made.

The exploration problem appears to be rather different from the prediction problem, and thus there is no reason to suppose that the results presented in this book should be directly generalizable to active learning problems. Rather, it would be interesting to attempt to decompose learning problems into basic building blocks, and solve them separately. Prediction is clearly one of such blocks, and so the question would be rather how to combine predictors with other algorithms to solve such general problems. Another statistical problem that appears to be orthogonal to prediction, and thus may  serve as a building block for solving more general learning problems (e.g., reinforcement learning) is hypothesis testing; see \cite{Ryabko:19c} for some general results concerning characterisation of those sets of process distributions for which consistent tests exist. 

On a very high level, it may be conjectured that {\em prediction, exploration and testing are the three  inference problems to which one might be able to  reduce all others} (including reinforcement learning). 

{\noindent\bf Beyond sequences: infinite random graphs.}
Finally, one last assumption made  in sequential prediction that is worth trying to generalize is the sequential nature of observations. Instead of a series of observations coming from the same alphabet $\X$, one may consider other information sequences~--- an observation made already in the seminal work \cite{Blackwell:62}. An interesting case is when, instead of an infinite sequence, the observations make up some more general structure, such as an infinite random graph. While these may be endowed with a structure of a probability space, see \cite{lyons2016probability}, making statistical inference on these structures brings about a whole different set of problems. Some first steps towards making statistical inference in this more general case, but assuming stationarity of the underlying graph, have been taken in \cite{Ryabko:17gratest}, based on the theory developed in \cite{Benjamini:12}. Generalizing the results presented here to this case  remains, so far, an unexplored direction of research.

\end{document}